%% file: main.tex
\documentclass[10pt]{article} 

\usepackage[preprint]{tmlr} 

\input{math_commands.tex}

\usepackage{hyperref}
\usepackage{xspace}
\usepackage{url}
\usepackage{enumitem}
\usepackage{booktabs}
\usepackage{float}          
\usepackage{algorithm}      
\usepackage{algpseudocode}  
\usepackage[nameinlink,noabbrev]{cleveref}
\usepackage{xcolor}
\usepackage{graphicx}
\usepackage{tikz}
\usetikzlibrary{arrows.meta,calc,positioning}
\usepackage{aliascnt}
\usepackage{amsthm}
\usepackage{pgf}

\newtheorem{theorem}{Theorem}
\newaliascnt{proposition}{theorem}
\newtheorem{proposition}[proposition]{Proposition}
\aliascntresetthe{proposition}
\newaliascnt{lemma}{theorem}
\newtheorem{lemma}[lemma]{Lemma}
\aliascntresetthe{lemma}
\newaliascnt{corollary}{theorem}
\newtheorem{corollary}[corollary]{Corollary}
\aliascntresetthe{corollary}
\theoremstyle{definition}
\newaliascnt{definition}{theorem}
\newtheorem{definition}[definition]{Definition}
\aliascntresetthe{definition}
\newaliascnt{example}{theorem}

\aliascntresetthe{example}
\newaliascnt{assumption}{theorem}

\aliascntresetthe{assumption}
\newaliascnt{remark}{theorem}

\aliascntresetthe{remark}

\crefname{theorem}{Theorem}{Theorems}
\Crefname{theorem}{Theorem}{Theorems}
\crefname{proposition}{Proposition}{Propositions}
\Crefname{proposition}{Proposition}{Propositions}
\crefname{lemma}{Lemma}{Lemmas}
\Crefname{lemma}{Lemma}{Lemmas}
\crefname{corollary}{Corollary}{Corollaries}
\Crefname{corollary}{Corollary}{Corollaries}
\crefname{definition}{Definition}{Definitions}
\Crefname{definition}{Definition}{Definitions}
\crefname{example}{Example}{Examples}
\Crefname{example}{Example}{Examples}
\crefname{remark}{Remark}{Remarks}
\Crefname{remark}{Remark}{Remarks}
\crefname{algorithm}{Algorithm}{Algorithms}
\Crefname{algorithm}{Algorithm}{Algorithms}

\title{Barycentric Projections of Optimal Transport Plans on Riemannian Manifolds}

\author{\name Kisung You \email kisung.you@baruch.cuny.edu \\
      \addr Department of Mathematics\\
      Baruch College
      }

\providecommand{\R}{\mathbb{R}}
\providecommand{\E}{\mathbb{E}}
\providecommand{\1}{\mathbf{1}}
\providecommand{\sM}{\mathcal{M}}
\providecommand{\sX}{\mathcal{X}}
\providecommand{\sY}{\mathcal{Y}}
\providecommand{\sD}{\mathcal{D}}
\providecommand{\sK}{\mathcal{K}}
\providecommand{\sO}{\mathcal{O}}
\providecommand{\va}{\mathbf{a}}
\providecommand{\vb}{\mathbf{b}}
\providecommand{\vone}{\mathbf{1}}
\providecommand{\rvx}{\mathbf{x}}
\providecommand{\rvy}{\mathbf{y}}
\providecommand{\mGamma}{\boldsymbol{\Gamma}}
\providecommand{\mLambda}{\boldsymbol{\Lambda}}

\providecommand{\Frechet}{Fr\'{e}chet\xspace}

\DeclareMathOperator{\spt}{spt}
\DeclareMathOperator{\Cut}{Cut}
\DeclareMathOperator{\id}{id}

\def\month{MM}  
\def\year{YYYY} 
\def\openreview{\url{https://openreview.net/forum?id=XXXX}} 

\begin{document}

\maketitle

\begin{abstract}
Optimal transport couplings are probabilistic objects, while many learning
pipelines require deterministic maps. In Euclidean space, barycentric projection
converts a coupling into a map by taking conditional expectations, but on a
Riemannian manifold curvature and cut loci make this operation nontrivial. We
develop a framework for barycentric projections of transport couplings on
Riemannian manifolds. The intrinsic projection maps each source point to the
conditional Fr\'echet mean of its destination law and is shown to be the best
deterministic representative under squared geodesic loss. The corresponding
minimum value is an integrated conditional Fr\'echet variance, which vanishes
exactly for map-induced couplings and therefore defines a conditional-variance
Monge defect. We also study a tangential log--exp projection, prove its
Euclidean exactness, its compatibility with Brenier--McCann maps in the Monge
case, and its interpretation as the first unit Riemannian gradient update for
the intrinsic objective. For discrete couplings, both constructions decompose
row-wise into weighted Fr\'echet mean and log--exp problems. Experiments on
spherical data, synthetic SPD data, and real EEG covariance matrices support
the proposed division of roles: the intrinsic projection is the variational
representative, while the tangential projection is a useful local displacement
surrogate.
\end{abstract}

\section{Introduction}
\label{sec:introduction}

Optimal transport (OT) tells us how to rearrange mass \citep{monge_1781_MemoireTheorieDeblais}. Learning systems often ask a sharper question: once a source point has been matched in distribution, where should it actually go? This distinction is harmless if the only goal is to compare probability measures. It becomes central when optimal transport is used as a deterministic layer, a correspondence rule, a feature-alignment procedure, a domain-adaptation mechanism, or a transport-based predictor. A transport plan may split the mass of one source point among several destinations, while such downstream tasks typically expect one output point for each input point.

In Euclidean space, the standard answer is barycentric projection that replaces the conditional cloud of destinations attached to each source point by its average. Equivalently, the barycentric projection is the conditional expectation of the destination given the source, and its optimality under squared Euclidean loss is the familiar $L^2$ projection property of conditional expectation. This Euclidean fact is conceptually simple, but it does not immediately transfer to curved spaces. On a Riemannian manifold, there is no unique global notion of averaging. One may average intrinsically using geodesic distance, or one may average displacement vectors in a tangent space and map the result back with the exponential map. These two procedures agree in flat Euclidean space. On a curved manifold, they may differ, and the tangent-space procedure may even be undefined because of the cut locus.

This paper develops a systematic framework for turning transport couplings on Riemannian manifolds into deterministic maps. Our motivating case is the exact, non-entropic optimal transport plan. In that setting, mass splitting is a genuine feature of the Kantorovich solution, especially for empirical measures with unequal weights or unequal support sizes. The problem of extracting a deterministic representative is therefore not merely a numerical post-processing choice, but a geometric plan-to-map problem.

We study two complementary constructions, which are visualized in Figure~\ref{fig:intrinsic-vs-tangential}. The first is intrinsic. For each source point, take the \Frechet mean of its conditional destination law. This ties barycentric projection directly to the classical theory of \Frechet and Karcher means on manifolds \citep{frechet_1948_ElementsAleatoiresNature,karcher_1977_RiemannianCenterMass,afsari_2011_Riemannian$L_p$Center}. The second is tangential. For each source point, we take an average of the logarithmic displacement vectors pointing toward its destinations and then project onto the manifold with the exponential map. This second construction is naturally connected to the displacement-based description of quadratic optimal maps due to Brenier in Euclidean space and McCann on Riemannian manifolds \citep{brenier_1991_PolarFactorizationMonotone, mccann_2001_PolarFactorizationMaps}.

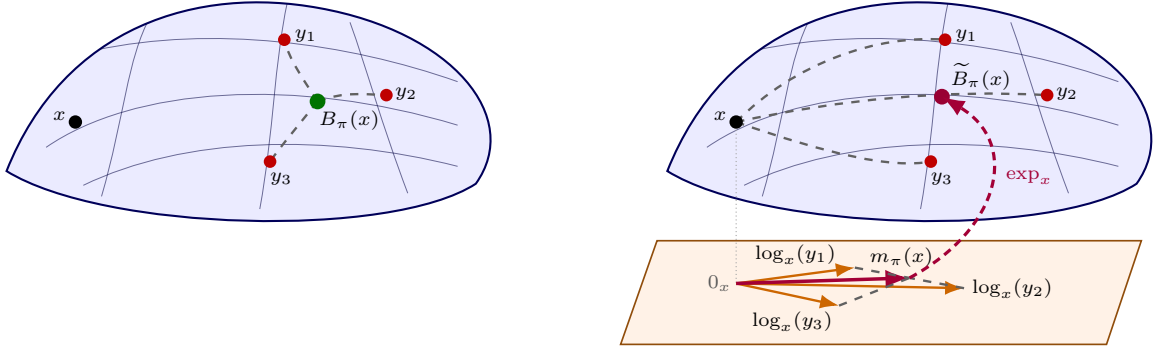
\begin{figure}[ht]
\centering
\resizebox{0.96\linewidth}{!}{%
  \input{figures/tikz_snippet}
}
\caption{Intrinsic and tangential barycentric projections. The intrinsic projection (left) solves the conditional \Frechet mean problem directly on the manifold, whereas the tangential projection (right) first averages log-mapped destination points in the tangent space and then pushes the result back to the manifold.}
\label{fig:intrinsic-vs-tangential}
\end{figure}

The main message is asymmetric. The intrinsic construction is the variational representative of a plan, which solves the conditional \Frechet mean problem and is optimal among all deterministic maps under squared geodesic loss. In regimes where conditional \Frechet means are unique, this representative is canonical. Without uniqueness, the natural object is a compact-valued barycentric correspondence together with measurable selections. Meanwhile, the  tangential construction is a structured surrogate that is compatible with the Brenier--McCann displacement formula in the Monge case and coincides with the intrinsic construction in Euclidean space, but it is local and sensitive to cut-locus phenomena.

This distinction is useful for manifold-valued learning. A transport plan is
often the right probabilistic object, while a deterministic map is often the
right computational object. We make this passage precise under explicit
assumptions on measurability, uniqueness, curvature, and cut loci.

Our contributions are as follows.
\begin{enumerate}
    \item We formalize the plan-to-map problem for manifold-valued transport couplings through conditional destination laws. The core theory applies to arbitrary finite-cost couplings, while exact optimal transport plans are the primary motivating case.
    \item We propose the intrinsic barycentric projection, defined by conditional Fr\'echet means, and prove that it is the best deterministic representative under squared geodesic loss. This is the Riemannian analogue of the Euclidean conditional-expectation property, with the necessary curvature and measurability infrastructure made explicit.
    \item We identify the integrated conditional Fr\'echet variance as a conditional-variance Monge defect. It is exactly the minimum deterministic representation error and vanishes precisely for map-induced couplings.
    \item We analyze the tangential log--exp projection as a displacement-coordinate surrogate. We prove its Euclidean exactness, its compatibility with Brenier--McCann maps in the Monge case, its Borel measurability in the Hadamard regime, and its interpretation as the first unit Riemannian gradient update for the intrinsic objective.
    \item We derive row-wise algorithms for discrete couplings. The intrinsic method reduces to weighted Fr\'echet mean problems, and the tangential method reduces to log--exp averages. A hybrid initialization is described as an algorithmic option. 
\end{enumerate}

The rest of the paper is organized as follows. \Cref{sec:background}
reviews the three ingredients behind the framework: barycentric projection of
transport plans, Fr\'echet means on nonlinear spaces, and tangent-space
displacement coordinates. \Cref{sec:setup} fixes notation and defines the
intrinsic and tangential projections from a common coupling. \Cref{sec:intrinsic}
develops the intrinsic projection and proves its variational optimality, its
Monge-defect characterization, and the Hadamard energy gap. \Cref{sec:tangent}
studies the tangential projection as a displacement-coordinate surrogate and
relates it to Brenier--McCann maps and to the first-order geometry of the
intrinsic objective. \Cref{sec:algorithms} gives row-wise algorithms for
discrete couplings. \Cref{sec:experiments} reports synthetic and real-data
experiments, and we conclude with scope and future directions in  \Cref{sec:conclusion}.

\section{Background}
\label{sec:background}

The paper is built from three familiar ideas that do not usually appear together in this form:  barycentric projection of an optimal transport plan, \Frechet means on nonlinear spaces, and tangent-space displacement coordinates. We recall these connections first, keeping the formal notation for the next section.

\subsection{Plan-to-map operations and conditional barycenters}

Formulation of the OT problem by \citet{kantorovitch_1958_TranslocationMasses} produces couplings, while many applications require maps. In computational OT, especially for empirical measures, barycentric projection is the standard device for converting a coupling into a transported point cloud or deterministic representative \citep{santambrogio_2015_OptimalTransportApplied, peyre_2019_ComputationalOptimalTransport}.  In Euclidean space this operation is simply conditional expectation. For a coupling of random variables $(\rvx,\rvy)$, the map $x\mapsto \E[\rvy\mid \rvx=x]$ is the $L^2$-optimal deterministic representative of the conditional law of $\rvy$ given $\rvx$. This observation underlies common OT-based map-estimation and domain-adaptation pipelines \citep{perrot_2016_MappingEstimationDiscrete,courty_2017_JointDistributionOptimal}.

We keep this plan-to-map motivation but change the geometry. Once the state space is a Riemannian manifold, the conditional expectation is no longer available as a linear average. The natural intrinsic replacement is the conditional Fr\'echet mean, which is defined as a point, or set of points, minimizing expected squared geodesic distance to the conditional destination law. Fr\'echet means, also called barycenters, generalize Euclidean averages \citep{frechet_1948_ElementsAleatoiresNature}, whose local versions on Riemannian manifolds are often called Karcher means \citep{karcher_1977_RiemannianCenterMass}, requiring attention to injectivity radii, curvature, and cut loci \citep{pennec_2006_IntrinsicStatisticsRiemannian, afsari_2011_Riemannian$L_p$Center}. In nonpositively curved spaces, they enjoy global uniqueness and variance inequalities \citep{sturm_2003_ProbabilityMeasuresMetric}.

There is also a statistical literature on conditional Fr\'echet means and Fr\'echet regression for metric-space-valued responses \citep{petersen_2019_FrechetRegressionRandom}. 
Our use of conditional Fr\'echet means is close in spirit but differs in its
source. The conditional law is specified by a transport coupling, not by an
external regression model.

\subsection{Displacement coordinates and tangent-space transport}

The second construction studied in the paper comes from a different OT tradition. In Euclidean quadratic OT, Brenier's theorem represents optimal maps by gradients of convex potentials \citep{brenier_1991_PolarFactorizationMonotone}. On a Riemannian manifold, McCann's polar factorization theorem gives the corresponding displacement representation $T(x)=\exp_x(-\nabla\psi(x))$ under the assumptions ensuring a Monge solution \citep{mccann_2001_PolarFactorizationMaps}. The vector $\log_x(T(x))$ is therefore the natural displacement coordinate of a deterministic optimal map.

Linearized or tangent-space OT methods similarly use displacement coordinates to replace nonlinear Wasserstein geometry by linear surrogates. Recent manifold versions include linearized optimal transport on manifolds, which studies barycentric logarithmic maps, discretization consistency through barycentric projection, and continuity properties of logarithmic maps \citep{sarrazin_2024_LinearizedOptimalTransport}. Our tangential projection is related in a row-wise sense that it averages the displacement vectors associated with a fixed conditional destination law rather than embedding an entire space of measures.

The analysis below separates these two viewpoints. The intrinsic projection is the variational object. The tangential projection is the displacement-coordinate surrogate. The point of the formal setup is to define both from the same coupling and then prove exactly what each one represents.

\section{Problem setup and proposed projections}
\label{sec:setup}

We now fix notation and state the plan-to-map problem precisely. The main theory applies to any finite-cost coupling between two probability measures. Exact optimal transport plans are the primary motivating case, but optimality is not needed for the core barycentric projection theorems.

\subsection{Geometry, couplings, and conditional laws}

Let $\sM$ be a connected, complete, finite-dimensional Riemannian manifold with Riemannian metric tensor $g$ and geodesic distance $d$. For $x\in\sM$, write $T_x\sM$ for the tangent space, $\langle\cdot,\cdot\rangle_x$ for the Riemannian inner product on $T_x\sM$, and $\|\cdot\|_x$ for the associated norm. The exponential map at $x$ is denoted by $\exp_x:T_x\sM\to\sM$. The cut locus of $x$ is denoted by $\Cut(x)$. Whenever $y\notin\Cut(x)$, we write $\log_x(y)\in T_x\sM$ for the inverse exponential, or logarithmic map, of \(y\) at \(x\), i.e., the initial velocity of the unique minimizing geodesic from $x$ to $y$.

For a Polish space $\sX$, $\mathcal P(\sX)$ denotes the set of Borel probability measures on $\sX$, and $\mathcal P_2(\sX)$ denotes the subset of measures with finite second moment. If $\rho\in\mathcal P(\sX)$, $\spt(\rho)$ denotes its support. If $T:\sX\to\sY$ is Borel, $T_\#\rho$ denotes the pushforward of $\rho$. On product spaces, $\mathrm{pr}_1$ and $\mathrm{pr}_2$ denote coordinate projections. All maps are understood to be Borel unless explicitly stated otherwise.

Let $\mu,\nu\in\mathcal P_2(\sM)$. We write
\[
\Pi(\mu,\nu)
:=
\left\{\gamma\in\mathcal P(\sM\times\sM):
(\mathrm{pr}_1)_\#\gamma=\mu,
\ (\mathrm{pr}_2)_\#\gamma=\nu
\right\}
\]
for the set of couplings between $\mu$ and $\nu$. Since both marginals have finite second moment, every $\gamma\in\Pi(\mu,\nu)$ has finite quadratic cost by the triangle inequality. Throughout the paper we use the cost
\[
 c(x,y)=\frac12 d(x,y)^2.
\]
The factor $1/2$ does not change minimizers but gives cleaner first-order formulas. The exact quadratic optimal transport problem is
\[
\inf_{\gamma\in\Pi(\mu,\nu)}
\int_{\sM\times\sM}\frac12 d(x,y)^2\,d\gamma(x,y).
\]
Any minimizer is called an exact optimal transport plan. Such a minimizer exists by the standard lower-semicontinuity and tightness argument for quadratic cost on Polish spaces \citep{villani_2009_OptimalTransportOld}. Most results below are stated for an arbitrary coupling $\pi\in\Pi(\mu,\nu)$. When $\pi$ is an optimal plan, these results apply to exact optimal transport. If there exists a Borel map $T:\sM\to\sM$ such that $\pi=(\id,T)_\#\mu$, we say that $\pi$ is {map-induced} or {Monge-induced}. In the Monge-induced case, the plan does not split mass and  each source point is assigned a single destination.

Fix a coupling $\pi\in\Pi(\mu,\nu)$. By disintegration with respect to its first marginal, there exists a Borel probability kernel $x\mapsto \pi(\cdot\mid x)$, unique for $\mu$-almost every $x$, such that
\[
\pi(dx,dy)=\mu(dx)\,d\pi(y\mid x).
\]
We call $d\pi(y\mid x)$ the {conditional destination law} at $x$. It describes how the mass located at $x$ is distributed among destinations by the coupling $\pi$. The plan-to-map problem is to replace this conditional probability measure by a single point of $\sM$, in a way that is geometrically meaningful and stable.

\subsection{Intrinsic and tangential projections}

The intrinsic construction starts from the conditional Fr\'echet functional
\[
F_x:\sM\to\R\cup\{+\infty\},
\qquad
F_x(z):=\frac12\int_{\sM} d(z,y)^2\,d\pi(y\mid x).
\]
Its minimizers are the Fr\'echet means of the conditional destination law.

\begin{definition}[Intrinsic barycentric projection]
The intrinsic barycentric set of $\pi$ at $x$ is
\[
\mathcal B_\pi(x):=\argmin_{z\in\sM} F_x(z).
\]
A Borel map $B_\pi:\sM\to\sM$ satisfying
\[
B_\pi(x)\in\mathcal B_\pi(x)
\quad\text{for }\mu\text{-almost every }x
\]
is called an {intrinsic barycentric projection} of $\pi$. When $\mathcal B_\pi(x)$ is a singleton for $\mu$-almost every $x$, we use $B_\pi(x)$ for its unique element.
\end{definition}

We also define the pointwise conditional Fr\'echet variance
\[
V_\pi(x):=\inf_{z\in\sM}F_x(z),
\]
and the integrated value
\[
V(\pi):=\int_{\sM} V_\pi(x)\,d\mu(x).
\]
The main result of Section~\ref{sec:intrinsic} shows that $V(\pi)$ is exactly the minimum deterministic representation error.

The second construction averages displacement vectors. It is only defined where logarithms are single-valued and square-integrable.

\begin{definition}[Tangential domain and projection]
Define
\[
\mathrm{Dom}_2(\widetilde B_\pi)
:=
\left\{
 x\in\sM:
 \begin{array}{l}
 y\mapsto\log_x(y)\text{ is }\pi(\cdot\mid x)\text{-almost surely single-valued,}\\[0.3em]
 \displaystyle \int_{\sM}\|\log_x(y)\|_x^2\,d\pi(y\mid x)<\infty
 \end{array}
\right\}.
\]
For $x\in\mathrm{Dom}_2(\widetilde B_\pi)$, define
\[
 m_\pi(x):=\int_{\sM}\log_x(y)\,d\pi(y\mid x)\in T_x\sM
\]
and
\[
\widetilde B_\pi(x):=\exp_x(m_\pi(x)).
\]
We call $\widetilde B_\pi$ the {tangential barycentric projection} of $\pi$ on its domain.
\end{definition}

The intrinsic projection is an optimizer on the manifold. The tangential projection is a first-order displacement surrogate based at the source point. Section~\ref{sec:tangent} makes this distinction precise.

\subsection{Losses, metrics, and geometric regimes}

For a Borel map $T:\sM\to\sM$, define its plan-to-map loss relative to $\pi$ by
\[
\mathcal E_\pi(T)
:=
\frac12\int_{\sM\times\sM}d(T(x),y)^2\,d\pi(x,y).
\]
Equivalently,
\[
\mathcal E_\pi(T)=\int_{\sM} F_x(T(x))\,d\mu(x).
\]
This quantity measures how well the deterministic map $T$ represents the original possibly mass-splitting plan. The intrinsic projection will attain the smallest possible value of $\mathcal E_\pi$. For any map $T$, we define the excess deterministic representation error
\[
\Delta_\pi(T):=\mathcal E_\pi(T)-V(\pi).
\]
When $T_\#\mu\in\mathcal P_2(\sM)$, we also define the target mismatch
\[
\mathcal M_\pi(T):=\frac12 W_2^2(T_\#\mu,\nu),
\]
where $W_2$ is the quadratic Wasserstein distance induced by $d$. These quantities will be used in the theory and in the empirical validation.

Some results hold on any complete finite-dimensional Riemannian manifold, while uniqueness and quantitative energy gaps require additional geometry. To clarify the assumptions, we use the following two regimes.

\begin{description}
    \item[(H) Hadamard regime.] The manifold $\sM$ is complete, simply connected, and has nonpositive sectional curvature.
    \item[(L) Local Fr\'echet uniqueness regime.] For $\mu$-almost every $x$, the conditional Fr\'echet functional $F_x$ has a unique global minimizer. Moreover, this minimizer lies in a geodesically convex open set $U_x$ containing $\spt(\pi(\cdot\mid x))$, and $F_x$ is strictly geodesically convex on $U_x$.
\end{description}

Regime \textbf{(H)} gives a global theory. Regime \textbf{(L)} is a compact way to encode the usual small-support Fr\'echet/Karcher mean assumptions on a general manifold. Concrete sufficient conditions are given by standard Karcher--Afsari radius bounds involving injectivity and curvature. For simplicity, we do not need their explicit constants in the main statements \citep{karcher_1977_RiemannianCenterMass,afsari_2011_Riemannian$L_p$Center}.  The same distinction will reappear algorithmically. In the Hadamard regime the row-wise objectives are globally geodesically convex, while in the local regime optimization must remain inside a convexity neighborhood. Standard convergence theory for first-order methods on geodesically convex functions is available in the global case \citep{zhang_2016_FirstorderMethodsGeodesically}.

\section{Intrinsic barycentric projection}
\label{sec:intrinsic}

This section shows that the intrinsic barycentric projection is the variational deterministic representative of a transport coupling. The proofs are stated for an arbitrary coupling $\pi\in\Pi(\mu,\nu)$, not only for an optimal plan. Throughout this section, equalities and inclusions involving maps are understood $\mu$-almost everywhere unless otherwise stated.

\subsection{Existence, compactness, and measurable selection}

The first task is purely foundational. Before asking whether an intrinsic barycentric projection is optimal, we must know that the conditional Fr\'echet problems have minimizers and that these minimizers can be assembled into a measurable map. The next proposition gives the pointwise existence statement. The following one turns the pointwise minimizers into an admissible deterministic representative.

\begin{proposition}[Existence and compactness of conditional Fr\'echet means]
\label{prop:existence_compactness}
Let $\mu,\nu\in\mathcal P_2(\sM)$, and let $\pi\in\Pi(\mu,\nu)$. Then there is a Borel set $\sX_0\subset\sM$ with $\mu(\sX_0)=1$ such that for every $x\in \sX_0$, the function $F_x$ is finite-valued, continuous, and coercive on $\sM$. Consequently,
\[
\mathcal B_\pi(x)=\argmin_{z\in\sM}F_x(z)
\]
is nonempty and compact for every $x\in \sX_0$.
\end{proposition}

Pointwise existence is not yet enough for a plan-to-map theorem, because the representative must be a measurable function of the source variable. The compactness of the argmin sets makes the standard measurable-selection machinery applicable.

\begin{proposition}[Borel measurable selection]
\label{prop:borel_selection}
Under the assumptions of \Cref{prop:existence_compactness}, there exists a Borel map $B_\pi:\sM\to\sM$ such that
\[
B_\pi(x)\in\mathcal B_\pi(x)
\quad\text{for }\mu\text{-almost every }x.
\]
Thus intrinsic barycentric projections exist. Under either regime \textbf{(H)} or regime \textbf{(L)}, the selector is unique up to $\mu$-almost everywhere equality.
\end{proposition}

The selected projection will later be pushed forward by $\mu$ and compared to the target law in $W_2$. The next lemma records the required integrability. It is a small point, but it prevents the target-mismatch bound from hiding an implicit moment assumption.

\begin{lemma}[Second moment of intrinsic barycentric projections]
\label{lem:second_moment_B}
Let $B_\pi$ be an intrinsic barycentric projection. Then $(B_\pi)_\#\mu\in\mathcal P_2(\sM)$. Moreover, $V(\pi)<\infty$.
\end{lemma}

\subsection{Variational optimality and Monge defect}

Let $\mathcal T_\mu$ be the class of Borel maps $T:\sM\to\sM$, identified up to $\mu$-almost everywhere equality. The loss $\mathcal E_\pi(T)$ is allowed to take the value $+\infty$. With the measurability and integrability issues settled, the deterministic representation problem separates cleanly across source points: for each $x$, the best value of $T(x)$ is precisely a minimizer of $F_x$. Integrating this pointwise statement gives the main theorem.

\begin{theorem}[Best deterministic representative]
\label{thm:best_deterministic}
For any $\pi\in\Pi(\mu,\nu)$,
\[
\inf_{T\in\mathcal T_\mu}\mathcal E_\pi(T)=V(\pi).
\]
Moreover, a map $T\in\mathcal T_\mu$ attains this infimum if and only if
\[
T(x)\in\mathcal B_\pi(x)
\quad\text{for }\mu\text{-almost every }x.
\]
In particular, every intrinsic barycentric projection $B_\pi$ minimizes $\mathcal E_\pi$.
\end{theorem}

This theorem is the intrinsic analogue of the Euclidean barycentric projection identity. Its simplicity is part of the point. Once the correct conditional objective is chosen, the global plan-to-map problem is exactly the integral of pointwise Fr\'echet mean problems.

When $\sM=\R^d$, the minimizer of $z\mapsto \frac12\int\|z-y\|^2\,d\pi(y\mid x)$ is the conditional expectation $\int y\,d\pi(y\mid x)$. In that case \Cref{thm:best_deterministic} is the usual $L^2$ optimality of conditional expectation. The contribution here is not to rediscover this Euclidean fact, but to formulate its manifold analogue with Fr\'echet means, measurable selections, curvature-dependent uniqueness, and a geometric variance interpretation.

\begin{definition}[Conditional-variance Monge defect]
\label{def:monge_defect}
The quantity
\[
V_\pi(x)=\inf_{z\in\sM}\frac12\int_{\sM} d(z,y)^2\,d\pi(y\mid x)
\]
is called the {conditional Fr\'echet variance} of $\pi$ at $x$. The integrated value
\[
V(\pi)=\int_{\sM} V_\pi(x)\,d\mu(x)
\]
is called the {conditional-variance Monge defect} of $\pi$.
\end{definition}

By \Cref{thm:best_deterministic}, this defect is exactly the irreducible squared-geodesic error incurred when a possibly mass-splitting coupling is forced to be represented by a deterministic map. The next two consequences check that this quantity has the desirable behavior in the deterministic limit. First, if the coupling was already induced by a map, barycentric projection recovers that map. Second, zero defect characterizes precisely this situation.

\begin{corollary}[Monge consistency]
\label{cor:monge_consistency}
If $\pi=(\id,T)_\#\mu$ for some Borel map $T:\sM\to\sM$, then
\[
\mathcal B_\pi(x)=\{T(x)\}
\quad\text{for }\mu\text{-almost every }x.
\]
Consequently, every intrinsic barycentric projection satisfies $B_\pi=T$ $\mu$-almost everywhere.
\end{corollary}

\begin{theorem}[Characterization of map-induced plans]
\label{thm:monge_characterization}
For $\pi\in\Pi(\mu,\nu)$, the following are equivalent:
\begin{enumerate}[label=(\roman*)]
    \item $V(\pi)=0$.
    \item $V_\pi(x)=0$ for $\mu$-almost every $x$.
    \item There exists a Borel map $T:\sM\to\sM$ such that
    \[
    d\pi(y\mid x)=\delta_{T(x)}(dy)
    \quad\text{for }\mu\text{-almost every }x.
    \]
    \item $\pi=(\id,T)_\#\mu$ for some Borel map $T:\sM\to\sM$.
\end{enumerate}
\end{theorem}

The preceding results measure fidelity to the original coupling. A second question is whether the deterministic representative still sends the source distribution near the target distribution. The next bound will answer this directly, which states that the deterministic pushforward cannot be farther from the target, in quadratic Wasserstein cost, than the conditional variance  lost by collapsing the plan to a map.

\begin{proposition}[Target mismatch bound]
\label{prop:target_mismatch}
Let $B_\pi$ be an intrinsic barycentric projection. Then
\[
\frac12 W_2^2((B_\pi)_\#\mu,\nu)\le V(\pi).
\]
\end{proposition}

\subsection{Hadamard strengthening}

The results so far use completeness, finite second moments, and measurable selection. Yet they do not use curvature except through possible uniqueness assumptions. Nonpositive curvature adds a genuinely stronger statement. In the Hadamard regime, barycenters satisfy a variance inequality, and after conditioning this becomes a quantitative energy gap for the whole plan-to-map problem \citep{sturm_2003_ProbabilityMeasuresMetric}. 

\begin{theorem}[Hadamard energy gap]
\label{thm:hadamard_gap}
Assume regime \textbf{(H)}. Then for every $T\in\mathcal T_\mu$,
\[
\mathcal E_\pi(T)-V(\pi)
\ge
\frac12\int_{\sM} d(T(x),B_\pi(x))^2\,d\mu(x).
\]
In particular, $B_\pi$ is the unique minimizer of $\mathcal E_\pi$ up to $\mu$-almost everywhere equality.
\end{theorem}

The energy gap immediately turns variational near-optimality into metric closeness to the intrinsic projection. A notable consequence is as follows. Under regime \textbf{(H)}, if $T_n\in\mathcal T_\mu$ satisfies $\mathcal E_\pi(T_n)\to V(\pi)$, then 
\begin{equation*}
    \int_{\sM} d(T_n(x),B_\pi(x))^2\,d\mu(x)\to 0.
\end{equation*}
This establishes stability of approximate deterministic minimizers, which is immediate from \Cref{thm:hadamard_gap}.

\section{Tangential projection as a displacement surrogate}
\label{sec:tangent}

The intrinsic projection is the variational solution of the plan-to-map problem. The tangential projection has a different role: it averages displacement coordinates at the source point. This section makes that role precise. The results below do not claim that the tangential projection is generated by an optimal potential for a mass-splitting plan. Rather, they show that it is compatible with the Brenier--McCann displacement representation in the Monge case and is a first-order surrogate for the intrinsic Fr\'echet objective.

\subsection{Tangent-space barycenter and measurability}

For the intrinsic projection, measurability was the main technical issue. For the tangential projection, the first issue is domain: the logarithm map may be multivalued or undefined on the cut locus. In the Hadamard regime this obstruction disappears, which gives the cleanest setting in which to treat the tangential projection as an honest map.

\begin{proposition}[Well-posedness in the Hadamard regime]
\label{prop:tangent_hadamard}
Assume regime \textbf{(H)}. Then $\mathrm{Dom}_2(\widetilde B_\pi)$ has full $\mu$-measure.
\end{proposition}

Full-measure well-posedness is still not enough if we want to insert $\widetilde B_\pi$ into the same losses used for deterministic maps. The following lemma records the corresponding Borel measurability statement in the global nonpositive-curvature setting.

\begin{lemma}[Measurability of the tangential projection in the Hadamard regime]
\label{lem:tangent_measurable}
Assume regime \textbf{(H)}. Then $\mathrm{Dom}_2(\widetilde B_\pi)$ contains a Borel full-$\mu$-measure set on which $x\mapsto m_\pi(x)$ is a Borel section of the tangent bundle $T\sM$, and $x\mapsto\widetilde B_\pi(x)$ is Borel. After arbitrary extension outside this set, $\widetilde B_\pi$ may be treated as an element of $\mathcal T_\mu$.
\end{lemma}

Once the logarithms are well-defined, the tangential construction has a very simple meaning as an ordinary Hilbert-space barycenter, but in the tangent space at the source point. This is the next proposition.

\begin{proposition}[Tangential variational characterization]
\label{prop:tangent_variational}
Let $x\in\mathrm{Dom}_2(\widetilde B_\pi)$, and define
\[
G_x(v):=\frac12\int_{\sM}\|v-\log_x(y)\|_x^2\,d\pi(y\mid x),
\qquad v\in T_x\sM.
\]
Then $G_x$ is strictly convex and has the unique minimizer $m_\pi(x)$. More precisely,
\[
G_x(v)=G_x(m_\pi(x))+\frac12\|v-m_\pi(x)\|_x^2.
\]
\end{proposition}

Thus the tangential projection is the Euclidean barycenter of the log-mapped conditional destination law in $T_x\sM$, pushed back to the manifold by $\exp_x$. This description also explains why there is no distinction between the intrinsic and tangential projections in flat space.

\begin{corollary}[Euclidean exactness]
\label{cor:euclidean_exactness}
If $\sM=\R^d$ with its Euclidean metric, then
\[
\widetilde B_\pi(x)=B_\pi(x)
\quad\text{for }\mu\text{-almost every }x.
\]
More explicitly,
\[
\widetilde B_\pi(x)=\int_{\R^d}y\,d\pi(y\mid x).
\]
\end{corollary}

\subsection{Displacement compatibility and gradient update}

Euclidean exactness is a flat-space sanity check. The next result shows that
the tangential projection also matches the displacement representation of
quadratic optimal maps on manifolds. In the Monge case, it recovers the usual
Brenier--McCann displacement field. For mass-splitting plans, it averages the
corresponding logarithmic displacement coordinates row by row.

\begin{proposition}[Monge compatibility of the tangential projection]
\label{prop:tangent_monge}
Assume $\pi=(\id,T)_\#\mu$ for some Borel map $T:\sM\to\sM$, and assume
\[
T(x)\notin\Cut(x)
\quad\text{for }\mu\text{-almost every }x.
\]
Then
\[
m_\pi(x)=\log_x(T(x))
\quad\text{and}\quad
\widetilde B_\pi(x)=T(x)
\]
for $\mu$-almost every $x$. If, in addition,
\[
T(x)=\exp_x(-\nabla\psi(x))
\]
for a differentiable potential $\psi$ at $x$, then
\[
m_\pi(x)=-\nabla\psi(x)
\]
for $\mu$-almost every such $x$.
\end{proposition}

This is the precise sense in which the tangential projection is compatible with the Brenier--McCann displacement picture: for non-Monge plans it averages the logarithmic displacement vectors that collapse to the optimal displacement field when the conditional law is Dirac. It is not claimed to be a potential-generated optimal map for a mass-splitting plan.

The previous results identify the tangential projection as a displacement average. We now connect it back to the intrinsic objective. The connection is first-order where evaluating the gradient of the conditional Fr\'echet functional at the source point gives exactly the negative of the averaged log displacement. Thus the tangential projection is the first unit Riemannian gradient update for the intrinsic conditional Fr\'echet functional. This is an identity, not a global descent guarantee.

\begin{theorem}[Tangential projection as a unit gradient update]
\label{thm:tangent_gradient}
Let $x\in\sM$ satisfy
\[
\int_{\sM} d(x,y)^2\,d\pi(y\mid x)<\infty
\quad\text{and}\quad
\pi(\Cut(x)\mid x)=0.
\]
Then $F_x$ is differentiable at the point $x$, and
\[
\nabla F_x(x)=-\int_{\sM}\log_x(y)\,d\pi(y\mid x)=-m_\pi(x).
\]
Consequently,
\[
\widetilde B_\pi(x)=\exp_x(-\nabla F_x(x)).
\]
\end{theorem}

\begin{corollary}[Barycentric equation]
\label{cor:barycentric_equation}
Let $x\in\sM$, and suppose $z\in\sM$ satisfies
\[
\int_{\sM} d(z,y)^2\,d\pi(y\mid x)<\infty,
\qquad
\pi(\Cut(z)\mid x)=0.
\]
If $z\in\mathcal B_\pi(x)$ and $F_x$ is differentiable at $z$, then
\[
\int_{\sM}\log_z(y)\,d\pi(y\mid x)=0.
\]
In the Hadamard regime, the converse holds: if the above integral vanishes, then $z=B_\pi(x)$.
\end{corollary}
\Cref{thm:tangent_gradient} implies first unit \Frechet gradient update. Consider the Riemannian gradient update for $F_x$
\[
z_{k+1}=\exp_{z_k}(-\nabla F_x(z_k)),
\]
whenever the right-hand side is well defined. If $z_0=x$, then
$
z_1=\widetilde B_\pi(x)
$. Here the word ``update'' is rather intentional. A unit step need not decrease \(F_x\) without additional smoothness and step-size assumptions. This observation says that $\widetilde B_\pi(x)$ is the first-order log--exp surrogate for $B_\pi(x)$, not that it globally solves the intrinsic Fr\'echet problem.

We close this section with a failure case for the tangential projection even when the intrinsic projection is perfectly well-defined. Let $\sM=S^n$ be the unit sphere with the standard round metric, fix $x\in S^n$, and suppose
\[
d\pi(y\mid x)=\delta_{-x}(dy).
\]
Then
\[
\mathcal B_\pi(x)=\{-x\},
\]
but $\widetilde B_\pi(x)$ is not defined because $-x\in\Cut(x)$ and $\log_x(-x)$ is multivalued. This antipodal example is the main reason the tangential projection should not be presented as the global manifold analogue of barycentric projection. The intrinsic projection is the variational object whereas the tangential projection is a displacement-based surrogate that is exact in flat and Monge settings but local on general manifolds.

\section{Row-wise algorithms for discrete transport couplings}
\label{sec:algorithms}

This section translates the theory into algorithms to convert discrete couplings into deterministic representatives on the manifold. We assume that a coupling matrix has already been computed, for example by a linear programming or min-cost flow solver. The key point is that no global optimization over maps remains after the coupling is fixed and everything decomposes over rows of the coupling matrix.

\subsection{Discrete plans and row-wise construction}

Suppose we have two empirical measures
\[
\mu=\sum_{i=1}^n a_i\delta_{x_i},
\qquad
\nu=\sum_{j=1}^m b_j\delta_{y_j},
\]
where $x_i,y_j\in\sM$, $a_i>0$, $b_j>0$, and $\sum_i a_i=\sum_j b_j=1$. Let $\va=(a_i)_{i=1}^n$ and $\vb=(b_j)_{j=1}^m$. Let
\[
\Pi(\va,\vb):=\{\mLambda\in\R_+^{n\times m}:\mLambda\vone_m=\va,\ \mLambda^\top\vone_n=\vb\}.
\]
An exact discrete optimal transport plan is any matrix
\[
\mGamma=(\gamma_{ij})\in\argmin_{\mLambda\in\Pi(\va,\vb)}
\frac12\sum_{i=1}^n\sum_{j=1}^m\mLambda_{ij}d(x_i,y_j)^2.
\]
More generally, the row-wise formulas below apply to any $\mGamma\in\Pi(\va,\vb)$. The associated coupling is
\[
\pi_{\mGamma}=\sum_{i=1}^n\sum_{j=1}^m\gamma_{ij}\delta_{(x_i,y_j)}.
\]
For each source atom define row weights
\[
w_{ij}:=\frac{\gamma_{ij}}{a_i},
\qquad
\sum_{j=1}^m w_{ij}=1.
\]
The conditional destination law at $x_i$ is
\[
\pi_{\mGamma}(\cdot\mid x_i)=\sum_{j=1}^m w_{ij}\delta_{y_j}.
\]
Define the row-wise Fr\'echet objective
\[
\Phi_i(z):=\frac12\sum_{j=1}^m w_{ij}d(z,y_j)^2.
\]
This is the discrete version of the conditional Fr\'echet functional. The following proposition is simply the disintegration formula written row by row, but it is the algorithmic reason the framework is practical for finite couplings.

\begin{proposition}[Row-wise decomposition]
\label{prop:rowwise_decomposition}
For the discrete coupling $\pi_{\mGamma}$,
\[
F_{x_i}(z)=\Phi_i(z)
\quad\text{for all }z\in\sM.
\]
Moreover, for any Borel map $T:\sM\to\sM$,
\[
\mathcal E_{\pi_{\mGamma}}(T)=\sum_{i=1}^n a_i\Phi_i(T(x_i)).
\]
Consequently,
\[
B_{\pi_{\mGamma}}(x_i)\in\argmin_{z\in\sM}\Phi_i(z),
\]
and, whenever the relevant logarithms are defined,
\[
\widetilde B_{\pi_{\mGamma}}(x_i)
=
\exp_{x_i}\left(\sum_{j=1}^m w_{ij}\log_{x_i}(y_j)\right).
\]
\end{proposition}

\subsection{Algorithms}

After \Cref{prop:rowwise_decomposition}, computing the intrinsic projection means solving one weighted Fr\'echet mean problem for each source atom. These problems are independent and can be solved in parallel. For each row $i$, the intrinsic projection solves
\[
B_i\in\argmin_{z\in\sM}\Phi_i(z),
\qquad
\Phi_i(z)=\frac12\sum_{j=1}^m w_{ij}d(z,y_j)^2.
\]
Whenever $y_j\notin\Cut(z)$ for every active target $j$ with $w_{ij}>0$, the gradient is
\[
\nabla\Phi_i(z)=-\sum_{j=1}^m w_{ij}\log_z(y_j).
\]
Thus a basic Riemannian gradient method has the form
\[
z_i^{(k+1)}
=
\exp_{z_i^{(k)}}\left(-\eta_{i,k} \nabla\Phi_i\left(z_i^{(k)}\right)\right)
=
\exp_{z_i^{(k)}}\left(\eta_{i,k}\sum_{j=1}^m w_{ij}\log_{z_i^{(k)}}(y_j)\right),
\]
with a step size $\eta_{i,k}>0$. In the Hadamard regime the row objective is globally geodesically convex, and standard results on first-order methods for geodesically convex optimization apply under the usual step-size and regularity assumptions \citep{zhang_2016_FirstorderMethodsGeodesically}. In the local regime, convergence is only a local statement. The solver should be initialized inside the relevant convexity neighborhood and use a step-size or projection strategy that keeps iterates in that region. A full numerical convergence theory in the local positive-curvature regime is outside the scope of this paper.

The tangential projection avoids the row-wise minimization and instead evaluates the first log--exp update directly at the source atom. This gives a simple companion to the intrinsic solver. The tangential projection is obtained by one log--exp average. For every row $i$ such that $y_j\notin\Cut(x_i)$ for all active $j$, compute
\[
v_i:=\sum_{j=1}^m w_{ij}\log_{x_i}(y_j)
\]
and output
\[
\widetilde B_i:=\exp_{x_i}(v_i).
\]

The tangential update is exactly the first unit gradient update for the row objective initialized at the source atom.

\begin{corollary}[Tangential update as first intrinsic update]
\label{cor:tangent_first_discrete}
Consider the gradient update for $\Phi_i$ with initialization $z_i^{(0)}=x_i$ and first step size $\eta_{i,0}=1$. If the tangential projection is defined at row $i$, then
\[
z_i^{(1)}=\widetilde B_{\pi_{\mGamma}}(x_i).
\]
\end{corollary}

\begin{figure}[t]
\centering

\begin{minipage}[t]{0.48\linewidth}
\begin{algorithm}[H]
\caption{Intrinsic projection from a discrete coupling}
\label{alg:intrinsic_discrete}
\begin{algorithmic}[1]
\Require Coupling matrix $\mGamma$, source weights $(a_i)_{i=1}^n$,
target points $(y_j)_{j=1}^m$
\Ensure Intrinsic representatives $(B_i)_{i=1}^n$
\For{$i=1,\ldots,n$}
    \State $J_i \gets \{j:\gamma_{ij}>0\}$
    \State $w_{ij} \gets \gamma_{ij}/a_i$ for $j\in J_i$
    \State Choose an initialization $z\gets z_i^{(0)}$
    \Repeat
        \State $u \gets \displaystyle\sum_{j\in J_i} w_{ij}\log_z(y_j)$
        \State Choose a step size $\eta>0$
        \State $z \gets \exp_z(\eta u)$
    \Until{converged}
    \State $B_i \gets z$
\EndFor
\State \Return $(B_i)_{i=1}^n$
\end{algorithmic}
\end{algorithm}
\end{minipage}
\hfill
\begin{minipage}[t]{0.48\linewidth}
\begin{algorithm}[H]
\caption{Tangential projection from a discrete coupling}
\label{alg:tangential_discrete}
\begin{algorithmic}[1]
\Require Coupling matrix $\mGamma$, source weights $(a_i)_{i=1}^n$,
source points $(x_i)_{i=1}^n$, target points $(y_j)_{j=1}^m$
\Ensure Tangential representatives $(\widetilde B_i)_{i=1}^n$ where defined
\For{$i=1,\ldots,n$}
    \State $J_i \gets \{j:\gamma_{ij}>0\}$
    \State $w_{ij} \gets \gamma_{ij}/a_i$ for $j\in J_i$
    \If{$y_j\notin\Cut(x_i)$ for all $j\in J_i$}
        \State $v_i \gets \displaystyle\sum_{j\in J_i} w_{ij}\log_{x_i}(y_j)$
        \State $\widetilde B_i \gets \exp_{x_i}(v_i)$
    \Else
        \State Mark $\widetilde B_i$ as undefined
    \EndIf
\EndFor
\State \Return $(\widetilde B_i)_{i=1}^n$
\end{algorithmic}
\end{algorithm}
\end{minipage}
\end{figure}

Algorithms for the intrinsic projection and tangential projection are summarized in \Cref{alg:intrinsic_discrete} and \Cref{alg:tangential_discrete}, respectively. A practical hybrid approach is to use the tangential projection as a high-fidelity initialization for the intrinsic solver. First, compute $\widetilde B_i$ by \Cref{alg:tangential_discrete} on rows where it is defined. Then run \Cref{alg:intrinsic_discrete} initialized at $z_i^{(0)}=\widetilde B_i$. This preserves the intrinsic objective while using the tangential point as an initialization. Whether this reduces runtime or iteration count is an empirical question left to \Cref{sec:experiments}. If a row is deterministic, i.e. if $\gamma_{ij(i)}=a_i$ for a unique $j(i)$, then the intrinsic projection returns $y_{j(i)}$. The tangential projection also returns $y_{j(i)}$ whenever $y_{j(i)}\notin\Cut(x_i)$. Otherwise, the tangential projection is not defined at that row.

\subsection{Stability with respect to the coupling matrix}

The final result in this section records a basic robustness property. If the support points are fixed and only the coupling weights are perturbed, then the row-wise objectives vary continuously. Under uniqueness of the limiting row mean, this continuity passes to the intrinsic projection. The tangential formula is continuous whenever the relevant logarithms are fixed and single-valued.

\begin{proposition}[Discrete stability]
\label{prop:discrete_stability}
Fix support points $x_1,\dots,x_n$ and $y_1,\dots,y_m$, and fix positive source weights $a_i$. Let $\mGamma^{(r)}\to\mGamma$ entrywise, with all matrices belonging to $\Pi(\va,\vb)$. For a fixed row $i$, define
\[
w_{ij}^{(r)}:=\frac{\gamma_{ij}^{(r)}}{a_i},
\qquad
w_{ij}:=\frac{\gamma_{ij}}{a_i},
\]
and row objectives
\[
\Phi_i^{(r)}(z):=\frac12\sum_j w_{ij}^{(r)}d(z,y_j)^2,
\qquad
\Phi_i(z):=\frac12\sum_j w_{ij}d(z,y_j)^2.
\]
If $\Phi_i$ has a unique minimizer $B_i$, and $B_i^{(r)}$ is any minimizer of $\Phi_i^{(r)}$, then
\[
B_i^{(r)}\to B_i.
\]
If, in addition, $y_j\notin\Cut(x_i)$ for every $j$, then
\[
\widetilde B_i^{(r)}\to \widetilde B_i.
\]
\end{proposition}

After the coupling is known, the row problems are independent and can be solved in parallel. The tangential method requires one logarithm evaluation per active coupling entry and one exponential map evaluation per source atom. The intrinsic method requires analogous logarithm and exponential evaluations at each Riemannian optimization iteration. If $K_i$ iterations are used for row $i$, then the total number of row-wise logarithm evaluations scales as
\[
\sum_{i=1}^n K_i\,\#\{j:\gamma_{ij}>0\}.
\]
This is an operation count, not an empirical runtime claim. Runtime depends on the manifold, the implementation of exponential and logarithm maps, the optimization method, and the sparsity of the coupling.

\section{Experiments}
\label{sec:experiments}

The theory identifies several quantities that should be visible in finite-sample computations. Some are theorem-level sanity checks that the intrinsic projection should have zero excess plan-fit energy, the deterministic target mismatch should be bounded by the Monge defect, and the Hadamard energy-gap ratio should not exceed one on nonpositively curved targets. Other quantities are empirical, including the size of the tangential approximation error, the practical benefit of tangential initialization, and the effect of the plan-to-map rule on a downstream manifold-valued learning pipeline. The experiments below separate these two roles. We use exact discrete OT throughout, keep the coupling fixed, and vary only the deterministic representative extracted from that coupling.

In every experiment, the source and target empirical measures have unequal support sizes and uniform weights, with the target support three times larger than the source support. 
Thus each source row of the exact plan must split its mass across target atoms that in all runs the average row split score is \(2/3\) and the average effective row
support is \(3\). We compare the intrinsic projection $T_{\mathrm{int}}$, the tangential projection $T_{\mathrm{tan}}$, a maximum-row-mass map $T_{\mathrm{MAP}}$, the identity map $T_{\mathrm{id}}$, and an ambient or log-Euclidean baseline when available. We report means $\pm$ standard errors over independent configurations or subject/split pairs.

\subsection{Sphere-valued data}

The first experiment uses empirical measures on \(S^2\) with \(32\) source atoms
and \(96\) target atoms. We vary the separation between the source and target
caps, using target separations \(0.5,1.2,2.4,2.9\) radians and cap radii
\(0.15\) and \(0.35\), over four random seeds. The purpose is not to benchmark
a downstream task, but to isolate the geometric difference between intrinsic
and tangential averaging on a positively curved space.

\begin{figure}[t]
\centering
\resizebox{0.98\linewidth}{!}{%
	\input{report/eeg_spd/eeg_geometric_metrics.pgf}
}
\caption{Sphere experiment. Normalized excess plan-fit energy as a function of
target separation for two cap radii. Panel (A) uses cap radius \(0.15\), and
panel (B) uses cap radius \(0.35\). The vertical axis is shown on a symmetric
logarithmic scale so that near-zero excess values and large identity-map
errors are visible in the same plot.}
\label{fig:sphere-excess}
\end{figure}
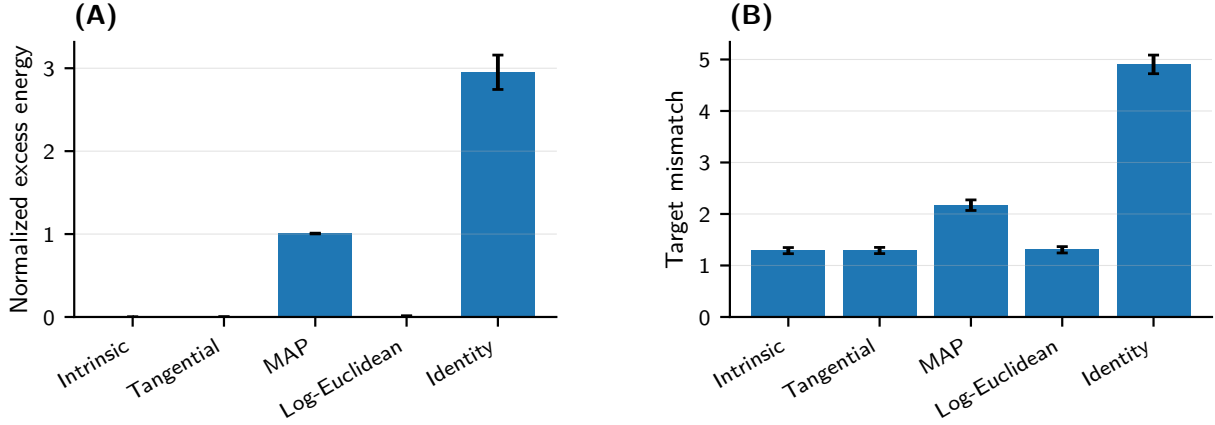

Across the \(32\) sphere configurations, the intrinsic projection has
plan-fit energy equal to the Monge defect, with
\[
V(\pi)=8.14\times 10^{-4}\pm 1.10\times 10^{-4},
\]
where we report mean \(\pm\) standard error over configurations. The ambient
normalized mean is essentially identical to the intrinsic projection in this
particular cap construction, with normalized excess
\(1.47\times 10^{-7}\pm 4.34\times 10^{-8}\). This should be read as a
property of the controlled spherical sampling scheme, not as a general
equivalence between ambient and intrinsic barycenters. The MAP rule has
normalized excess \(0.904\pm0.023\), and the identity map has much larger
normalized excess, \(3.99\times 10^3\pm8.21\times10^2\); see
\Cref{fig:sphere-excess}.

The tangential projection is defined in all sphere runs, but its approximation
quality is sensitive to configurations near the edge of the local linearization
regime. Averaged over all configurations, its normalized excess is
\(1.79\pm1.75\), with the large standard error driven by the highest-separation
cases. This is consistent with the role of \(T_{\mathrm{tan}}\) as a local
displacement surrogate rather than a globally canonical barycenter.

\subsection{Synthetic SPD data}
The second experiment uses $\mathrm{SPD}(3)$ with the affine-invariant metric \citep{pennec_2006_RiemannianFrameworkTensor}, again with $32$ source atoms and $96$ target atoms. We vary a source--target displacement parameter $\beta\in\{0.3,0.8,1.5\}$ and a dispersion parameter in $\{0.15,0.45\}$ over four random seeds, for $24$ configurations. This is the cleanest numerical setting for the Hadamard part of the theory. Since the affine-invariant SPD manifold is nonpositively curved, the intrinsic projection is unique and the energy-gap theorem applies globally.

\begin{figure}[t]
\centering
\resizebox{0.98\linewidth}{!}{%
	\input{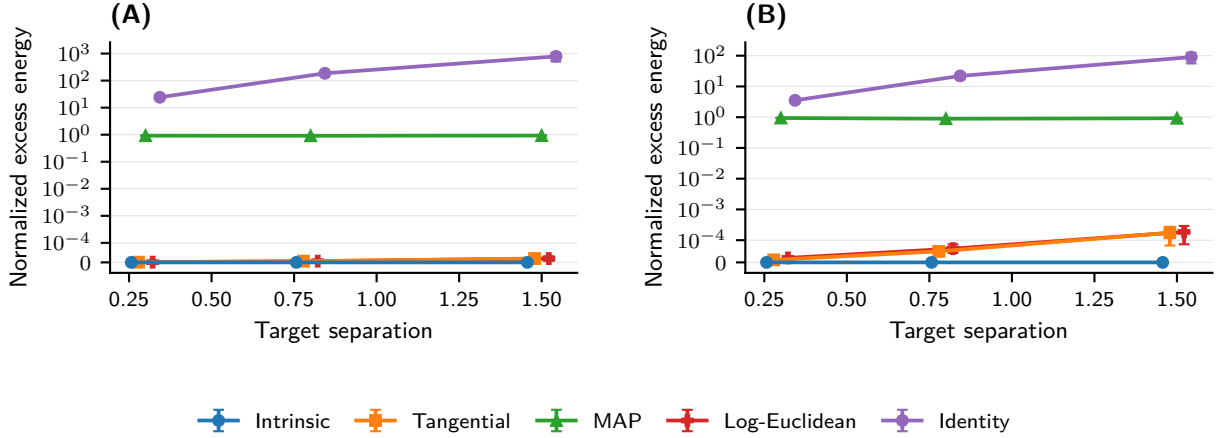}
}
\caption{Synthetic SPD experiment. Normalized excess plan-fit energy for $\mathrm{SPD}(3)$ under the affine-invariant metric.}
\label{fig:spd-excess}
\end{figure}

The intrinsic projection again has zero excess energy by definition and attains $V(\pi)=1.017\times10^{-2}\pm2.076\times10^{-3}$. In contrast to the sphere stress case, the tangential approximation is extremely close to intrinsic: its normalized excess is $4.36\times10^{-5}\pm1.93\times10^{-5}$, comparable to the log-Euclidean baseline at $4.92\times10^{-5}\pm2.06\times10^{-5}$. The MAP rule has normalized excess $0.919\pm0.012$, and the identity map has normalized excess $185.3\pm70.9$. Thus, in this Hadamard setting with moderate row dispersion, the log--exp surrogate captures almost all of the intrinsic plan-fit benefit while remaining conceptually distinct from the intrinsic Fr\'echet optimizer.

\begin{figure}[t]
\centering
\resizebox{0.98\linewidth}{!}{%
	\input{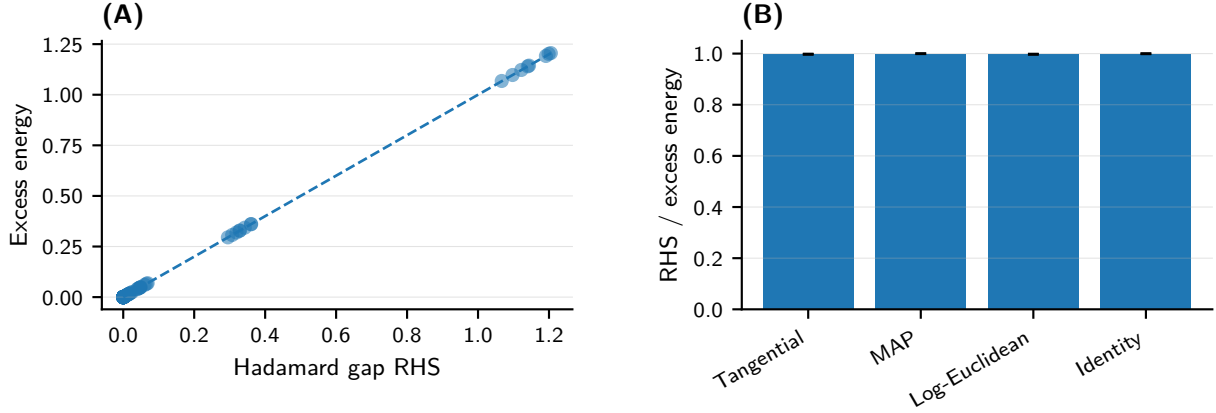}
}
\caption{Hadamard energy-gap sanity check on synthetic SPD data. The left panel plots the excess energy against the right-hand side of the Hadamard gap inequality. The right panel shows the ratio of the right-hand side to the excess energy for non-intrinsic methods.}
\label{fig:spd-gap}
\end{figure}

The Hadamard energy-gap inequality is also visible numerically. For each non-intrinsic representative $T$, the ratio between the gap lower bound and the measured excess energy is below one. It is $0.998\pm0.001$ for the tangential and log-Euclidean surrogates, $1.000\pm0.0001$ for MAP up to numerical precision, and $0.999\pm0.0001$ for identity. The near-tightness in \Cref{fig:spd-gap} reflects the fact that the synthetic rows are generated in localized regions where the affine-invariant geometry is close to its tangent approximation.

\begin{figure}[t]
\centering
\resizebox{0.58\linewidth}{!}{%
	\input{report/spd_synthetic/spd_intrinsic_iterations.pgf}
}
\caption{Intrinsic Fr\'echet solver iterations on synthetic SPD data. }
\label{fig:spd-iterations}
\end{figure}
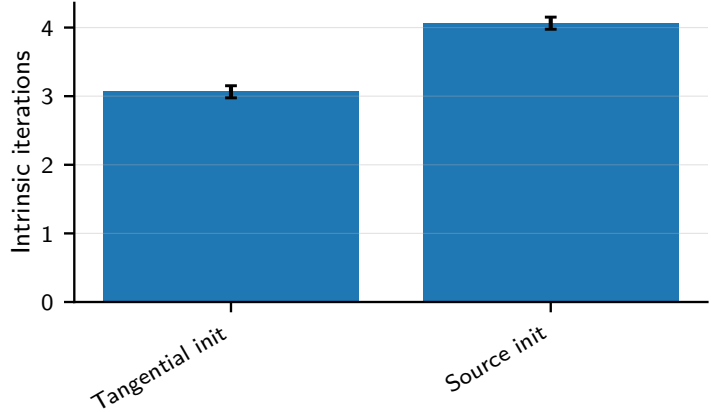

The row-wise optimization diagnostic supports the algorithmic interpretation of the tangential map as a first-order initialization. Starting the intrinsic Fr\'echet solver at the tangential projection requires $3.06\pm0.09$ iterations on average, compared with $4.06\pm0.09$ iterations from the source point, a reduction of about $25\%$ in this experiment as shown in \Cref{fig:spd-iterations}. Although it should be warned that this is not a universal convergence theorem, it empirically supports the proposed hybrid use of the tangential projection as a warm start.

\subsection{Real EEG covariance data}

The real-data experiment uses \textsf{BNCI2014\_001}, the four-class motor-imagery
dataset originally released as data set 2a of BCI Competition IV
\citep{brunner_2008_BCICompetition2008,tangermann_2012_ReviewBCICompetition,jayaram_2018_MOABBTrustworthyAlgorithm}. The
dataset contains EEG recordings from nine subjects performing cue-based motor
imagery of the left hand, right hand, both feet, and tongue. For each subject,
two sessions were recorded on different days. Each session consists of six runs
of 48 trials, with 12 trials per class per run, giving 288 trials per session.
The recordings contain 22 EEG channels and 3 EOG channels. We use only the 22
EEG channels, so each trial is represented by a regularized covariance matrix
in \(\mathrm{SPD}(22)\). Concretely, if \(E_i\in \R^{22\times T}\) denotes the preprocessed EEG epoch for trial \(i\), we estimate its covariance using the OAS estimator \citep{chen_2010_ShrinkageAlgorithmsMMSE} and then add a ridge term \(10^{-6}I_{22}\), followed by eigenvalue flooring at \(10^{-6}\). The resulting matrix \(X_i\in\mathrm{SPD}(22)\) is used with the affine-invariant distance for the OT cost and all Riemannian barycentric computations.

For each of the nine subjects, we perform three balanced splits. In each split,
\(64\) labeled source covariances are transported to an unlabeled target
adaptation pool of \(192\) covariances, while a disjoint held-out target test
set of \(96\) covariances is used only for downstream evaluation. The OT stage
is label-agnostic: the cost matrix, exact coupling, and plan-to-map
representatives use only covariance matrices and weights. Labels are used only
for balanced sampling, for training the downstream nearest-centroid classifier after source
mapping, and for held-out evaluation. These values are summarized visually in \Cref{fig:eeg-geometric} and numerically in \Cref{tab:eeg-results}.

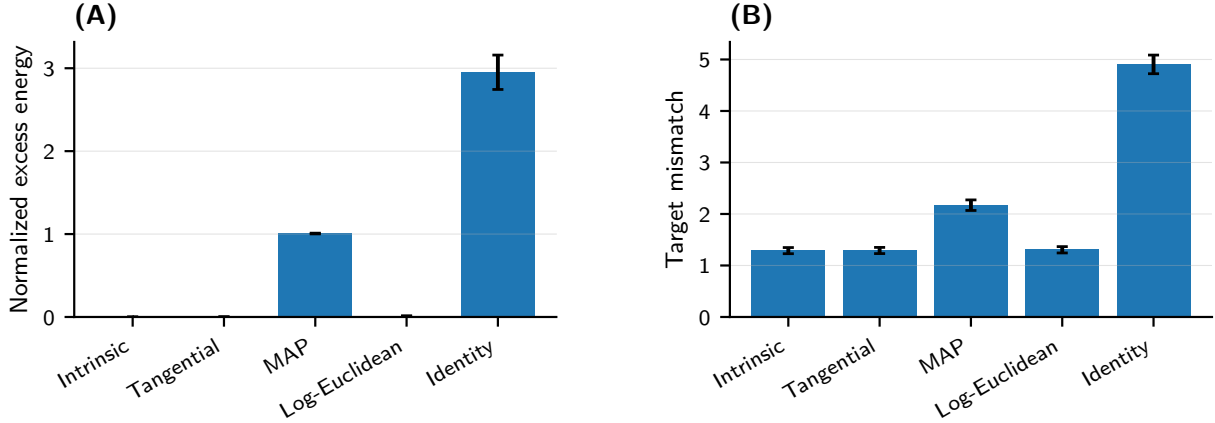
\begin{figure}[t]
\centering
\resizebox{0.98\linewidth}{!}{%
	\input{report/eeg_spd/eeg_geometric_metrics.pgf}
}
\caption{Geometric metrics on real EEG covariance data. The intrinsic projection has zero normalized excess by definition. The tangential projection is close to intrinsic, while the hard MAP and identity maps incur substantially larger excess and target mismatch.}
\label{fig:eeg-geometric}
\end{figure}

\begin{table}[t]
\caption{Real EEG covariance results on $\mathrm{SPD}(22)$. Values are mean $\pm$ standard error over $27$ subject/split pairs. $R_{\mathcal E}$ is normalized excess plan-fit energy, $\mathcal M_\pi$ is target mismatch, and smaller class-prototype alignment is better.}
\label{tab:eeg-results}
\centering
\small
\begin{tabular}{lcccc}
\toprule
Method & $R_{\mathcal E}$ & $\mathcal M_\pi$ & Balanced acc. & Class align. \\
\midrule
Intrinsic & $0.0000\pm0.0000$ & $1.288\pm0.060$ & $0.552\pm0.034$ & $0.601\pm0.023$ \\
Tangential & $0.0018\pm0.0002$ & $1.291\pm0.060$ & $0.550\pm0.034$ & $0.604\pm0.023$ \\
Log-Euclidean & $0.0119\pm0.0008$ & $1.305\pm0.061$ & $0.555\pm0.033$ & $0.613\pm0.023$ \\
MAP & $1.006\pm0.005$ & $2.171\pm0.103$ & $0.540\pm0.032$ & $0.700\pm0.023$ \\
Identity & $2.951\pm0.207$ & $4.905\pm0.181$ & $0.540\pm0.026$ & $1.785\pm0.078$ \\
\bottomrule
\end{tabular}
\end{table}

\begin{figure}[t]
\centering
\resizebox{0.98\linewidth}{!}{%
	\input{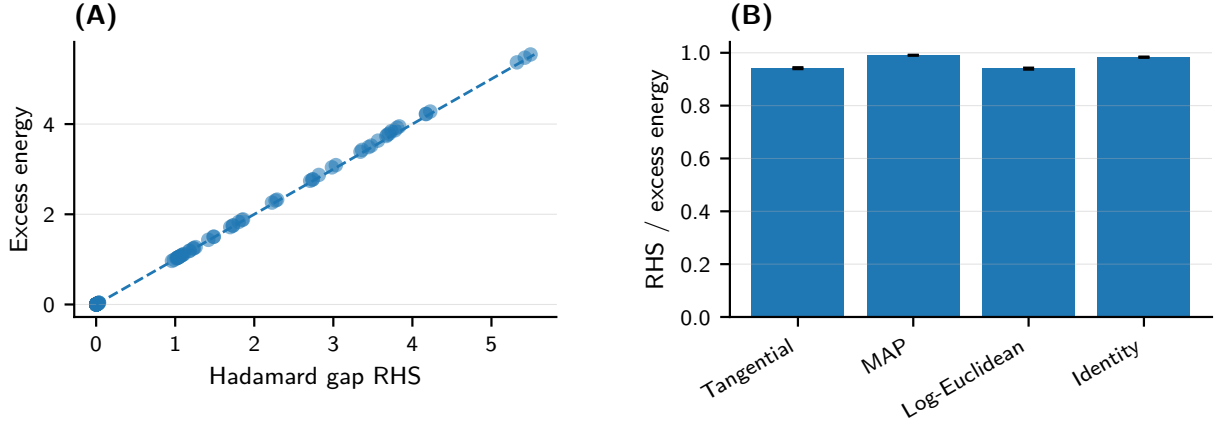}
}
\caption{Hadamard energy-gap sanity check on EEG covariance data. The gap lower bound stays below the measured excess energy. The inequality is theorem-guaranteed on the affine-invariant SPD manifold, so this figure should be read as an implementation-level sanity check rather than an empirical discovery.}
\label{fig:eeg-gap}
\end{figure}

The Hadamard gap check on EEG data is shown in \Cref{fig:eeg-gap}. The ratio between the right-hand side of the gap inequality and the measured excess energy remains below one for all non-intrinsic methods: $0.942\pm0.003$ for tangential, $0.940\pm0.003$ for log-Euclidean, $0.991\pm0.001$ for MAP, and $0.983\pm0.001$ for identity. The row-level diagnostics also show that the tangential approximation error grows with conditional spread: the Pearson correlation between the row Monge defect $V_i$ and $d(B_i,\widetilde B_i)^2$ is $0.84$ over the EEG rows. Since the split score is fixed by the $1{:}3$ source--target support design, the row-split diagnostic is primarily a sanity check that mass splitting was indeed enforced.

\begin{figure}[t]
\centering
\resizebox{0.98\linewidth}{!}{%
	\input{report/eeg_spd/eeg_downstream_metrics.pgf}
}
\caption{Downstream EEG covariance results. Balanced accuracy is similar across the manifold-aware representatives, while geometric class-prototype alignment improves substantially over the identity baseline. The downstream task is included to test whether the plan-to-map choice affects a real manifold-valued pipeline, not to claim state-of-the-art BCI performance.}
\label{fig:eeg-downstream}
\end{figure}
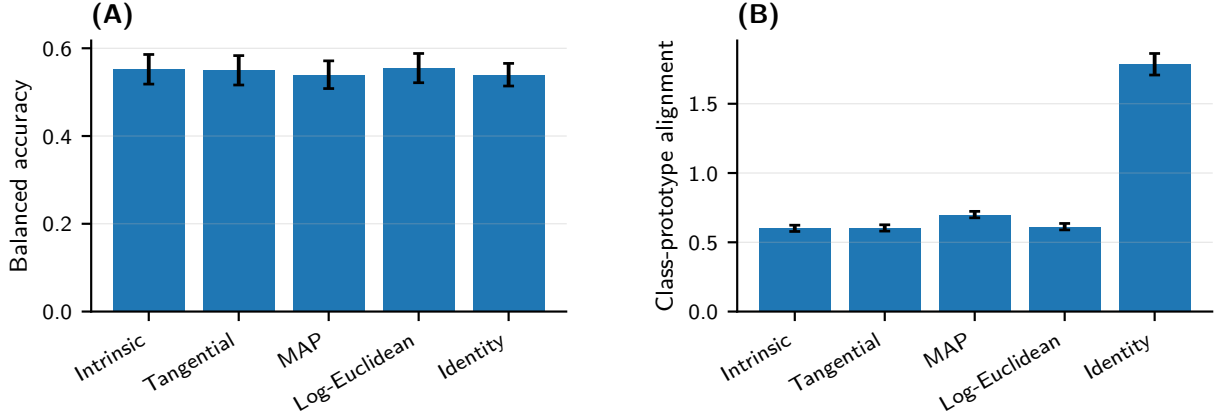

Finally, the downstream EEG evaluation gives a more cautious but useful picture. Balanced accuracy is similar across the three soft barycentric representatives: $0.552\pm0.034$ for intrinsic, $0.550\pm0.034$ for tangential, and $0.555\pm0.033$ for log-Euclidean. These differences are small relative to subject-level variation, and we do not claim a classification advantage for the intrinsic map. The geometric class-prototype alignment tells a clearer story: intrinsic, tangential, and log-Euclidean representatives all substantially improve alignment relative to identity, with intrinsic at $0.601\pm0.023$ versus $1.785\pm0.078$ for identity. MAP is intermediate at $0.700\pm0.023$. Thus, on real SPD-valued EEG data, the deterministic representative matters geometrically even when downstream classification accuracy is governed by additional factors such as class overlap and subject variability.

Taken together, the experiments support the paper's division of roles. The intrinsic projection is the accuracy-first deterministic representative of a coupling: it exactly minimizes the plan-fit objective and realizes the Monge defect. The tangential projection is a useful local surrogate: it is nearly intrinsic on the Hadamard SPD experiments and real EEG data, but its behavior is more delicate on the spherical stress test. Hard row assignment discards conditional information and consistently incurs larger geometric error. Downstream EEG results should be interpreted modestly: they show that the proposed maps are usable in a real manifold-valued learning pipeline and can improve geometric alignment, but they are not intended as a claim of state-of-the-art EEG classification performance.

\section{Conclusion}
\label{sec:conclusion}

We developed a plan-to-map framework for transport couplings on Riemannian
manifolds. The intrinsic barycentric projection maps each source point to the
conditional Fr\'echet mean of its destination law and is the best deterministic
representative of the coupling under squared geodesic loss. The minimum value
of this representation problem is the integrated conditional Fr\'echet
variance, which vanishes exactly for map-induced couplings and therefore gives
a conditional-variance Monge defect. This identifies, in intrinsic geometric
terms, the irreducible cost of replacing a possibly mass-splitting coupling by
a deterministic map.

The tangential projection plays a complementary role. It averages logarithmic
displacements in the tangent space at the source point and pushes the result
back to the manifold. This construction agrees with the intrinsic projection in
Euclidean space, is compatible with the Brenier--McCann displacement
representation in the Monge case, and equals the first unit Riemannian gradient
update for the intrinsic conditional Fr\'echet objective. Thus the intrinsic
projection is the variational representative, while the tangential
projection is a first-order displacement surrogate. For discrete couplings, both
constructions decompose row-wise, yielding direct algorithms based on weighted
Fr\'echet means and log--exp averages.

The experiments support this division of roles. Across the synthetic and EEG
settings, the intrinsic projection realizes the plan-fit optimum, the tangential
projection is close in local or Hadamard regimes but less reliable in the
spherical stress test, and hard row assignment consistently loses geometric
information. On EEG covariance data, the plan-to-map choice affects geometric
class alignment even when classification accuracy is dominated by subject
variability and class overlap.

The scope of the paper is intentionally focused. The theory applies to
finite-cost couplings, and the computational development emphasizes exact
discrete couplings processed row by row. Exact OT solvers are not as scalable
as entropic solvers for very large empirical measures. Extending the framework
to entropic plans is a natural next step. The intrinsic projection requires
solving one Fr\'echet mean problem per source atom, which is parallelizable but
iterative. On general manifolds, uniqueness of intrinsic barycenters is local
unless additional global curvature assumptions are imposed, and the tangential
projection can fail at the cut locus. The paper also leaves open several
directions such as out-of-sample extension, regularity of the map \(B_\pi\),
statistical rates for empirical plans, and quantitative curvature-controlled
bounds between \(B_\pi\) and \(\widetilde B_\pi\). These questions would further
connect the geometric plan-to-map viewpoint developed here with statistical
estimation and scalable manifold-valued learning.



\bibliography{references}
\bibliographystyle{tmlr}

\appendix
\section{Proofs of theoretical results}
\subsection{Proof of \Cref{prop:existence_compactness}}
\begin{proof}
Fix a reference point $o\in\sM$. Since $\nu\in\mathcal P_2(\sM)$,
\[
\int_{\sM}\left(\int_{\sM} d(o,y)^2\,d\pi(y\mid x)\right)d\mu(x)
=
\int_{\sM} d(o,y)^2\,d\nu(y)<\infty.
\]
Let
\[
\sX_0:=\left\{x\in\sM:\int_{\sM} d(o,y)^2\,d\pi(y\mid x)<\infty\right\}.
\]
The map inside the braces is Borel because $x\mapsto\pi(\cdot\mid x)$ is a Borel probability kernel. Hence $\sX_0$ is Borel and has full $\mu$-measure.

Fix $x\in \sX_0$. For any $z\in\sM$,
\[
d(z,y)^2\le 2d(z,o)^2+2d(o,y)^2,
\]
so $F_x(z)<\infty$. If $z_n\to z$, then for all sufficiently large $n$, $d(z_n,z)\le 1$, and therefore
\[
d(z_n,y)^2
\le 2d(z_n,o)^2+2d(o,y)^2
\le 2(d(z,o)+1)^2+2d(o,y)^2.
\]
The right-hand side is integrable with respect to $d\pi(y\mid x)$, and $d(z_n,y)^2\to d(z,y)^2$ pointwise. Dominated convergence gives $F_x(z_n)\to F_x(z)$. Thus $F_x$ is continuous.

For coercivity, the triangle inequality gives
\[
d(z,y)\ge d(z,o)-d(o,y),
\]
and hence
\[
d(z,y)^2\ge \frac12 d(z,o)^2-d(o,y)^2.
\]
Integrating yields
\[
F_x(z)\ge \frac14 d(z,o)^2-\frac12\int_{\sM} d(o,y)^2\,d\pi(y\mid x).
\]
Thus $F_x(z)\to\infty$ as $d(z,o)\to\infty$. Since $\sM$ is complete and finite-dimensional, Hopf--Rinow implies that closed bounded subsets of $\sM$ are compact. Therefore a continuous coercive function attains its minimum, and its minimizer set is compact.
\end{proof}
\subsection{Proof of \Cref{prop:borel_selection}}
\begin{proof}
Let $\sX_0$ be the full-measure Borel set from \Cref{prop:existence_compactness}. Define
\[
F:\sX_0\times\sM\to\R,
\qquad
F(x,z):=F_x(z).
\]
For each fixed $z$, $x\mapsto F(x,z)$ is Borel by the kernel property. For each fixed $x\in \sX_0$, $z\mapsto F(x,z)$ is continuous by \Cref{prop:existence_compactness}. Thus $F$ is a Carath\'eodory function and is jointly Borel.

Let $\sD\subset\sM$ be countable and dense. Since $F(x,\cdot)$ is continuous,
\[
V_\pi(x):=\inf_{z\in\sM}F(x,z)=\inf_{z\in\sD}F(x,z),
\]
so $V_\pi$ is Borel on $\sX_0$. The argmin graph is
\[
\Gamma_\pi:=\{(x,z)\in \sX_0\times\sM:F(x,z)=V_\pi(x)\}.
\]
It is Borel, and by \Cref{prop:existence_compactness} its sections
\[
(\Gamma_\pi)_x=\mathcal B_\pi(x)
\]
are nonempty compact subsets of $\sM$. Hence the argmin correspondence is weakly measurable: for every open set $\sO\subset\sM$, the set
\[
\{x\in \sX_0:\mathcal B_\pi(x)\cap \sO\neq\emptyset\}
\]
is the projection of the Borel set $\Gamma_\pi\cap (\sX_0\times \sO)$ and is Borel for compact-valued Borel correspondences on Polish spaces. Equivalently, one may apply the measurable maximum theorem. The Kuratowski--Ryll-Nardzewski selection theorem then yields a Borel selector on $\sX_0$. Extending the selector arbitrarily on $\sM\setminus \sX_0$ gives a Borel map on $\sM$.

In regime \textbf{(H)}, the conditional Fr\'echet functional is strictly geodesically convex. Indeed, in a Hadamard manifold the squared distance satisfies the CAT(0) convexity inequality along geodesics:
\[
d(\gamma_t,y)^2
\le (1-t)d(\gamma_0,y)^2+t d(\gamma_1,y)^2-t(1-t)d(\gamma_0,\gamma_1)^2.
\]
After integration against $d\pi(y\mid x)$, this implies strict geodesic convexity of $F_x$. Hence the minimizer is unique. In regime \textbf{(L)}, uniqueness is part of the standing local Fr\'echet uniqueness hypothesis.
\end{proof}

\subsection{Proof of \Cref{lem:second_moment_B}}
\begin{proof}
Let $o\in\sM$. Since $B_\pi(x)$ minimizes $F_x$,
\[
V_\pi(x)=F_x(B_\pi(x))\le F_x(o)=\frac12\int_{\sM} d(o,y)^2\,d\pi(y\mid x).
\]
Integrating in $x$ gives
\[
V(\pi)\le \frac12\int_{\sM} d(o,y)^2\,d\nu(y)<\infty.
\]
Next, for every $(x,y)$,
\[
d(o,B_\pi(x))^2\le 2d(B_\pi(x),y)^2+2d(o,y)^2.
\]
Integrating with respect to $d\pi(x,y)$ yields
\begin{align*}
\int_{\sM} d(o,B_\pi(x))^2\,d\mu(x)
&\le 2\int_{\sM\times\sM}d(B_\pi(x),y)^2\,d\pi(x,y)
   +2\int_{\sM} d(o,y)^2\,d\nu(y)\\
&=4\mathcal E_\pi(B_\pi)+2\int_{\sM} d(o,y)^2\,d\nu(y)\\
&=4V(\pi)+2\int_{\sM} d(o,y)^2\,d\nu(y)<\infty.
\end{align*}
Thus $(B_\pi)_\#\mu$ has finite second moment.
\end{proof}

\subsection{Proof of \Cref{thm:best_deterministic}}
\begin{proof}
For every $T\in\mathcal T_\mu$,
\[
F_x(T(x))\ge V_\pi(x)
\]
for $\mu$-almost every $x$. Integrating gives
\[
\mathcal E_\pi(T)=\int_{\sM} F_x(T(x))\,d\mu(x)
\ge \int_{\sM} V_\pi(x)\,d\mu(x)=V(\pi).
\]
Thus $\inf_T\mathcal E_\pi(T)\ge V(\pi)$. Conversely, if $B_\pi$ is a measurable selector from \Cref{prop:borel_selection}, then
\[
\mathcal E_\pi(B_\pi)=\int_{\sM} F_x(B_\pi(x))\,d\mu(x)=\int_{\sM} V_\pi(x)\,d\mu(x)=V(\pi),
\]
so equality holds.

Finally, $T$ attains the infimum if and only if
\[
0=\mathcal E_\pi(T)-V(\pi)
=\int_{\sM}\bigl(F_x(T(x))-V_\pi(x)\bigr)\,d\mu(x).
\]
The integrand is nonnegative. Hence the integral vanishes if and only if $F_x(T(x))=V_\pi(x)$ for $\mu$-almost every $x$, which is equivalent to $T(x)\in\mathcal B_\pi(x)$ almost everywhere.
\end{proof}

\subsection{Proof of \Cref{cor:monge_consistency}}
\begin{proof}
If $\pi=(\id,T)_\#\mu$, then $d\pi(y\mid x)=\delta_{T(x)}(dy)$ for $\mu$-almost every $x$. Therefore
\[
F_x(z)=\frac12 d(z,T(x))^2,
\]
whose unique minimizer is $T(x)$.
\end{proof}

\subsection{Proof of \Cref{thm:monge_characterization}}
\begin{proof}
Since $V_\pi(x)\ge 0$, (i) and (ii) are equivalent. Assume (ii), and let $B_\pi$ be a Borel intrinsic barycentric projection. Then
\[
0=V_\pi(x)=F_x(B_\pi(x))
=\frac12\int_{\sM} d(B_\pi(x),y)^2\,d\pi(y\mid x)
\]
for $\mu$-almost every $x$. The integrand is nonnegative, so $d(B_\pi(x),y)=0$ for $\pi(\cdot\mid x)$-almost every $y$. Hence
\[
d\pi(y\mid x)=\delta_{B_\pi(x)}(dy)
\]
for $\mu$-almost every $x$, proving (iii) with $T=B_\pi$.

If (iii) holds, then for every bounded Borel test function $\varphi:\sM\times\sM\to\R$,
\[
\int\varphi(x,y)\,d\pi(x,y)
=\int_{\sM}\int_{\sM}\varphi(x,y)\,d\pi(y\mid x)\,d\mu(x)
=\int_{\sM}\varphi(x,T(x))\,d\mu(x),
\]
which is exactly $\pi=(\id,T)_\#\mu$. Thus (iii) implies (iv). Finally, (iv) implies (i) by \Cref{cor:monge_consistency}.
\end{proof}

\subsection{Proof of \Cref{prop:target_mismatch}}
\begin{proof}
By \Cref{lem:second_moment_B}, $(B_\pi)_\#\mu\in\mathcal P_2(\sM)$, so the left-hand side is well defined. Define
\[
\widehat\pi:=(B_\pi\circ\mathrm{pr}_1,\mathrm{pr}_2)_\#\pi.
\]
Then $\widehat\pi\in\Pi((B_\pi)_\#\mu,\nu)$. Therefore
\begin{align*}
\frac12 W_2^2((B_\pi)_\#\mu,\nu)
&\le \frac12\int_{\sM\times\sM}d(z,y)^2\,d\widehat\pi(z,y)\\
&=\frac12\int_{\sM\times\sM}d(B_\pi(x),y)^2\,d\pi(x,y)\\
&=\mathcal E_\pi(B_\pi)=V(\pi),
\end{align*}
where the last equality follows from \Cref{thm:best_deterministic}.
\end{proof}

\subsection{Proof of \Cref{thm:hadamard_gap}}
\begin{proof}
Fix $x$ in the full-measure set where $B_\pi(x)$ is defined. In a Hadamard manifold, the barycenter variance inequality gives, for every $z\in\sM$,
\[
\int_{\sM} d(z,y)^2\,d\pi(y\mid x)
\ge
\int_{\sM} d(B_\pi(x),y)^2\,d\pi(y\mid x)
+d(z,B_\pi(x))^2.
\]
Dividing by $2$ and setting $z=T(x)$ yields
\[
F_x(T(x))-V_\pi(x)
\ge
\frac12 d(T(x),B_\pi(x))^2.
\]
Integrating over $x$ proves the inequality. If $\mathcal E_\pi(T)=V(\pi)$, then the right-hand side is zero, so $T=B_\pi$ $\mu$-almost everywhere.
\end{proof}

\subsection{Proof of \Cref{prop:tangent_hadamard}}
\begin{proof}
In a Hadamard manifold, $\log_x(y)$ is globally single-valued for every $x,y\in\sM$. Moreover, $\|\log_x(y)\|_x=d(x,y)$. Since $\pi\in\Pi(\mu,\nu)$ and both marginals have finite second moment, $\int d(x,y)^2\,d\pi(x,y)<\infty$. Therefore
\[
\int_{\sM} d(x,y)^2\,d\pi(y\mid x)<\infty
\]
for $\mu$-almost every $x$. This is exactly the square-integrability condition in $\mathrm{Dom}_2(\widetilde B_\pi)$.
\end{proof}

\subsection{Proof of \Cref{lem:tangent_measurable}}
\begin{proof}
In a Hadamard manifold, the map
\[
L:\sM\times\sM\to T\sM,\qquad L(x,y)=\log_x(y),
\]
is continuous when $T\sM$ is equipped with its usual Borel structure. Let
\[
\sX_1:=\left\{x\in\sM:\int_{\sM} d(x,y)^2\,d\pi(y\mid x)<\infty\right\}.
\]
The function $x\mapsto\int d(x,y)^2\,d\pi(y\mid x)$ is Borel by the kernel property and monotone approximation, hence $\sX_1$ is Borel and has full $\mu$-measure by \Cref{prop:tangent_hadamard}.

It remains to show that $m_\pi$ is Borel on $\sX_1$. Work in a local trivialization of $T\sM$ over a coordinate chart $U\subset\sM$. In this chart, the components of $L(x,y)$ are Borel functions of $(x,y)$. For $R>0$, the truncated components
\[
L_R(x,y):=L(x,y)\1_{\{\|L(x,y)\|_x\le R\}}
\]
are bounded Borel functions in the local trivialization, so
\[
x\mapsto\int_{\sM} L_R(x,y)\,d\pi(y\mid x)
\]
is Borel. As $R\to\infty$, these integrals converge pointwise on $\sX_1$ to $m_\pi(x)$, because $\|L(x,y)\|_x=d(x,y)$ has finite first moment by Cauchy--Schwarz. Thus the local coordinate components of $m_\pi$ are Borel. Since this holds in each chart of a countable atlas, $m_\pi$ is a Borel section of $T\sM$ on $\sX_1$.

Finally, the exponential map $\exp:T\sM\to\sM$ is continuous, so $\widetilde B_\pi(x)=\exp_x(m_\pi(x))$ is Borel on $\sX_1$. Extending arbitrarily outside $\sX_1$ gives a Borel map on $\sM$.
\end{proof}

\subsection{Proof of \Cref{prop:tangent_variational}}
\begin{proof}
Let $a(y):=\log_x(y)$. Since $m_\pi(x)=\int a(y)\,d\pi(y\mid x)$,
\[
v-a(y)=(v-m_\pi(x))+(m_\pi(x)-a(y)).
\]
Expanding the squared norm in the Hilbert space $T_x\sM$ and integrating, the cross term vanishes because
\[
\int_{\sM}(m_\pi(x)-a(y))\,d\pi(y\mid x)=0.
\]
This gives the stated identity. Strict convexity and uniqueness follow immediately.
\end{proof}
\subsection{Proof of \Cref{cor:euclidean_exactness}}
\begin{proof}
In Euclidean space, $\log_x(y)=y-x$ and $\exp_x(v)=x+v$. Hence
\[
\widetilde B_\pi(x)=x+\int(y-x)\,d\pi(y\mid x)=\int y\,d\pi(y\mid x),
\]
which is the unique Euclidean minimizer of $z\mapsto\frac12\int\|z-y\|^2\,d\pi(y\mid x)$.
\end{proof}

\subsection{Proof of \Cref{prop:tangent_monge}}
\begin{proof}
If $\pi=(\id,T)_\#\mu$, then $d\pi(y\mid x)=\delta_{T(x)}(dy)$ for $\mu$-almost every $x$. Hence
\[
m_\pi(x)=\int_{\sM}\log_x(y)\,d\pi(y\mid x)=\log_x(T(x)),
\]
and
\[
\widetilde B_\pi(x)=\exp_x(\log_x(T(x)))=T(x),
\]
where the cut-locus assumption guarantees single-valuedness. If $T(x)=\exp_x(-\nabla\psi(x))$, then, again away from the cut locus, $\log_x(T(x))=-\nabla\psi(x)$.
\end{proof}

\subsection{Proof of \Cref{thm:tangent_gradient}}
\begin{proof}
Write $z_v:=\exp_x(v)$ for $v\in T_x\sM$ sufficiently small. For fixed $y\notin\Cut(x)$, the first-variation formula gives differentiability at $v=0$ of
\[
h_y(v):=\frac12 d(z_v,y)^2
\]
with differential
\[
Dh_y(0)[v]=-\langle v,\log_x(y)\rangle_x.
\]
See \citet{karcher_1977_RiemannianCenterMass} for more details. Thus, for $y\notin\Cut(x)$,
\[
\frac{h_y(v)-h_y(0)+\langle v,\log_x(y)\rangle_x}{\|v\|_x}\to 0
\qquad\text{as }v\to 0.
\]
To pass this first-order expansion through the integral, use the bound
\begin{align*}
\frac{|h_y(v)-h_y(0)|}{\|v\|_x}
&\le \frac{1}{2\|v\|_x}|d(z_v,y)-d(x,y)|\bigl(d(z_v,y)+d(x,y)\bigr)\\
&\le d(x,y)+\frac12\|v\|_x,
\end{align*}
where $|d(z_v,y)-d(x,y)|\le d(z_v,x)=\|v\|_x$ and $d(z_v,y)\le d(x,y)+\|v\|_x$. Also,
\[
\frac{|\langle v,\log_x(y)\rangle_x|}{\|v\|_x}\le d(x,y).
\]
For $\|v\|_x\le 1$, the normalized remainder is therefore bounded by $2d(x,y)+1/2$, which is integrable because the conditional second moment implies the conditional first moment. Since $\pi(\Cut(x)\mid x)=0$, dominated convergence gives
\[
F_x(\exp_x(v))-F_x(x)
+
\left\langle v,\int_{\sM}\log_x(y)\,d\pi(y\mid x)\right\rangle_x
=o(\|v\|_x).
\]
Thus $F_x$ is differentiable at $x$ and
\[
\nabla F_x(x)=-\int_{\sM}\log_x(y)\,d\pi(y\mid x).
\]
The final identity follows from the definition of $\widetilde B_\pi$.
\end{proof}

\subsection{Proof of \Cref{cor:barycentric_equation}}
\begin{proof}
If $z$ is a differentiable minimizer of $F_x$, then $\nabla F_x(z)=0$. Applying the gradient formula at base point $z$ gives the stated barycentric equation. In the Hadamard regime, $F_x$ is strictly geodesically convex, so every critical point is the unique global minimizer.
\end{proof}

\subsection{Proof of \Cref{cor:tangent_first_discrete}}
\begin{proof}
The first gradient is
\[
\nabla\Phi_i(x_i)=-\sum_jw_{ij}\log_{x_i}(y_j).
\]
Therefore
\[
z_i^{(1)}=\exp_{x_i}(-\nabla\Phi_i(x_i))
=\exp_{x_i}\left(\sum_jw_{ij}\log_{x_i}(y_j)\right)
=\widetilde B_{\pi_{\mGamma}}(x_i).
\]
\end{proof}

\subsection{Proof of \Cref{prop:rowwise_decomposition}}
\begin{proof}
The conditional law at $x_i$ is $\sum_jw_{ij}\delta_{y_j}$. Substituting this into the definitions of $F_{x_i}$, $\mathcal E_{\pi_{\mGamma}}$, $B_{\pi_{\mGamma}}$, and $\widetilde B_{\pi_{\mGamma}}$ proves all identities.
\end{proof}

\subsection{Proof of \Cref{prop:discrete_stability}}
\begin{proof}
The convergence of weights implies local uniform convergence $\Phi_i^{(r)}\to\Phi_i$, since on any compact set $\sK\subset\sM$,
\[
\sup_{z\in\sK}|\Phi_i^{(r)}(z)-\Phi_i(z)|
\le \frac12\sum_j|w_{ij}^{(r)}-w_{ij}|\sup_{z\in\sK}d(z,y_j)^2\to0.
\]
The minimizers are uniformly bounded. Indeed, for a reference point $o$, the same coercivity bound as in \Cref{prop:existence_compactness} gives
\[
\Phi_i^{(r)}(z)\ge \frac14d(z,o)^2-C,
\]
with $C$ independent of $r$, because the target support is finite. Since $\Phi_i^{(r)}(B_i^{(r)})\le\Phi_i^{(r)}(o)$, the sequence $B_i^{(r)}$ lies in a common compact set. Every convergent subsequence has a limit that minimizes $\Phi_i$, by local uniform convergence. Since $\Phi_i$ has a unique minimizer, the whole sequence converges to $B_i$.

For the tangential statement, the vectors $\log_{x_i}(y_j)$ are fixed and well-defined. Therefore
\[
\sum_jw_{ij}^{(r)}\log_{x_i}(y_j)\to \sum_jw_{ij}\log_{x_i}(y_j)
\]
in $T_{x_i}\sM$, and continuity of $\exp_{x_i}$ gives the claim.
\end{proof}
\end{document}

%% file: math_commands.tex
\usepackage{amsmath,amsfonts,bm}

\def\eqref#1{equation~\ref{#1}}

\def\1{\bm{1}}

\def\rvx{{\mathbf{x}}}
\def\rvy{{\mathbf{y}}}

\def\vone{{\bm{1}}}

\def\va{{\bm{a}}}
\def\vb{{\bm{b}}}

\def\mLambda{{\bm{\Lambda}}}

\DeclareMathAlphabet{\mathsfit}{\encodingdefault}{\sfdefault}{m}{sl}
\SetMathAlphabet{\mathsfit}{bold}{\encodingdefault}{\sfdefault}{bx}{n}

\def\sD{{\mathbb{D}}}

\def\sK{{\mathbb{K}}}

\def\sM{{\mathbb{M}}}

\def\sO{{\mathbb{O}}}

\def\sX{{\mathbb{X}}}
\def\sY{{\mathbb{Y}}}

\newcommand{\E}{\mathbb{E}}

\newcommand{\R}{\mathbb{R}}

\DeclareMathOperator*{\argmin}{arg\,min}

%% file: figures/tikz_snippet.tex
\begin{tikzpicture}[
  >=Latex,
  manifold/.style={fill=blue!8, draw=blue!35!black, line width=0.75pt},
  gridline/.style={draw=blue!25!black, line width=0.35pt, opacity=0.55},
  geodesic/.style={draw=black!62, line width=0.85pt, dashed},
  meanconnect/.style={draw=black!62, line width=0.85pt, dashed},
  logvec/.style={-Latex, draw=orange!80!black, line width=0.85pt},
  meanvec/.style={-Latex, draw=purple!85!black, line width=1.25pt},
  expmap/.style={-Latex, draw=purple!85!black, line width=1.15pt, densely dashed},
  point/.style={circle, inner sep=1.7pt, fill=black},
  dest/.style={circle, inner sep=1.6pt, fill=red!75!black},
  bary/.style={circle, inner sep=2.0pt, fill=green!45!black},
  tang/.style={circle, inner sep=2.0pt, fill=purple!80!black},
  label/.style={font=\scriptsize, inner sep=1pt},  panel/.style={font=\bfseries\small}
]

\def\ManifoldPatch{%
  (-3.15,-0.86) .. controls (-2.42,0.98) and (-0.72,1.48) .. (0.80,1.20)
  .. controls (2.48,0.90) and (3.38,-0.18) .. (2.78,-1.02)
  .. controls (1.70,-1.72) and (-1.92,-1.62) .. (-3.15,-0.86) -- cycle}

\begin{scope}[shift={(0,0)}]
  \draw[manifold] \ManifoldPatch;

  \begin{scope}
    \clip \ManifoldPatch;
    \draw[gridline] (-2.65,-0.56) .. controls (-1.55,0.20) and (0.50,0.34) .. (2.55,-0.28);
    \draw[gridline] (-2.18,-1.04) .. controls (-0.90,-0.42) and (0.90,-0.38) .. (2.55,-0.80);
    \draw[gridline] (-2.02,0.67) .. controls (-0.55,0.98) and (1.13,0.77) .. (2.55,0.27);
    \draw[gridline] (-1.95,-1.18) .. controls (-1.66,-0.18) and (-1.63,0.55) .. (-1.33,1.10);
    \draw[gridline] (0.05,-1.34) .. controls (0.23,-0.35) and (0.23,0.50) .. (0.40,1.28);
    \draw[gridline] (1.92,-1.14) .. controls (1.58,-0.24) and (1.40,0.44) .. (1.19,1.02);

    \coordinate (BAclip) at (0.78,0.02);
    \coordinate (yAoneclip) at (0.36,0.80);
    \coordinate (yAtwoclip) at (1.65,0.10);
    \coordinate (yAthreeclip) at (0.18,-0.74);
    \draw[geodesic] (BAclip) .. controls (0.60,0.30) and (0.42,0.55) .. (yAoneclip);
    \draw[geodesic] (BAclip) .. controls (1.12,0.15) and (1.38,0.11) .. (yAtwoclip);
    \draw[geodesic] (BAclip) .. controls (0.54,-0.28) and (0.32,-0.54) .. (yAthreeclip);
  \end{scope}

  \coordinate (xA) at (-2.28,-0.24);
  \coordinate (yAone) at (0.36,0.80);
  \coordinate (yAtwo) at (1.65,0.10);
  \coordinate (yAthree) at (0.18,-0.74);
  \coordinate (BA) at (0.78,0.02);

  \node[point] at (xA) {};
  \node[dest] at (yAone) {};
  \node[dest] at (yAtwo) {};
  \node[dest] at (yAthree) {};
  \node[bary] at (BA) {};

  \node[label, anchor=east] at ($(xA)+(-0.08,0.08)$) {$x$};
  \node[label, anchor=west] at ($(yAone)+(0.09,0.05)$) {$y_1$};
  \node[label, anchor=west] at ($(yAtwo)+(0.07,0.02)$) {$y_2$};
  \node[label, anchor=north] at ($(yAthree)+(0.14,-0.10)$) {$y_3$};
  \node[label, anchor=south west] at ($(BA)+(-0.02,-0.4)$) {$B_\pi(x)$};
\end{scope}

\begin{scope}[shift={(8.35,0)}]
  \draw[manifold] \ManifoldPatch;

  \begin{scope}
    \clip \ManifoldPatch;
    \draw[gridline] (-2.65,-0.56) .. controls (-1.55,0.20) and (0.50,0.34) .. (2.55,-0.28);
    \draw[gridline] (-2.18,-1.04) .. controls (-0.90,-0.42) and (0.90,-0.38) .. (2.55,-0.80);
    \draw[gridline] (-2.02,0.67) .. controls (-0.55,0.98) and (1.13,0.77) .. (2.55,0.27);
    \draw[gridline] (-1.95,-1.18) .. controls (-1.66,-0.18) and (-1.63,0.55) .. (-1.33,1.10);
    \draw[gridline] (0.05,-1.34) .. controls (0.23,-0.35) and (0.23,0.50) .. (0.40,1.28);
    \draw[gridline] (1.92,-1.14) .. controls (1.58,-0.24) and (1.40,0.44) .. (1.19,1.02);

    \coordinate (xBclip) at (-2.28,-0.24);
    \coordinate (yBoneclip) at (0.36,0.80);
    \coordinate (yBtwoclip) at (1.65,0.10);
    \coordinate (yBthreeclip) at (0.18,-0.74);
    \draw[geodesic] (xBclip) .. controls (-1.55,0.50) and (-0.50,0.87) .. (yBoneclip);
    \draw[geodesic] (xBclip) .. controls (-0.95,0.05) and (0.35,0.17) .. (yBtwoclip);
    \draw[geodesic] (xBclip) .. controls (-1.10,-0.65) and (-0.32,-0.83) .. (yBthreeclip);
  \end{scope}

  \coordinate (xB) at (-2.28,-0.24);
  \coordinate (yBone) at (0.36,0.80);
  \coordinate (yBtwo) at (1.65,0.10);
  \coordinate (yBthree) at (0.18,-0.74);
  \coordinate (TB) at (0.32,0.08);

  \coordinate (p0) at (-2.28,-2.28);
  \coordinate (p1) at (-3.74,-3.05);
  \coordinate (p2) at (2.40,-3.05);
  \coordinate (p3) at (2.84,-1.74);
  \coordinate (p4) at (-3.28,-1.74);
  \draw[fill=orange!10, draw=orange!55!black, line width=0.60pt, opacity=0.96]
    (p1) -- (p2) -- (p3) -- (p4) -- cycle;
  \draw[draw=black!35, densely dotted] (xB) -- (p0);
  \node[label, anchor=east, black!65] at ($(p0)+(-0.02,0.00)$) {$0_x$};

  \node[point] at (xB) {};
  \node[dest] at (yBone) {};
  \node[dest] at (yBtwo) {};
  \node[dest] at (yBthree) {};
  \node[tang] at (TB) {};

  \node[label, anchor=east] at ($(xB)+(-0.08,0.08)$) {$x$};
  \node[label, anchor=west] at ($(yBone)+(0.09,0.05)$) {$y_1$};
  \node[label, anchor=west] at ($(yBtwo)+(0.07,0.02)$) {$y_2$};
  \node[label, anchor=north] at ($(yBthree)+(0.14,-0.10)$) {$y_3$};
  \node[label, anchor=south west] at ($(TB)+(0.07,0.05)$) {$\widetilde B_\pi(x)$};

  \coordinate (vone) at (-0.78,-2.08);
  \coordinate (vtwo) at (0.60,-2.34);
  \coordinate (vthree) at (-0.98,-2.56);
  \coordinate (vmean) at (-0.09,-2.21);

  \draw[logvec] (p0) -- (vone);
  \draw[logvec] (p0) -- (vtwo);
  \draw[logvec] (p0) -- (vthree);
  \draw[meanvec] (p0) -- (vmean);

  \draw[meanconnect] (vone) -- (vmean);
  \draw[meanconnect] (vtwo) -- (vmean);
  \draw[meanconnect] (vthree) -- (vmean);

  \node[label, anchor=south east] at ($(vone)+(-0.2,0)$) {$\log_x(y_1)$};
  \node[label, anchor=west] at ($(vtwo)+(0.06,0.00)$) {$\log_x(y_2)$};
  \node[label, anchor=north east] at ($(vthree)+(-0.04,-0.03)$) {$\log_x(y_3)$};
  \node[label, anchor=south] at ($(vmean)+(-0.1,0.10)$) {$m_\pi(x)$};

  \draw[expmap] (vmean)
    .. controls (1.08,-1.58) and (1.30,-0.50) ..
    node[label, pos=0.54, anchor=west, xshift=0.14cm, yshift=0.01cm, purple!85!black] {$\exp_x$}
    (TB);
\end{scope}

\end{tikzpicture}

%% file: report/eeg_spd/eeg_geometric_metrics.pgf
\begingroup%
\makeatletter%
\begin{pgfpicture}%
\pgfpathrectangle{\pgfpointorigin}{\pgfqpoint{6.102918in}{2.166643in}}%
\pgfusepath{use as bounding box, clip}%
\begin{pgfscope}%
\pgfsetbuttcap%
\pgfsetmiterjoin%
\definecolor{currentfill}{rgb}{1.000000,1.000000,1.000000}%
\pgfsetfillcolor{currentfill}%
\pgfsetlinewidth{0.000000pt}%
\definecolor{currentstroke}{rgb}{1.000000,1.000000,1.000000}%
\pgfsetstrokecolor{currentstroke}%
\pgfsetdash{}{0pt}%
\pgfpathmoveto{\pgfqpoint{0.000000in}{0.000000in}}%
\pgfpathlineto{\pgfqpoint{6.102918in}{0.000000in}}%
\pgfpathlineto{\pgfqpoint{6.102918in}{2.166643in}}%
\pgfpathlineto{\pgfqpoint{0.000000in}{2.166643in}}%
\pgfpathlineto{\pgfqpoint{0.000000in}{0.000000in}}%
\pgfpathclose%
\pgfusepath{fill}%
\end{pgfscope}%
\begin{pgfscope}%
\pgfsetbuttcap%
\pgfsetmiterjoin%
\definecolor{currentfill}{rgb}{1.000000,1.000000,1.000000}%
\pgfsetfillcolor{currentfill}%
\pgfsetlinewidth{0.000000pt}%
\definecolor{currentstroke}{rgb}{0.000000,0.000000,0.000000}%
\pgfsetstrokecolor{currentstroke}%
\pgfsetstrokeopacity{0.000000}%
\pgfsetdash{}{0pt}%
\pgfpathmoveto{\pgfqpoint{0.352917in}{0.559404in}}%
\pgfpathlineto{\pgfqpoint{2.776646in}{0.559404in}}%
\pgfpathlineto{\pgfqpoint{2.776646in}{1.942004in}}%
\pgfpathlineto{\pgfqpoint{0.352917in}{1.942004in}}%
\pgfpathlineto{\pgfqpoint{0.352917in}{0.559404in}}%
\pgfpathclose%
\pgfusepath{fill}%
\end{pgfscope}%
\begin{pgfscope}%
\pgfpathrectangle{\pgfqpoint{0.352917in}{0.559404in}}{\pgfqpoint{2.423729in}{1.382600in}}%
\pgfusepath{clip}%
\pgfsetbuttcap%
\pgfsetmiterjoin%
\definecolor{currentfill}{rgb}{0.121569,0.466667,0.705882}%
\pgfsetfillcolor{currentfill}%
\pgfsetlinewidth{0.000000pt}%
\definecolor{currentstroke}{rgb}{0.000000,0.000000,0.000000}%
\pgfsetstrokecolor{currentstroke}%
\pgfsetstrokeopacity{0.000000}%
\pgfsetdash{}{0pt}%
\pgfpathmoveto{\pgfqpoint{0.463087in}{0.559404in}}%
\pgfpathlineto{\pgfqpoint{0.830319in}{0.559404in}}%
\pgfpathlineto{\pgfqpoint{0.830319in}{0.559404in}}%
\pgfpathlineto{\pgfqpoint{0.463087in}{0.559404in}}%
\pgfpathlineto{\pgfqpoint{0.463087in}{0.559404in}}%
\pgfpathclose%
\pgfusepath{fill}%
\end{pgfscope}%
\begin{pgfscope}%
\pgfpathrectangle{\pgfqpoint{0.352917in}{0.559404in}}{\pgfqpoint{2.423729in}{1.382600in}}%
\pgfusepath{clip}%
\pgfsetbuttcap%
\pgfsetmiterjoin%
\definecolor{currentfill}{rgb}{0.121569,0.466667,0.705882}%
\pgfsetfillcolor{currentfill}%
\pgfsetlinewidth{0.000000pt}%
\definecolor{currentstroke}{rgb}{0.000000,0.000000,0.000000}%
\pgfsetstrokecolor{currentstroke}%
\pgfsetstrokeopacity{0.000000}%
\pgfsetdash{}{0pt}%
\pgfpathmoveto{\pgfqpoint{0.922127in}{0.559404in}}%
\pgfpathlineto{\pgfqpoint{1.289358in}{0.559404in}}%
\pgfpathlineto{\pgfqpoint{1.289358in}{0.560159in}}%
\pgfpathlineto{\pgfqpoint{0.922127in}{0.560159in}}%
\pgfpathlineto{\pgfqpoint{0.922127in}{0.559404in}}%
\pgfpathclose%
\pgfusepath{fill}%
\end{pgfscope}%
\begin{pgfscope}%
\pgfpathrectangle{\pgfqpoint{0.352917in}{0.559404in}}{\pgfqpoint{2.423729in}{1.382600in}}%
\pgfusepath{clip}%
\pgfsetbuttcap%
\pgfsetmiterjoin%
\definecolor{currentfill}{rgb}{0.121569,0.466667,0.705882}%
\pgfsetfillcolor{currentfill}%
\pgfsetlinewidth{0.000000pt}%
\definecolor{currentstroke}{rgb}{0.000000,0.000000,0.000000}%
\pgfsetstrokecolor{currentstroke}%
\pgfsetstrokeopacity{0.000000}%
\pgfsetdash{}{0pt}%
\pgfpathmoveto{\pgfqpoint{1.381166in}{0.559404in}}%
\pgfpathlineto{\pgfqpoint{1.748398in}{0.559404in}}%
\pgfpathlineto{\pgfqpoint{1.748398in}{0.978944in}}%
\pgfpathlineto{\pgfqpoint{1.381166in}{0.978944in}}%
\pgfpathlineto{\pgfqpoint{1.381166in}{0.559404in}}%
\pgfpathclose%
\pgfusepath{fill}%
\end{pgfscope}%
\begin{pgfscope}%
\pgfpathrectangle{\pgfqpoint{0.352917in}{0.559404in}}{\pgfqpoint{2.423729in}{1.382600in}}%
\pgfusepath{clip}%
\pgfsetbuttcap%
\pgfsetmiterjoin%
\definecolor{currentfill}{rgb}{0.121569,0.466667,0.705882}%
\pgfsetfillcolor{currentfill}%
\pgfsetlinewidth{0.000000pt}%
\definecolor{currentstroke}{rgb}{0.000000,0.000000,0.000000}%
\pgfsetstrokecolor{currentstroke}%
\pgfsetstrokeopacity{0.000000}%
\pgfsetdash{}{0pt}%
\pgfpathmoveto{\pgfqpoint{1.840206in}{0.559404in}}%
\pgfpathlineto{\pgfqpoint{2.207437in}{0.559404in}}%
\pgfpathlineto{\pgfqpoint{2.207437in}{0.564377in}}%
\pgfpathlineto{\pgfqpoint{1.840206in}{0.564377in}}%
\pgfpathlineto{\pgfqpoint{1.840206in}{0.559404in}}%
\pgfpathclose%
\pgfusepath{fill}%
\end{pgfscope}%
\begin{pgfscope}%
\pgfpathrectangle{\pgfqpoint{0.352917in}{0.559404in}}{\pgfqpoint{2.423729in}{1.382600in}}%
\pgfusepath{clip}%
\pgfsetbuttcap%
\pgfsetmiterjoin%
\definecolor{currentfill}{rgb}{0.121569,0.466667,0.705882}%
\pgfsetfillcolor{currentfill}%
\pgfsetlinewidth{0.000000pt}%
\definecolor{currentstroke}{rgb}{0.000000,0.000000,0.000000}%
\pgfsetstrokecolor{currentstroke}%
\pgfsetstrokeopacity{0.000000}%
\pgfsetdash{}{0pt}%
\pgfpathmoveto{\pgfqpoint{2.299245in}{0.559404in}}%
\pgfpathlineto{\pgfqpoint{2.666477in}{0.559404in}}%
\pgfpathlineto{\pgfqpoint{2.666477in}{1.789713in}}%
\pgfpathlineto{\pgfqpoint{2.299245in}{1.789713in}}%
\pgfpathlineto{\pgfqpoint{2.299245in}{0.559404in}}%
\pgfpathclose%
\pgfusepath{fill}%
\end{pgfscope}%
\begin{pgfscope}%
\pgfsetbuttcap%
\pgfsetroundjoin%
\definecolor{currentfill}{rgb}{0.000000,0.000000,0.000000}%
\pgfsetfillcolor{currentfill}%
\pgfsetlinewidth{0.803000pt}%
\definecolor{currentstroke}{rgb}{0.000000,0.000000,0.000000}%
\pgfsetstrokecolor{currentstroke}%
\pgfsetdash{}{0pt}%
\pgfsys@defobject{currentmarker}{\pgfqpoint{0.000000in}{-0.048611in}}{\pgfqpoint{0.000000in}{0.000000in}}{%
\pgfpathmoveto{\pgfqpoint{0.000000in}{0.000000in}}%
\pgfpathlineto{\pgfqpoint{0.000000in}{-0.048611in}}%
\pgfusepath{stroke,fill}%
}%
\begin{pgfscope}%
\pgfsys@transformshift{0.646703in}{0.559404in}%
\pgfsys@useobject{currentmarker}{}%
\end{pgfscope}%
\end{pgfscope}%
\begin{pgfscope}%
\definecolor{textcolor}{rgb}{0.000000,0.000000,0.000000}%
\pgfsetstrokecolor{textcolor}%
\pgfsetfillcolor{textcolor}%
\pgftext[x=0.296669in, y=0.199504in, left, base,rotate=30.000000]{\color{textcolor}{\sffamily\fontsize{8.000000}{9.600000}\selectfont\catcode`\^=\active\def^{\ifmmode\sp\else\^{}\fi}\catcode`\%=\active\def
\end{pgfscope}%
\begin{pgfscope}%
\pgfsetbuttcap%
\pgfsetroundjoin%
\definecolor{currentfill}{rgb}{0.000000,0.000000,0.000000}%
\pgfsetfillcolor{currentfill}%
\pgfsetlinewidth{0.803000pt}%
\definecolor{currentstroke}{rgb}{0.000000,0.000000,0.000000}%
\pgfsetstrokecolor{currentstroke}%
\pgfsetdash{}{0pt}%
\pgfsys@defobject{currentmarker}{\pgfqpoint{0.000000in}{-0.048611in}}{\pgfqpoint{0.000000in}{0.000000in}}{%
\pgfpathmoveto{\pgfqpoint{0.000000in}{0.000000in}}%
\pgfpathlineto{\pgfqpoint{0.000000in}{-0.048611in}}%
\pgfusepath{stroke,fill}%
}%
\begin{pgfscope}%
\pgfsys@transformshift{1.105742in}{0.559404in}%
\pgfsys@useobject{currentmarker}{}%
\end{pgfscope}%
\end{pgfscope}%
\begin{pgfscope}%
\definecolor{textcolor}{rgb}{0.000000,0.000000,0.000000}%
\pgfsetstrokecolor{textcolor}%
\pgfsetfillcolor{textcolor}%
\pgftext[x=0.647554in, y=0.137061in, left, base,rotate=30.000000]{\color{textcolor}{\sffamily\fontsize{8.000000}{9.600000}\selectfont\catcode`\^=\active\def^{\ifmmode\sp\else\^{}\fi}\catcode`\%=\active\def
\end{pgfscope}%
\begin{pgfscope}%
\pgfsetbuttcap%
\pgfsetroundjoin%
\definecolor{currentfill}{rgb}{0.000000,0.000000,0.000000}%
\pgfsetfillcolor{currentfill}%
\pgfsetlinewidth{0.803000pt}%
\definecolor{currentstroke}{rgb}{0.000000,0.000000,0.000000}%
\pgfsetstrokecolor{currentstroke}%
\pgfsetdash{}{0pt}%
\pgfsys@defobject{currentmarker}{\pgfqpoint{0.000000in}{-0.048611in}}{\pgfqpoint{0.000000in}{0.000000in}}{%
\pgfpathmoveto{\pgfqpoint{0.000000in}{0.000000in}}%
\pgfpathlineto{\pgfqpoint{0.000000in}{-0.048611in}}%
\pgfusepath{stroke,fill}%
}%
\begin{pgfscope}%
\pgfsys@transformshift{1.564782in}{0.559404in}%
\pgfsys@useobject{currentmarker}{}%
\end{pgfscope}%
\end{pgfscope}%
\begin{pgfscope}%
\definecolor{textcolor}{rgb}{0.000000,0.000000,0.000000}%
\pgfsetstrokecolor{textcolor}%
\pgfsetfillcolor{textcolor}%
\pgftext[x=1.331289in, y=0.266788in, left, base,rotate=30.000000]{\color{textcolor}{\sffamily\fontsize{8.000000}{9.600000}\selectfont\catcode`\^=\active\def^{\ifmmode\sp\else\^{}\fi}\catcode`\%=\active\def
\end{pgfscope}%
\begin{pgfscope}%
\pgfsetbuttcap%
\pgfsetroundjoin%
\definecolor{currentfill}{rgb}{0.000000,0.000000,0.000000}%
\pgfsetfillcolor{currentfill}%
\pgfsetlinewidth{0.803000pt}%
\definecolor{currentstroke}{rgb}{0.000000,0.000000,0.000000}%
\pgfsetstrokecolor{currentstroke}%
\pgfsetdash{}{0pt}%
\pgfsys@defobject{currentmarker}{\pgfqpoint{0.000000in}{-0.048611in}}{\pgfqpoint{0.000000in}{0.000000in}}{%
\pgfpathmoveto{\pgfqpoint{0.000000in}{0.000000in}}%
\pgfpathlineto{\pgfqpoint{0.000000in}{-0.048611in}}%
\pgfusepath{stroke,fill}%
}%
\begin{pgfscope}%
\pgfsys@transformshift{2.023821in}{0.559404in}%
\pgfsys@useobject{currentmarker}{}%
\end{pgfscope}%
\end{pgfscope}%
\begin{pgfscope}%
\definecolor{textcolor}{rgb}{0.000000,0.000000,0.000000}%
\pgfsetstrokecolor{textcolor}%
\pgfsetfillcolor{textcolor}%
\pgftext[x=1.412606in, y=0.048710in, left, base,rotate=30.000000]{\color{textcolor}{\sffamily\fontsize{8.000000}{9.600000}\selectfont\catcode`\^=\active\def^{\ifmmode\sp\else\^{}\fi}\catcode`\%=\active\def
\end{pgfscope}%
\begin{pgfscope}%
\pgfsetbuttcap%
\pgfsetroundjoin%
\definecolor{currentfill}{rgb}{0.000000,0.000000,0.000000}%
\pgfsetfillcolor{currentfill}%
\pgfsetlinewidth{0.803000pt}%
\definecolor{currentstroke}{rgb}{0.000000,0.000000,0.000000}%
\pgfsetstrokecolor{currentstroke}%
\pgfsetdash{}{0pt}%
\pgfsys@defobject{currentmarker}{\pgfqpoint{0.000000in}{-0.048611in}}{\pgfqpoint{0.000000in}{0.000000in}}{%
\pgfpathmoveto{\pgfqpoint{0.000000in}{0.000000in}}%
\pgfpathlineto{\pgfqpoint{0.000000in}{-0.048611in}}%
\pgfusepath{stroke,fill}%
}%
\begin{pgfscope}%
\pgfsys@transformshift{2.482861in}{0.559404in}%
\pgfsys@useobject{currentmarker}{}%
\end{pgfscope}%
\end{pgfscope}%
\begin{pgfscope}%
\definecolor{textcolor}{rgb}{0.000000,0.000000,0.000000}%
\pgfsetstrokecolor{textcolor}%
\pgfsetfillcolor{textcolor}%
\pgftext[x=2.150135in, y=0.209496in, left, base,rotate=30.000000]{\color{textcolor}{\sffamily\fontsize{8.000000}{9.600000}\selectfont\catcode`\^=\active\def^{\ifmmode\sp\else\^{}\fi}\catcode`\%=\active\def
\end{pgfscope}%
\begin{pgfscope}%
\pgfpathrectangle{\pgfqpoint{0.352917in}{0.559404in}}{\pgfqpoint{2.423729in}{1.382600in}}%
\pgfusepath{clip}%
\pgfsetrectcap%
\pgfsetroundjoin%
\pgfsetlinewidth{0.401500pt}%
\definecolor{currentstroke}{rgb}{0.690196,0.690196,0.690196}%
\pgfsetstrokecolor{currentstroke}%
\pgfsetstrokeopacity{0.350000}%
\pgfsetdash{}{0pt}%
\pgfpathmoveto{\pgfqpoint{0.352917in}{0.559404in}}%
\pgfpathlineto{\pgfqpoint{2.776646in}{0.559404in}}%
\pgfusepath{stroke}%
\end{pgfscope}%
\begin{pgfscope}%
\pgfsetbuttcap%
\pgfsetroundjoin%
\definecolor{currentfill}{rgb}{0.000000,0.000000,0.000000}%
\pgfsetfillcolor{currentfill}%
\pgfsetlinewidth{0.803000pt}%
\definecolor{currentstroke}{rgb}{0.000000,0.000000,0.000000}%
\pgfsetstrokecolor{currentstroke}%
\pgfsetdash{}{0pt}%
\pgfsys@defobject{currentmarker}{\pgfqpoint{-0.048611in}{0.000000in}}{\pgfqpoint{-0.000000in}{0.000000in}}{%
\pgfpathmoveto{\pgfqpoint{-0.000000in}{0.000000in}}%
\pgfpathlineto{\pgfqpoint{-0.048611in}{0.000000in}}%
\pgfusepath{stroke,fill}%
}%
\begin{pgfscope}%
\pgfsys@transformshift{0.352917in}{0.559404in}%
\pgfsys@useobject{currentmarker}{}%
\end{pgfscope}%
\end{pgfscope}%
\begin{pgfscope}%
\definecolor{textcolor}{rgb}{0.000000,0.000000,0.000000}%
\pgfsetstrokecolor{textcolor}%
\pgfsetfillcolor{textcolor}%
\pgftext[x=0.196667in, y=0.520824in, left, base]{\color{textcolor}{\sffamily\fontsize{8.000000}{9.600000}\selectfont\catcode`\^=\active\def^{\ifmmode\sp\else\^{}\fi}\catcode`\%=\active\def
\end{pgfscope}%
\begin{pgfscope}%
\pgfpathrectangle{\pgfqpoint{0.352917in}{0.559404in}}{\pgfqpoint{2.423729in}{1.382600in}}%
\pgfusepath{clip}%
\pgfsetrectcap%
\pgfsetroundjoin%
\pgfsetlinewidth{0.401500pt}%
\definecolor{currentstroke}{rgb}{0.690196,0.690196,0.690196}%
\pgfsetstrokecolor{currentstroke}%
\pgfsetstrokeopacity{0.350000}%
\pgfsetdash{}{0pt}%
\pgfpathmoveto{\pgfqpoint{0.352917in}{0.976265in}}%
\pgfpathlineto{\pgfqpoint{2.776646in}{0.976265in}}%
\pgfusepath{stroke}%
\end{pgfscope}%
\begin{pgfscope}%
\pgfsetbuttcap%
\pgfsetroundjoin%
\definecolor{currentfill}{rgb}{0.000000,0.000000,0.000000}%
\pgfsetfillcolor{currentfill}%
\pgfsetlinewidth{0.803000pt}%
\definecolor{currentstroke}{rgb}{0.000000,0.000000,0.000000}%
\pgfsetstrokecolor{currentstroke}%
\pgfsetdash{}{0pt}%
\pgfsys@defobject{currentmarker}{\pgfqpoint{-0.048611in}{0.000000in}}{\pgfqpoint{-0.000000in}{0.000000in}}{%
\pgfpathmoveto{\pgfqpoint{-0.000000in}{0.000000in}}%
\pgfpathlineto{\pgfqpoint{-0.048611in}{0.000000in}}%
\pgfusepath{stroke,fill}%
}%
\begin{pgfscope}%
\pgfsys@transformshift{0.352917in}{0.976265in}%
\pgfsys@useobject{currentmarker}{}%
\end{pgfscope}%
\end{pgfscope}%
\begin{pgfscope}%
\definecolor{textcolor}{rgb}{0.000000,0.000000,0.000000}%
\pgfsetstrokecolor{textcolor}%
\pgfsetfillcolor{textcolor}%
\pgftext[x=0.196667in, y=0.937685in, left, base]{\color{textcolor}{\sffamily\fontsize{8.000000}{9.600000}\selectfont\catcode`\^=\active\def^{\ifmmode\sp\else\^{}\fi}\catcode`\%=\active\def
\end{pgfscope}%
\begin{pgfscope}%
\pgfpathrectangle{\pgfqpoint{0.352917in}{0.559404in}}{\pgfqpoint{2.423729in}{1.382600in}}%
\pgfusepath{clip}%
\pgfsetrectcap%
\pgfsetroundjoin%
\pgfsetlinewidth{0.401500pt}%
\definecolor{currentstroke}{rgb}{0.690196,0.690196,0.690196}%
\pgfsetstrokecolor{currentstroke}%
\pgfsetstrokeopacity{0.350000}%
\pgfsetdash{}{0pt}%
\pgfpathmoveto{\pgfqpoint{0.352917in}{1.393126in}}%
\pgfpathlineto{\pgfqpoint{2.776646in}{1.393126in}}%
\pgfusepath{stroke}%
\end{pgfscope}%
\begin{pgfscope}%
\pgfsetbuttcap%
\pgfsetroundjoin%
\definecolor{currentfill}{rgb}{0.000000,0.000000,0.000000}%
\pgfsetfillcolor{currentfill}%
\pgfsetlinewidth{0.803000pt}%
\definecolor{currentstroke}{rgb}{0.000000,0.000000,0.000000}%
\pgfsetstrokecolor{currentstroke}%
\pgfsetdash{}{0pt}%
\pgfsys@defobject{currentmarker}{\pgfqpoint{-0.048611in}{0.000000in}}{\pgfqpoint{-0.000000in}{0.000000in}}{%
\pgfpathmoveto{\pgfqpoint{-0.000000in}{0.000000in}}%
\pgfpathlineto{\pgfqpoint{-0.048611in}{0.000000in}}%
\pgfusepath{stroke,fill}%
}%
\begin{pgfscope}%
\pgfsys@transformshift{0.352917in}{1.393126in}%
\pgfsys@useobject{currentmarker}{}%
\end{pgfscope}%
\end{pgfscope}%
\begin{pgfscope}%
\definecolor{textcolor}{rgb}{0.000000,0.000000,0.000000}%
\pgfsetstrokecolor{textcolor}%
\pgfsetfillcolor{textcolor}%
\pgftext[x=0.196667in, y=1.354546in, left, base]{\color{textcolor}{\sffamily\fontsize{8.000000}{9.600000}\selectfont\catcode`\^=\active\def^{\ifmmode\sp\else\^{}\fi}\catcode`\%=\active\def
\end{pgfscope}%
\begin{pgfscope}%
\pgfpathrectangle{\pgfqpoint{0.352917in}{0.559404in}}{\pgfqpoint{2.423729in}{1.382600in}}%
\pgfusepath{clip}%
\pgfsetrectcap%
\pgfsetroundjoin%
\pgfsetlinewidth{0.401500pt}%
\definecolor{currentstroke}{rgb}{0.690196,0.690196,0.690196}%
\pgfsetstrokecolor{currentstroke}%
\pgfsetstrokeopacity{0.350000}%
\pgfsetdash{}{0pt}%
\pgfpathmoveto{\pgfqpoint{0.352917in}{1.809987in}}%
\pgfpathlineto{\pgfqpoint{2.776646in}{1.809987in}}%
\pgfusepath{stroke}%
\end{pgfscope}%
\begin{pgfscope}%
\pgfsetbuttcap%
\pgfsetroundjoin%
\definecolor{currentfill}{rgb}{0.000000,0.000000,0.000000}%
\pgfsetfillcolor{currentfill}%
\pgfsetlinewidth{0.803000pt}%
\definecolor{currentstroke}{rgb}{0.000000,0.000000,0.000000}%
\pgfsetstrokecolor{currentstroke}%
\pgfsetdash{}{0pt}%
\pgfsys@defobject{currentmarker}{\pgfqpoint{-0.048611in}{0.000000in}}{\pgfqpoint{-0.000000in}{0.000000in}}{%
\pgfpathmoveto{\pgfqpoint{-0.000000in}{0.000000in}}%
\pgfpathlineto{\pgfqpoint{-0.048611in}{0.000000in}}%
\pgfusepath{stroke,fill}%
}%
\begin{pgfscope}%
\pgfsys@transformshift{0.352917in}{1.809987in}%
\pgfsys@useobject{currentmarker}{}%
\end{pgfscope}%
\end{pgfscope}%
\begin{pgfscope}%
\definecolor{textcolor}{rgb}{0.000000,0.000000,0.000000}%
\pgfsetstrokecolor{textcolor}%
\pgfsetfillcolor{textcolor}%
\pgftext[x=0.196667in, y=1.771407in, left, base]{\color{textcolor}{\sffamily\fontsize{8.000000}{9.600000}\selectfont\catcode`\^=\active\def^{\ifmmode\sp\else\^{}\fi}\catcode`\%=\active\def
\end{pgfscope}%
\begin{pgfscope}%
\definecolor{textcolor}{rgb}{0.000000,0.000000,0.000000}%
\pgfsetstrokecolor{textcolor}%
\pgfsetfillcolor{textcolor}%
\pgftext[x=0.141111in,y=1.250704in,,bottom,rotate=90.000000]{\color{textcolor}{\sffamily\fontsize{9.000000}{10.800000}\selectfont\catcode`\^=\active\def^{\ifmmode\sp\else\^{}\fi}\catcode`\%=\active\def
\end{pgfscope}%
\begin{pgfscope}%
\pgfpathrectangle{\pgfqpoint{0.352917in}{0.559404in}}{\pgfqpoint{2.423729in}{1.382600in}}%
\pgfusepath{clip}%
\pgfsetbuttcap%
\pgfsetroundjoin%
\pgfsetlinewidth{1.405250pt}%
\definecolor{currentstroke}{rgb}{0.000000,0.000000,0.000000}%
\pgfsetstrokecolor{currentstroke}%
\pgfsetdash{}{0pt}%
\pgfpathmoveto{\pgfqpoint{0.646703in}{0.559404in}}%
\pgfpathlineto{\pgfqpoint{0.646703in}{0.559404in}}%
\pgfusepath{stroke}%
\end{pgfscope}%
\begin{pgfscope}%
\pgfpathrectangle{\pgfqpoint{0.352917in}{0.559404in}}{\pgfqpoint{2.423729in}{1.382600in}}%
\pgfusepath{clip}%
\pgfsetbuttcap%
\pgfsetroundjoin%
\pgfsetlinewidth{1.405250pt}%
\definecolor{currentstroke}{rgb}{0.000000,0.000000,0.000000}%
\pgfsetstrokecolor{currentstroke}%
\pgfsetdash{}{0pt}%
\pgfpathmoveto{\pgfqpoint{1.105742in}{0.560091in}}%
\pgfpathlineto{\pgfqpoint{1.105742in}{0.560227in}}%
\pgfusepath{stroke}%
\end{pgfscope}%
\begin{pgfscope}%
\pgfpathrectangle{\pgfqpoint{0.352917in}{0.559404in}}{\pgfqpoint{2.423729in}{1.382600in}}%
\pgfusepath{clip}%
\pgfsetbuttcap%
\pgfsetroundjoin%
\pgfsetlinewidth{1.405250pt}%
\definecolor{currentstroke}{rgb}{0.000000,0.000000,0.000000}%
\pgfsetstrokecolor{currentstroke}%
\pgfsetdash{}{0pt}%
\pgfpathmoveto{\pgfqpoint{1.564782in}{0.977037in}}%
\pgfpathlineto{\pgfqpoint{1.564782in}{0.980851in}}%
\pgfusepath{stroke}%
\end{pgfscope}%
\begin{pgfscope}%
\pgfpathrectangle{\pgfqpoint{0.352917in}{0.559404in}}{\pgfqpoint{2.423729in}{1.382600in}}%
\pgfusepath{clip}%
\pgfsetbuttcap%
\pgfsetroundjoin%
\pgfsetlinewidth{1.405250pt}%
\definecolor{currentstroke}{rgb}{0.000000,0.000000,0.000000}%
\pgfsetstrokecolor{currentstroke}%
\pgfsetdash{}{0pt}%
\pgfpathmoveto{\pgfqpoint{2.023821in}{0.564062in}}%
\pgfpathlineto{\pgfqpoint{2.023821in}{0.564692in}}%
\pgfusepath{stroke}%
\end{pgfscope}%
\begin{pgfscope}%
\pgfpathrectangle{\pgfqpoint{0.352917in}{0.559404in}}{\pgfqpoint{2.423729in}{1.382600in}}%
\pgfusepath{clip}%
\pgfsetbuttcap%
\pgfsetroundjoin%
\pgfsetlinewidth{1.405250pt}%
\definecolor{currentstroke}{rgb}{0.000000,0.000000,0.000000}%
\pgfsetstrokecolor{currentstroke}%
\pgfsetdash{}{0pt}%
\pgfpathmoveto{\pgfqpoint{2.482861in}{1.703260in}}%
\pgfpathlineto{\pgfqpoint{2.482861in}{1.876166in}}%
\pgfusepath{stroke}%
\end{pgfscope}%
\begin{pgfscope}%
\pgfpathrectangle{\pgfqpoint{0.352917in}{0.559404in}}{\pgfqpoint{2.423729in}{1.382600in}}%
\pgfusepath{clip}%
\pgfsetbuttcap%
\pgfsetroundjoin%
\definecolor{currentfill}{rgb}{0.121569,0.466667,0.705882}%
\pgfsetfillcolor{currentfill}%
\pgfsetlinewidth{1.003750pt}%
\definecolor{currentstroke}{rgb}{0.000000,0.000000,0.000000}%
\pgfsetstrokecolor{currentstroke}%
\pgfsetdash{}{0pt}%
\pgfsys@defobject{currentmarker}{\pgfqpoint{-0.027778in}{-0.000000in}}{\pgfqpoint{0.027778in}{0.000000in}}{%
\pgfpathmoveto{\pgfqpoint{0.027778in}{-0.000000in}}%
\pgfpathlineto{\pgfqpoint{-0.027778in}{0.000000in}}%
\pgfusepath{stroke,fill}%
}%
\begin{pgfscope}%
\pgfsys@transformshift{0.646703in}{0.559404in}%
\pgfsys@useobject{currentmarker}{}%
\end{pgfscope}%
\begin{pgfscope}%
\pgfsys@transformshift{1.105742in}{0.560091in}%
\pgfsys@useobject{currentmarker}{}%
\end{pgfscope}%
\begin{pgfscope}%
\pgfsys@transformshift{1.564782in}{0.977037in}%
\pgfsys@useobject{currentmarker}{}%
\end{pgfscope}%
\begin{pgfscope}%
\pgfsys@transformshift{2.023821in}{0.564062in}%
\pgfsys@useobject{currentmarker}{}%
\end{pgfscope}%
\begin{pgfscope}%
\pgfsys@transformshift{2.482861in}{1.703260in}%
\pgfsys@useobject{currentmarker}{}%
\end{pgfscope}%
\end{pgfscope}%
\begin{pgfscope}%
\pgfpathrectangle{\pgfqpoint{0.352917in}{0.559404in}}{\pgfqpoint{2.423729in}{1.382600in}}%
\pgfusepath{clip}%
\pgfsetbuttcap%
\pgfsetroundjoin%
\definecolor{currentfill}{rgb}{0.121569,0.466667,0.705882}%
\pgfsetfillcolor{currentfill}%
\pgfsetlinewidth{1.003750pt}%
\definecolor{currentstroke}{rgb}{0.000000,0.000000,0.000000}%
\pgfsetstrokecolor{currentstroke}%
\pgfsetdash{}{0pt}%
\pgfsys@defobject{currentmarker}{\pgfqpoint{-0.027778in}{-0.000000in}}{\pgfqpoint{0.027778in}{0.000000in}}{%
\pgfpathmoveto{\pgfqpoint{0.027778in}{-0.000000in}}%
\pgfpathlineto{\pgfqpoint{-0.027778in}{0.000000in}}%
\pgfusepath{stroke,fill}%
}%
\begin{pgfscope}%
\pgfsys@transformshift{0.646703in}{0.559404in}%
\pgfsys@useobject{currentmarker}{}%
\end{pgfscope}%
\begin{pgfscope}%
\pgfsys@transformshift{1.105742in}{0.560227in}%
\pgfsys@useobject{currentmarker}{}%
\end{pgfscope}%
\begin{pgfscope}%
\pgfsys@transformshift{1.564782in}{0.980851in}%
\pgfsys@useobject{currentmarker}{}%
\end{pgfscope}%
\begin{pgfscope}%
\pgfsys@transformshift{2.023821in}{0.564692in}%
\pgfsys@useobject{currentmarker}{}%
\end{pgfscope}%
\begin{pgfscope}%
\pgfsys@transformshift{2.482861in}{1.876166in}%
\pgfsys@useobject{currentmarker}{}%
\end{pgfscope}%
\end{pgfscope}%
\begin{pgfscope}%
\pgfsetrectcap%
\pgfsetmiterjoin%
\pgfsetlinewidth{0.803000pt}%
\definecolor{currentstroke}{rgb}{0.000000,0.000000,0.000000}%
\pgfsetstrokecolor{currentstroke}%
\pgfsetdash{}{0pt}%
\pgfpathmoveto{\pgfqpoint{0.352917in}{0.559404in}}%
\pgfpathlineto{\pgfqpoint{0.352917in}{1.942004in}}%
\pgfusepath{stroke}%
\end{pgfscope}%
\begin{pgfscope}%
\pgfsetrectcap%
\pgfsetmiterjoin%
\pgfsetlinewidth{0.803000pt}%
\definecolor{currentstroke}{rgb}{0.000000,0.000000,0.000000}%
\pgfsetstrokecolor{currentstroke}%
\pgfsetdash{}{0pt}%
\pgfpathmoveto{\pgfqpoint{0.352918in}{0.559404in}}%
\pgfpathlineto{\pgfqpoint{2.776646in}{0.559404in}}%
\pgfusepath{stroke}%
\end{pgfscope}%
\begin{pgfscope}%
\pgfsetbuttcap%
\pgfsetmiterjoin%
\definecolor{currentfill}{rgb}{1.000000,1.000000,1.000000}%
\pgfsetfillcolor{currentfill}%
\pgfsetlinewidth{0.000000pt}%
\definecolor{currentstroke}{rgb}{0.000000,0.000000,0.000000}%
\pgfsetstrokecolor{currentstroke}%
\pgfsetstrokeopacity{0.000000}%
\pgfsetdash{}{0pt}%
\pgfpathmoveto{\pgfqpoint{3.649189in}{0.559404in}}%
\pgfpathlineto{\pgfqpoint{6.072918in}{0.559404in}}%
\pgfpathlineto{\pgfqpoint{6.072918in}{1.942004in}}%
\pgfpathlineto{\pgfqpoint{3.649189in}{1.942004in}}%
\pgfpathlineto{\pgfqpoint{3.649189in}{0.559404in}}%
\pgfpathclose%
\pgfusepath{fill}%
\end{pgfscope}%
\begin{pgfscope}%
\pgfpathrectangle{\pgfqpoint{3.649189in}{0.559404in}}{\pgfqpoint{2.423729in}{1.382600in}}%
\pgfusepath{clip}%
\pgfsetbuttcap%
\pgfsetmiterjoin%
\definecolor{currentfill}{rgb}{0.121569,0.466667,0.705882}%
\pgfsetfillcolor{currentfill}%
\pgfsetlinewidth{0.000000pt}%
\definecolor{currentstroke}{rgb}{0.000000,0.000000,0.000000}%
\pgfsetstrokecolor{currentstroke}%
\pgfsetstrokeopacity{0.000000}%
\pgfsetdash{}{0pt}%
\pgfpathmoveto{\pgfqpoint{3.759358in}{0.559404in}}%
\pgfpathlineto{\pgfqpoint{4.126590in}{0.559404in}}%
\pgfpathlineto{\pgfqpoint{4.126590in}{0.893005in}}%
\pgfpathlineto{\pgfqpoint{3.759358in}{0.893005in}}%
\pgfpathlineto{\pgfqpoint{3.759358in}{0.559404in}}%
\pgfpathclose%
\pgfusepath{fill}%
\end{pgfscope}%
\begin{pgfscope}%
\pgfpathrectangle{\pgfqpoint{3.649189in}{0.559404in}}{\pgfqpoint{2.423729in}{1.382600in}}%
\pgfusepath{clip}%
\pgfsetbuttcap%
\pgfsetmiterjoin%
\definecolor{currentfill}{rgb}{0.121569,0.466667,0.705882}%
\pgfsetfillcolor{currentfill}%
\pgfsetlinewidth{0.000000pt}%
\definecolor{currentstroke}{rgb}{0.000000,0.000000,0.000000}%
\pgfsetstrokecolor{currentstroke}%
\pgfsetstrokeopacity{0.000000}%
\pgfsetdash{}{0pt}%
\pgfpathmoveto{\pgfqpoint{4.218398in}{0.559404in}}%
\pgfpathlineto{\pgfqpoint{4.585629in}{0.559404in}}%
\pgfpathlineto{\pgfqpoint{4.585629in}{0.893669in}}%
\pgfpathlineto{\pgfqpoint{4.218398in}{0.893669in}}%
\pgfpathlineto{\pgfqpoint{4.218398in}{0.559404in}}%
\pgfpathclose%
\pgfusepath{fill}%
\end{pgfscope}%
\begin{pgfscope}%
\pgfpathrectangle{\pgfqpoint{3.649189in}{0.559404in}}{\pgfqpoint{2.423729in}{1.382600in}}%
\pgfusepath{clip}%
\pgfsetbuttcap%
\pgfsetmiterjoin%
\definecolor{currentfill}{rgb}{0.121569,0.466667,0.705882}%
\pgfsetfillcolor{currentfill}%
\pgfsetlinewidth{0.000000pt}%
\definecolor{currentstroke}{rgb}{0.000000,0.000000,0.000000}%
\pgfsetstrokecolor{currentstroke}%
\pgfsetstrokeopacity{0.000000}%
\pgfsetdash{}{0pt}%
\pgfpathmoveto{\pgfqpoint{4.677437in}{0.559404in}}%
\pgfpathlineto{\pgfqpoint{5.044669in}{0.559404in}}%
\pgfpathlineto{\pgfqpoint{5.044669in}{1.121457in}}%
\pgfpathlineto{\pgfqpoint{4.677437in}{1.121457in}}%
\pgfpathlineto{\pgfqpoint{4.677437in}{0.559404in}}%
\pgfpathclose%
\pgfusepath{fill}%
\end{pgfscope}%
\begin{pgfscope}%
\pgfpathrectangle{\pgfqpoint{3.649189in}{0.559404in}}{\pgfqpoint{2.423729in}{1.382600in}}%
\pgfusepath{clip}%
\pgfsetbuttcap%
\pgfsetmiterjoin%
\definecolor{currentfill}{rgb}{0.121569,0.466667,0.705882}%
\pgfsetfillcolor{currentfill}%
\pgfsetlinewidth{0.000000pt}%
\definecolor{currentstroke}{rgb}{0.000000,0.000000,0.000000}%
\pgfsetstrokecolor{currentstroke}%
\pgfsetstrokeopacity{0.000000}%
\pgfsetdash{}{0pt}%
\pgfpathmoveto{\pgfqpoint{5.136477in}{0.559404in}}%
\pgfpathlineto{\pgfqpoint{5.503708in}{0.559404in}}%
\pgfpathlineto{\pgfqpoint{5.503708in}{0.897255in}}%
\pgfpathlineto{\pgfqpoint{5.136477in}{0.897255in}}%
\pgfpathlineto{\pgfqpoint{5.136477in}{0.559404in}}%
\pgfpathclose%
\pgfusepath{fill}%
\end{pgfscope}%
\begin{pgfscope}%
\pgfpathrectangle{\pgfqpoint{3.649189in}{0.559404in}}{\pgfqpoint{2.423729in}{1.382600in}}%
\pgfusepath{clip}%
\pgfsetbuttcap%
\pgfsetmiterjoin%
\definecolor{currentfill}{rgb}{0.121569,0.466667,0.705882}%
\pgfsetfillcolor{currentfill}%
\pgfsetlinewidth{0.000000pt}%
\definecolor{currentstroke}{rgb}{0.000000,0.000000,0.000000}%
\pgfsetstrokecolor{currentstroke}%
\pgfsetstrokeopacity{0.000000}%
\pgfsetdash{}{0pt}%
\pgfpathmoveto{\pgfqpoint{5.595516in}{0.559404in}}%
\pgfpathlineto{\pgfqpoint{5.962748in}{0.559404in}}%
\pgfpathlineto{\pgfqpoint{5.962748in}{1.829295in}}%
\pgfpathlineto{\pgfqpoint{5.595516in}{1.829295in}}%
\pgfpathlineto{\pgfqpoint{5.595516in}{0.559404in}}%
\pgfpathclose%
\pgfusepath{fill}%
\end{pgfscope}%
\begin{pgfscope}%
\pgfsetbuttcap%
\pgfsetroundjoin%
\definecolor{currentfill}{rgb}{0.000000,0.000000,0.000000}%
\pgfsetfillcolor{currentfill}%
\pgfsetlinewidth{0.803000pt}%
\definecolor{currentstroke}{rgb}{0.000000,0.000000,0.000000}%
\pgfsetstrokecolor{currentstroke}%
\pgfsetdash{}{0pt}%
\pgfsys@defobject{currentmarker}{\pgfqpoint{0.000000in}{-0.048611in}}{\pgfqpoint{0.000000in}{0.000000in}}{%
\pgfpathmoveto{\pgfqpoint{0.000000in}{0.000000in}}%
\pgfpathlineto{\pgfqpoint{0.000000in}{-0.048611in}}%
\pgfusepath{stroke,fill}%
}%
\begin{pgfscope}%
\pgfsys@transformshift{3.942974in}{0.559404in}%
\pgfsys@useobject{currentmarker}{}%
\end{pgfscope}%
\end{pgfscope}%
\begin{pgfscope}%
\definecolor{textcolor}{rgb}{0.000000,0.000000,0.000000}%
\pgfsetstrokecolor{textcolor}%
\pgfsetfillcolor{textcolor}%
\pgftext[x=3.592940in, y=0.199504in, left, base,rotate=30.000000]{\color{textcolor}{\sffamily\fontsize{8.000000}{9.600000}\selectfont\catcode`\^=\active\def^{\ifmmode\sp\else\^{}\fi}\catcode`\%=\active\def
\end{pgfscope}%
\begin{pgfscope}%
\pgfsetbuttcap%
\pgfsetroundjoin%
\definecolor{currentfill}{rgb}{0.000000,0.000000,0.000000}%
\pgfsetfillcolor{currentfill}%
\pgfsetlinewidth{0.803000pt}%
\definecolor{currentstroke}{rgb}{0.000000,0.000000,0.000000}%
\pgfsetstrokecolor{currentstroke}%
\pgfsetdash{}{0pt}%
\pgfsys@defobject{currentmarker}{\pgfqpoint{0.000000in}{-0.048611in}}{\pgfqpoint{0.000000in}{0.000000in}}{%
\pgfpathmoveto{\pgfqpoint{0.000000in}{0.000000in}}%
\pgfpathlineto{\pgfqpoint{0.000000in}{-0.048611in}}%
\pgfusepath{stroke,fill}%
}%
\begin{pgfscope}%
\pgfsys@transformshift{4.402014in}{0.559404in}%
\pgfsys@useobject{currentmarker}{}%
\end{pgfscope}%
\end{pgfscope}%
\begin{pgfscope}%
\definecolor{textcolor}{rgb}{0.000000,0.000000,0.000000}%
\pgfsetstrokecolor{textcolor}%
\pgfsetfillcolor{textcolor}%
\pgftext[x=3.943825in, y=0.137061in, left, base,rotate=30.000000]{\color{textcolor}{\sffamily\fontsize{8.000000}{9.600000}\selectfont\catcode`\^=\active\def^{\ifmmode\sp\else\^{}\fi}\catcode`\%=\active\def
\end{pgfscope}%
\begin{pgfscope}%
\pgfsetbuttcap%
\pgfsetroundjoin%
\definecolor{currentfill}{rgb}{0.000000,0.000000,0.000000}%
\pgfsetfillcolor{currentfill}%
\pgfsetlinewidth{0.803000pt}%
\definecolor{currentstroke}{rgb}{0.000000,0.000000,0.000000}%
\pgfsetstrokecolor{currentstroke}%
\pgfsetdash{}{0pt}%
\pgfsys@defobject{currentmarker}{\pgfqpoint{0.000000in}{-0.048611in}}{\pgfqpoint{0.000000in}{0.000000in}}{%
\pgfpathmoveto{\pgfqpoint{0.000000in}{0.000000in}}%
\pgfpathlineto{\pgfqpoint{0.000000in}{-0.048611in}}%
\pgfusepath{stroke,fill}%
}%
\begin{pgfscope}%
\pgfsys@transformshift{4.861053in}{0.559404in}%
\pgfsys@useobject{currentmarker}{}%
\end{pgfscope}%
\end{pgfscope}%
\begin{pgfscope}%
\definecolor{textcolor}{rgb}{0.000000,0.000000,0.000000}%
\pgfsetstrokecolor{textcolor}%
\pgfsetfillcolor{textcolor}%
\pgftext[x=4.627560in, y=0.266788in, left, base,rotate=30.000000]{\color{textcolor}{\sffamily\fontsize{8.000000}{9.600000}\selectfont\catcode`\^=\active\def^{\ifmmode\sp\else\^{}\fi}\catcode`\%=\active\def
\end{pgfscope}%
\begin{pgfscope}%
\pgfsetbuttcap%
\pgfsetroundjoin%
\definecolor{currentfill}{rgb}{0.000000,0.000000,0.000000}%
\pgfsetfillcolor{currentfill}%
\pgfsetlinewidth{0.803000pt}%
\definecolor{currentstroke}{rgb}{0.000000,0.000000,0.000000}%
\pgfsetstrokecolor{currentstroke}%
\pgfsetdash{}{0pt}%
\pgfsys@defobject{currentmarker}{\pgfqpoint{0.000000in}{-0.048611in}}{\pgfqpoint{0.000000in}{0.000000in}}{%
\pgfpathmoveto{\pgfqpoint{0.000000in}{0.000000in}}%
\pgfpathlineto{\pgfqpoint{0.000000in}{-0.048611in}}%
\pgfusepath{stroke,fill}%
}%
\begin{pgfscope}%
\pgfsys@transformshift{5.320093in}{0.559404in}%
\pgfsys@useobject{currentmarker}{}%
\end{pgfscope}%
\end{pgfscope}%
\begin{pgfscope}%
\definecolor{textcolor}{rgb}{0.000000,0.000000,0.000000}%
\pgfsetstrokecolor{textcolor}%
\pgfsetfillcolor{textcolor}%
\pgftext[x=4.708877in, y=0.048710in, left, base,rotate=30.000000]{\color{textcolor}{\sffamily\fontsize{8.000000}{9.600000}\selectfont\catcode`\^=\active\def^{\ifmmode\sp\else\^{}\fi}\catcode`\%=\active\def
\end{pgfscope}%
\begin{pgfscope}%
\pgfsetbuttcap%
\pgfsetroundjoin%
\definecolor{currentfill}{rgb}{0.000000,0.000000,0.000000}%
\pgfsetfillcolor{currentfill}%
\pgfsetlinewidth{0.803000pt}%
\definecolor{currentstroke}{rgb}{0.000000,0.000000,0.000000}%
\pgfsetstrokecolor{currentstroke}%
\pgfsetdash{}{0pt}%
\pgfsys@defobject{currentmarker}{\pgfqpoint{0.000000in}{-0.048611in}}{\pgfqpoint{0.000000in}{0.000000in}}{%
\pgfpathmoveto{\pgfqpoint{0.000000in}{0.000000in}}%
\pgfpathlineto{\pgfqpoint{0.000000in}{-0.048611in}}%
\pgfusepath{stroke,fill}%
}%
\begin{pgfscope}%
\pgfsys@transformshift{5.779132in}{0.559404in}%
\pgfsys@useobject{currentmarker}{}%
\end{pgfscope}%
\end{pgfscope}%
\begin{pgfscope}%
\definecolor{textcolor}{rgb}{0.000000,0.000000,0.000000}%
\pgfsetstrokecolor{textcolor}%
\pgfsetfillcolor{textcolor}%
\pgftext[x=5.446406in, y=0.209496in, left, base,rotate=30.000000]{\color{textcolor}{\sffamily\fontsize{8.000000}{9.600000}\selectfont\catcode`\^=\active\def^{\ifmmode\sp\else\^{}\fi}\catcode`\%=\active\def
\end{pgfscope}%
\begin{pgfscope}%
\pgfpathrectangle{\pgfqpoint{3.649189in}{0.559404in}}{\pgfqpoint{2.423729in}{1.382600in}}%
\pgfusepath{clip}%
\pgfsetrectcap%
\pgfsetroundjoin%
\pgfsetlinewidth{0.401500pt}%
\definecolor{currentstroke}{rgb}{0.690196,0.690196,0.690196}%
\pgfsetstrokecolor{currentstroke}%
\pgfsetstrokeopacity{0.350000}%
\pgfsetdash{}{0pt}%
\pgfpathmoveto{\pgfqpoint{3.649189in}{0.559404in}}%
\pgfpathlineto{\pgfqpoint{6.072918in}{0.559404in}}%
\pgfusepath{stroke}%
\end{pgfscope}%
\begin{pgfscope}%
\pgfsetbuttcap%
\pgfsetroundjoin%
\definecolor{currentfill}{rgb}{0.000000,0.000000,0.000000}%
\pgfsetfillcolor{currentfill}%
\pgfsetlinewidth{0.803000pt}%
\definecolor{currentstroke}{rgb}{0.000000,0.000000,0.000000}%
\pgfsetstrokecolor{currentstroke}%
\pgfsetdash{}{0pt}%
\pgfsys@defobject{currentmarker}{\pgfqpoint{-0.048611in}{0.000000in}}{\pgfqpoint{-0.000000in}{0.000000in}}{%
\pgfpathmoveto{\pgfqpoint{-0.000000in}{0.000000in}}%
\pgfpathlineto{\pgfqpoint{-0.048611in}{0.000000in}}%
\pgfusepath{stroke,fill}%
}%
\begin{pgfscope}%
\pgfsys@transformshift{3.649189in}{0.559404in}%
\pgfsys@useobject{currentmarker}{}%
\end{pgfscope}%
\end{pgfscope}%
\begin{pgfscope}%
\definecolor{textcolor}{rgb}{0.000000,0.000000,0.000000}%
\pgfsetstrokecolor{textcolor}%
\pgfsetfillcolor{textcolor}%
\pgftext[x=3.492938in, y=0.520824in, left, base]{\color{textcolor}{\sffamily\fontsize{8.000000}{9.600000}\selectfont\catcode`\^=\active\def^{\ifmmode\sp\else\^{}\fi}\catcode`\%=\active\def
\end{pgfscope}%
\begin{pgfscope}%
\pgfpathrectangle{\pgfqpoint{3.649189in}{0.559404in}}{\pgfqpoint{2.423729in}{1.382600in}}%
\pgfusepath{clip}%
\pgfsetrectcap%
\pgfsetroundjoin%
\pgfsetlinewidth{0.401500pt}%
\definecolor{currentstroke}{rgb}{0.690196,0.690196,0.690196}%
\pgfsetstrokecolor{currentstroke}%
\pgfsetstrokeopacity{0.350000}%
\pgfsetdash{}{0pt}%
\pgfpathmoveto{\pgfqpoint{3.649189in}{0.818321in}}%
\pgfpathlineto{\pgfqpoint{6.072918in}{0.818321in}}%
\pgfusepath{stroke}%
\end{pgfscope}%
\begin{pgfscope}%
\pgfsetbuttcap%
\pgfsetroundjoin%
\definecolor{currentfill}{rgb}{0.000000,0.000000,0.000000}%
\pgfsetfillcolor{currentfill}%
\pgfsetlinewidth{0.803000pt}%
\definecolor{currentstroke}{rgb}{0.000000,0.000000,0.000000}%
\pgfsetstrokecolor{currentstroke}%
\pgfsetdash{}{0pt}%
\pgfsys@defobject{currentmarker}{\pgfqpoint{-0.048611in}{0.000000in}}{\pgfqpoint{-0.000000in}{0.000000in}}{%
\pgfpathmoveto{\pgfqpoint{-0.000000in}{0.000000in}}%
\pgfpathlineto{\pgfqpoint{-0.048611in}{0.000000in}}%
\pgfusepath{stroke,fill}%
}%
\begin{pgfscope}%
\pgfsys@transformshift{3.649189in}{0.818321in}%
\pgfsys@useobject{currentmarker}{}%
\end{pgfscope}%
\end{pgfscope}%
\begin{pgfscope}%
\definecolor{textcolor}{rgb}{0.000000,0.000000,0.000000}%
\pgfsetstrokecolor{textcolor}%
\pgfsetfillcolor{textcolor}%
\pgftext[x=3.492938in, y=0.779741in, left, base]{\color{textcolor}{\sffamily\fontsize{8.000000}{9.600000}\selectfont\catcode`\^=\active\def^{\ifmmode\sp\else\^{}\fi}\catcode`\%=\active\def
\end{pgfscope}%
\begin{pgfscope}%
\pgfpathrectangle{\pgfqpoint{3.649189in}{0.559404in}}{\pgfqpoint{2.423729in}{1.382600in}}%
\pgfusepath{clip}%
\pgfsetrectcap%
\pgfsetroundjoin%
\pgfsetlinewidth{0.401500pt}%
\definecolor{currentstroke}{rgb}{0.690196,0.690196,0.690196}%
\pgfsetstrokecolor{currentstroke}%
\pgfsetstrokeopacity{0.350000}%
\pgfsetdash{}{0pt}%
\pgfpathmoveto{\pgfqpoint{3.649189in}{1.077238in}}%
\pgfpathlineto{\pgfqpoint{6.072918in}{1.077238in}}%
\pgfusepath{stroke}%
\end{pgfscope}%
\begin{pgfscope}%
\pgfsetbuttcap%
\pgfsetroundjoin%
\definecolor{currentfill}{rgb}{0.000000,0.000000,0.000000}%
\pgfsetfillcolor{currentfill}%
\pgfsetlinewidth{0.803000pt}%
\definecolor{currentstroke}{rgb}{0.000000,0.000000,0.000000}%
\pgfsetstrokecolor{currentstroke}%
\pgfsetdash{}{0pt}%
\pgfsys@defobject{currentmarker}{\pgfqpoint{-0.048611in}{0.000000in}}{\pgfqpoint{-0.000000in}{0.000000in}}{%
\pgfpathmoveto{\pgfqpoint{-0.000000in}{0.000000in}}%
\pgfpathlineto{\pgfqpoint{-0.048611in}{0.000000in}}%
\pgfusepath{stroke,fill}%
}%
\begin{pgfscope}%
\pgfsys@transformshift{3.649189in}{1.077238in}%
\pgfsys@useobject{currentmarker}{}%
\end{pgfscope}%
\end{pgfscope}%
\begin{pgfscope}%
\definecolor{textcolor}{rgb}{0.000000,0.000000,0.000000}%
\pgfsetstrokecolor{textcolor}%
\pgfsetfillcolor{textcolor}%
\pgftext[x=3.492938in, y=1.038658in, left, base]{\color{textcolor}{\sffamily\fontsize{8.000000}{9.600000}\selectfont\catcode`\^=\active\def^{\ifmmode\sp\else\^{}\fi}\catcode`\%=\active\def
\end{pgfscope}%
\begin{pgfscope}%
\pgfpathrectangle{\pgfqpoint{3.649189in}{0.559404in}}{\pgfqpoint{2.423729in}{1.382600in}}%
\pgfusepath{clip}%
\pgfsetrectcap%
\pgfsetroundjoin%
\pgfsetlinewidth{0.401500pt}%
\definecolor{currentstroke}{rgb}{0.690196,0.690196,0.690196}%
\pgfsetstrokecolor{currentstroke}%
\pgfsetstrokeopacity{0.350000}%
\pgfsetdash{}{0pt}%
\pgfpathmoveto{\pgfqpoint{3.649189in}{1.336155in}}%
\pgfpathlineto{\pgfqpoint{6.072918in}{1.336155in}}%
\pgfusepath{stroke}%
\end{pgfscope}%
\begin{pgfscope}%
\pgfsetbuttcap%
\pgfsetroundjoin%
\definecolor{currentfill}{rgb}{0.000000,0.000000,0.000000}%
\pgfsetfillcolor{currentfill}%
\pgfsetlinewidth{0.803000pt}%
\definecolor{currentstroke}{rgb}{0.000000,0.000000,0.000000}%
\pgfsetstrokecolor{currentstroke}%
\pgfsetdash{}{0pt}%
\pgfsys@defobject{currentmarker}{\pgfqpoint{-0.048611in}{0.000000in}}{\pgfqpoint{-0.000000in}{0.000000in}}{%
\pgfpathmoveto{\pgfqpoint{-0.000000in}{0.000000in}}%
\pgfpathlineto{\pgfqpoint{-0.048611in}{0.000000in}}%
\pgfusepath{stroke,fill}%
}%
\begin{pgfscope}%
\pgfsys@transformshift{3.649189in}{1.336155in}%
\pgfsys@useobject{currentmarker}{}%
\end{pgfscope}%
\end{pgfscope}%
\begin{pgfscope}%
\definecolor{textcolor}{rgb}{0.000000,0.000000,0.000000}%
\pgfsetstrokecolor{textcolor}%
\pgfsetfillcolor{textcolor}%
\pgftext[x=3.492938in, y=1.297575in, left, base]{\color{textcolor}{\sffamily\fontsize{8.000000}{9.600000}\selectfont\catcode`\^=\active\def^{\ifmmode\sp\else\^{}\fi}\catcode`\%=\active\def
\end{pgfscope}%
\begin{pgfscope}%
\pgfpathrectangle{\pgfqpoint{3.649189in}{0.559404in}}{\pgfqpoint{2.423729in}{1.382600in}}%
\pgfusepath{clip}%
\pgfsetrectcap%
\pgfsetroundjoin%
\pgfsetlinewidth{0.401500pt}%
\definecolor{currentstroke}{rgb}{0.690196,0.690196,0.690196}%
\pgfsetstrokecolor{currentstroke}%
\pgfsetstrokeopacity{0.350000}%
\pgfsetdash{}{0pt}%
\pgfpathmoveto{\pgfqpoint{3.649189in}{1.595072in}}%
\pgfpathlineto{\pgfqpoint{6.072918in}{1.595072in}}%
\pgfusepath{stroke}%
\end{pgfscope}%
\begin{pgfscope}%
\pgfsetbuttcap%
\pgfsetroundjoin%
\definecolor{currentfill}{rgb}{0.000000,0.000000,0.000000}%
\pgfsetfillcolor{currentfill}%
\pgfsetlinewidth{0.803000pt}%
\definecolor{currentstroke}{rgb}{0.000000,0.000000,0.000000}%
\pgfsetstrokecolor{currentstroke}%
\pgfsetdash{}{0pt}%
\pgfsys@defobject{currentmarker}{\pgfqpoint{-0.048611in}{0.000000in}}{\pgfqpoint{-0.000000in}{0.000000in}}{%
\pgfpathmoveto{\pgfqpoint{-0.000000in}{0.000000in}}%
\pgfpathlineto{\pgfqpoint{-0.048611in}{0.000000in}}%
\pgfusepath{stroke,fill}%
}%
\begin{pgfscope}%
\pgfsys@transformshift{3.649189in}{1.595072in}%
\pgfsys@useobject{currentmarker}{}%
\end{pgfscope}%
\end{pgfscope}%
\begin{pgfscope}%
\definecolor{textcolor}{rgb}{0.000000,0.000000,0.000000}%
\pgfsetstrokecolor{textcolor}%
\pgfsetfillcolor{textcolor}%
\pgftext[x=3.492938in, y=1.556492in, left, base]{\color{textcolor}{\sffamily\fontsize{8.000000}{9.600000}\selectfont\catcode`\^=\active\def^{\ifmmode\sp\else\^{}\fi}\catcode`\%=\active\def
\end{pgfscope}%
\begin{pgfscope}%
\pgfpathrectangle{\pgfqpoint{3.649189in}{0.559404in}}{\pgfqpoint{2.423729in}{1.382600in}}%
\pgfusepath{clip}%
\pgfsetrectcap%
\pgfsetroundjoin%
\pgfsetlinewidth{0.401500pt}%
\definecolor{currentstroke}{rgb}{0.690196,0.690196,0.690196}%
\pgfsetstrokecolor{currentstroke}%
\pgfsetstrokeopacity{0.350000}%
\pgfsetdash{}{0pt}%
\pgfpathmoveto{\pgfqpoint{3.649189in}{1.853989in}}%
\pgfpathlineto{\pgfqpoint{6.072918in}{1.853989in}}%
\pgfusepath{stroke}%
\end{pgfscope}%
\begin{pgfscope}%
\pgfsetbuttcap%
\pgfsetroundjoin%
\definecolor{currentfill}{rgb}{0.000000,0.000000,0.000000}%
\pgfsetfillcolor{currentfill}%
\pgfsetlinewidth{0.803000pt}%
\definecolor{currentstroke}{rgb}{0.000000,0.000000,0.000000}%
\pgfsetstrokecolor{currentstroke}%
\pgfsetdash{}{0pt}%
\pgfsys@defobject{currentmarker}{\pgfqpoint{-0.048611in}{0.000000in}}{\pgfqpoint{-0.000000in}{0.000000in}}{%
\pgfpathmoveto{\pgfqpoint{-0.000000in}{0.000000in}}%
\pgfpathlineto{\pgfqpoint{-0.048611in}{0.000000in}}%
\pgfusepath{stroke,fill}%
}%
\begin{pgfscope}%
\pgfsys@transformshift{3.649189in}{1.853989in}%
\pgfsys@useobject{currentmarker}{}%
\end{pgfscope}%
\end{pgfscope}%
\begin{pgfscope}%
\definecolor{textcolor}{rgb}{0.000000,0.000000,0.000000}%
\pgfsetstrokecolor{textcolor}%
\pgfsetfillcolor{textcolor}%
\pgftext[x=3.492938in, y=1.815408in, left, base]{\color{textcolor}{\sffamily\fontsize{8.000000}{9.600000}\selectfont\catcode`\^=\active\def^{\ifmmode\sp\else\^{}\fi}\catcode`\%=\active\def
\end{pgfscope}%
\begin{pgfscope}%
\definecolor{textcolor}{rgb}{0.000000,0.000000,0.000000}%
\pgfsetstrokecolor{textcolor}%
\pgfsetfillcolor{textcolor}%
\pgftext[x=3.437382in,y=1.250704in,,bottom,rotate=90.000000]{\color{textcolor}{\sffamily\fontsize{9.000000}{10.800000}\selectfont\catcode`\^=\active\def^{\ifmmode\sp\else\^{}\fi}\catcode`\%=\active\def
\end{pgfscope}%
\begin{pgfscope}%
\pgfpathrectangle{\pgfqpoint{3.649189in}{0.559404in}}{\pgfqpoint{2.423729in}{1.382600in}}%
\pgfusepath{clip}%
\pgfsetbuttcap%
\pgfsetroundjoin%
\pgfsetlinewidth{1.405250pt}%
\definecolor{currentstroke}{rgb}{0.000000,0.000000,0.000000}%
\pgfsetstrokecolor{currentstroke}%
\pgfsetdash{}{0pt}%
\pgfpathmoveto{\pgfqpoint{3.942974in}{0.877546in}}%
\pgfpathlineto{\pgfqpoint{3.942974in}{0.908463in}}%
\pgfusepath{stroke}%
\end{pgfscope}%
\begin{pgfscope}%
\pgfpathrectangle{\pgfqpoint{3.649189in}{0.559404in}}{\pgfqpoint{2.423729in}{1.382600in}}%
\pgfusepath{clip}%
\pgfsetbuttcap%
\pgfsetroundjoin%
\pgfsetlinewidth{1.405250pt}%
\definecolor{currentstroke}{rgb}{0.000000,0.000000,0.000000}%
\pgfsetstrokecolor{currentstroke}%
\pgfsetdash{}{0pt}%
\pgfpathmoveto{\pgfqpoint{4.402014in}{0.878124in}}%
\pgfpathlineto{\pgfqpoint{4.402014in}{0.909214in}}%
\pgfusepath{stroke}%
\end{pgfscope}%
\begin{pgfscope}%
\pgfpathrectangle{\pgfqpoint{3.649189in}{0.559404in}}{\pgfqpoint{2.423729in}{1.382600in}}%
\pgfusepath{clip}%
\pgfsetbuttcap%
\pgfsetroundjoin%
\pgfsetlinewidth{1.405250pt}%
\definecolor{currentstroke}{rgb}{0.000000,0.000000,0.000000}%
\pgfsetstrokecolor{currentstroke}%
\pgfsetdash{}{0pt}%
\pgfpathmoveto{\pgfqpoint{4.861053in}{1.094907in}}%
\pgfpathlineto{\pgfqpoint{4.861053in}{1.148007in}}%
\pgfusepath{stroke}%
\end{pgfscope}%
\begin{pgfscope}%
\pgfpathrectangle{\pgfqpoint{3.649189in}{0.559404in}}{\pgfqpoint{2.423729in}{1.382600in}}%
\pgfusepath{clip}%
\pgfsetbuttcap%
\pgfsetroundjoin%
\pgfsetlinewidth{1.405250pt}%
\definecolor{currentstroke}{rgb}{0.000000,0.000000,0.000000}%
\pgfsetstrokecolor{currentstroke}%
\pgfsetdash{}{0pt}%
\pgfpathmoveto{\pgfqpoint{5.320093in}{0.881338in}}%
\pgfpathlineto{\pgfqpoint{5.320093in}{0.913172in}}%
\pgfusepath{stroke}%
\end{pgfscope}%
\begin{pgfscope}%
\pgfpathrectangle{\pgfqpoint{3.649189in}{0.559404in}}{\pgfqpoint{2.423729in}{1.382600in}}%
\pgfusepath{clip}%
\pgfsetbuttcap%
\pgfsetroundjoin%
\pgfsetlinewidth{1.405250pt}%
\definecolor{currentstroke}{rgb}{0.000000,0.000000,0.000000}%
\pgfsetstrokecolor{currentstroke}%
\pgfsetdash{}{0pt}%
\pgfpathmoveto{\pgfqpoint{5.779132in}{1.782424in}}%
\pgfpathlineto{\pgfqpoint{5.779132in}{1.876166in}}%
\pgfusepath{stroke}%
\end{pgfscope}%
\begin{pgfscope}%
\pgfpathrectangle{\pgfqpoint{3.649189in}{0.559404in}}{\pgfqpoint{2.423729in}{1.382600in}}%
\pgfusepath{clip}%
\pgfsetbuttcap%
\pgfsetroundjoin%
\definecolor{currentfill}{rgb}{0.121569,0.466667,0.705882}%
\pgfsetfillcolor{currentfill}%
\pgfsetlinewidth{1.003750pt}%
\definecolor{currentstroke}{rgb}{0.000000,0.000000,0.000000}%
\pgfsetstrokecolor{currentstroke}%
\pgfsetdash{}{0pt}%
\pgfsys@defobject{currentmarker}{\pgfqpoint{-0.027778in}{-0.000000in}}{\pgfqpoint{0.027778in}{0.000000in}}{%
\pgfpathmoveto{\pgfqpoint{0.027778in}{-0.000000in}}%
\pgfpathlineto{\pgfqpoint{-0.027778in}{0.000000in}}%
\pgfusepath{stroke,fill}%
}%
\begin{pgfscope}%
\pgfsys@transformshift{3.942974in}{0.877546in}%
\pgfsys@useobject{currentmarker}{}%
\end{pgfscope}%
\begin{pgfscope}%
\pgfsys@transformshift{4.402014in}{0.878124in}%
\pgfsys@useobject{currentmarker}{}%
\end{pgfscope}%
\begin{pgfscope}%
\pgfsys@transformshift{4.861053in}{1.094907in}%
\pgfsys@useobject{currentmarker}{}%
\end{pgfscope}%
\begin{pgfscope}%
\pgfsys@transformshift{5.320093in}{0.881338in}%
\pgfsys@useobject{currentmarker}{}%
\end{pgfscope}%
\begin{pgfscope}%
\pgfsys@transformshift{5.779132in}{1.782424in}%
\pgfsys@useobject{currentmarker}{}%
\end{pgfscope}%
\end{pgfscope}%
\begin{pgfscope}%
\pgfpathrectangle{\pgfqpoint{3.649189in}{0.559404in}}{\pgfqpoint{2.423729in}{1.382600in}}%
\pgfusepath{clip}%
\pgfsetbuttcap%
\pgfsetroundjoin%
\definecolor{currentfill}{rgb}{0.121569,0.466667,0.705882}%
\pgfsetfillcolor{currentfill}%
\pgfsetlinewidth{1.003750pt}%
\definecolor{currentstroke}{rgb}{0.000000,0.000000,0.000000}%
\pgfsetstrokecolor{currentstroke}%
\pgfsetdash{}{0pt}%
\pgfsys@defobject{currentmarker}{\pgfqpoint{-0.027778in}{-0.000000in}}{\pgfqpoint{0.027778in}{0.000000in}}{%
\pgfpathmoveto{\pgfqpoint{0.027778in}{-0.000000in}}%
\pgfpathlineto{\pgfqpoint{-0.027778in}{0.000000in}}%
\pgfusepath{stroke,fill}%
}%
\begin{pgfscope}%
\pgfsys@transformshift{3.942974in}{0.908463in}%
\pgfsys@useobject{currentmarker}{}%
\end{pgfscope}%
\begin{pgfscope}%
\pgfsys@transformshift{4.402014in}{0.909214in}%
\pgfsys@useobject{currentmarker}{}%
\end{pgfscope}%
\begin{pgfscope}%
\pgfsys@transformshift{4.861053in}{1.148007in}%
\pgfsys@useobject{currentmarker}{}%
\end{pgfscope}%
\begin{pgfscope}%
\pgfsys@transformshift{5.320093in}{0.913172in}%
\pgfsys@useobject{currentmarker}{}%
\end{pgfscope}%
\begin{pgfscope}%
\pgfsys@transformshift{5.779132in}{1.876166in}%
\pgfsys@useobject{currentmarker}{}%
\end{pgfscope}%
\end{pgfscope}%
\begin{pgfscope}%
\pgfsetrectcap%
\pgfsetmiterjoin%
\pgfsetlinewidth{0.803000pt}%
\definecolor{currentstroke}{rgb}{0.000000,0.000000,0.000000}%
\pgfsetstrokecolor{currentstroke}%
\pgfsetdash{}{0pt}%
\pgfpathmoveto{\pgfqpoint{3.649189in}{0.559404in}}%
\pgfpathlineto{\pgfqpoint{3.649189in}{1.942004in}}%
\pgfusepath{stroke}%
\end{pgfscope}%
\begin{pgfscope}%
\pgfsetrectcap%
\pgfsetmiterjoin%
\pgfsetlinewidth{0.803000pt}%
\definecolor{currentstroke}{rgb}{0.000000,0.000000,0.000000}%
\pgfsetstrokecolor{currentstroke}%
\pgfsetdash{}{0pt}%
\pgfpathmoveto{\pgfqpoint{3.649189in}{0.559404in}}%
\pgfpathlineto{\pgfqpoint{6.072918in}{0.559404in}}%
\pgfusepath{stroke}%
\end{pgfscope}%
\begin{pgfscope}%
\definecolor{textcolor}{rgb}{0.000000,0.000000,0.000000}%
\pgfsetstrokecolor{textcolor}%
\pgfsetfillcolor{textcolor}%
\pgftext[x=0.352917in,y=1.997754in,left,bottom]{\color{textcolor}{\sffamily\fontsize{10.000000}{12.000000}\bfseries\selectfont\catcode`\^=\active\def^{\ifmmode\sp\else\^{}\fi}\catcode`\%=\active\def
\end{pgfscope}%
\begin{pgfscope}%
\definecolor{textcolor}{rgb}{0.000000,0.000000,0.000000}%
\pgfsetstrokecolor{textcolor}%
\pgfsetfillcolor{textcolor}%
\pgftext[x=3.649189in,y=1.997754in,left,bottom]{\color{textcolor}{\sffamily\fontsize{10.000000}{12.000000}\bfseries\selectfont\catcode`\^=\active\def^{\ifmmode\sp\else\^{}\fi}\catcode`\%=\active\def
\end{pgfscope}%
\end{pgfpicture}%
\makeatother%
\endgroup%

%% file: report/spd_synthetic/spd_intrinsic_iterations.pgf
\begingroup%
\makeatletter%
\begin{pgfpicture}%
\pgfpathrectangle{\pgfpointorigin}{\pgfqpoint{3.374918in}{2.014289in}}%
\pgfusepath{use as bounding box, clip}%
\begin{pgfscope}%
\pgfsetbuttcap%
\pgfsetmiterjoin%
\definecolor{currentfill}{rgb}{1.000000,1.000000,1.000000}%
\pgfsetfillcolor{currentfill}%
\pgfsetlinewidth{0.000000pt}%
\definecolor{currentstroke}{rgb}{1.000000,1.000000,1.000000}%
\pgfsetstrokecolor{currentstroke}%
\pgfsetdash{}{0pt}%
\pgfpathmoveto{\pgfqpoint{0.000000in}{0.000000in}}%
\pgfpathlineto{\pgfqpoint{3.374917in}{0.000000in}}%
\pgfpathlineto{\pgfqpoint{3.374917in}{2.014289in}}%
\pgfpathlineto{\pgfqpoint{0.000000in}{2.014289in}}%
\pgfpathlineto{\pgfqpoint{0.000000in}{0.000000in}}%
\pgfpathclose%
\pgfusepath{fill}%
\end{pgfscope}%
\begin{pgfscope}%
\pgfsetbuttcap%
\pgfsetmiterjoin%
\definecolor{currentfill}{rgb}{1.000000,1.000000,1.000000}%
\pgfsetfillcolor{currentfill}%
\pgfsetlinewidth{0.000000pt}%
\definecolor{currentstroke}{rgb}{0.000000,0.000000,0.000000}%
\pgfsetstrokecolor{currentstroke}%
\pgfsetstrokeopacity{0.000000}%
\pgfsetdash{}{0pt}%
\pgfpathmoveto{\pgfqpoint{0.352918in}{0.570689in}}%
\pgfpathlineto{\pgfqpoint{3.344918in}{0.570689in}}%
\pgfpathlineto{\pgfqpoint{3.344918in}{1.984289in}}%
\pgfpathlineto{\pgfqpoint{0.352918in}{1.984289in}}%
\pgfpathlineto{\pgfqpoint{0.352918in}{0.570689in}}%
\pgfpathclose%
\pgfusepath{fill}%
\end{pgfscope}%
\begin{pgfscope}%
\pgfpathrectangle{\pgfqpoint{0.352918in}{0.570689in}}{\pgfqpoint{2.992000in}{1.413600in}}%
\pgfusepath{clip}%
\pgfsetbuttcap%
\pgfsetmiterjoin%
\definecolor{currentfill}{rgb}{0.121569,0.466667,0.705882}%
\pgfsetfillcolor{currentfill}%
\pgfsetlinewidth{0.000000pt}%
\definecolor{currentstroke}{rgb}{0.000000,0.000000,0.000000}%
\pgfsetstrokecolor{currentstroke}%
\pgfsetstrokeopacity{0.000000}%
\pgfsetdash{}{0pt}%
\pgfpathmoveto{\pgfqpoint{0.488917in}{0.570689in}}%
\pgfpathlineto{\pgfqpoint{1.697806in}{0.570689in}}%
\pgfpathlineto{\pgfqpoint{1.697806in}{1.564131in}}%
\pgfpathlineto{\pgfqpoint{0.488917in}{1.564131in}}%
\pgfpathlineto{\pgfqpoint{0.488917in}{0.570689in}}%
\pgfpathclose%
\pgfusepath{fill}%
\end{pgfscope}%
\begin{pgfscope}%
\pgfpathrectangle{\pgfqpoint{0.352918in}{0.570689in}}{\pgfqpoint{2.992000in}{1.413600in}}%
\pgfusepath{clip}%
\pgfsetbuttcap%
\pgfsetmiterjoin%
\definecolor{currentfill}{rgb}{0.121569,0.466667,0.705882}%
\pgfsetfillcolor{currentfill}%
\pgfsetlinewidth{0.000000pt}%
\definecolor{currentstroke}{rgb}{0.000000,0.000000,0.000000}%
\pgfsetstrokecolor{currentstroke}%
\pgfsetstrokeopacity{0.000000}%
\pgfsetdash{}{0pt}%
\pgfpathmoveto{\pgfqpoint{2.000029in}{0.570689in}}%
\pgfpathlineto{\pgfqpoint{3.208917in}{0.570689in}}%
\pgfpathlineto{\pgfqpoint{3.208917in}{1.888383in}}%
\pgfpathlineto{\pgfqpoint{2.000029in}{1.888383in}}%
\pgfpathlineto{\pgfqpoint{2.000029in}{0.570689in}}%
\pgfpathclose%
\pgfusepath{fill}%
\end{pgfscope}%
\begin{pgfscope}%
\pgfsetbuttcap%
\pgfsetroundjoin%
\definecolor{currentfill}{rgb}{0.000000,0.000000,0.000000}%
\pgfsetfillcolor{currentfill}%
\pgfsetlinewidth{0.803000pt}%
\definecolor{currentstroke}{rgb}{0.000000,0.000000,0.000000}%
\pgfsetstrokecolor{currentstroke}%
\pgfsetdash{}{0pt}%
\pgfsys@defobject{currentmarker}{\pgfqpoint{0.000000in}{-0.048611in}}{\pgfqpoint{0.000000in}{0.000000in}}{%
\pgfpathmoveto{\pgfqpoint{0.000000in}{0.000000in}}%
\pgfpathlineto{\pgfqpoint{0.000000in}{-0.048611in}}%
\pgfusepath{stroke,fill}%
}%
\begin{pgfscope}%
\pgfsys@transformshift{1.093362in}{0.570689in}%
\pgfsys@useobject{currentmarker}{}%
\end{pgfscope}%
\end{pgfscope}%
\begin{pgfscope}%
\definecolor{textcolor}{rgb}{0.000000,0.000000,0.000000}%
\pgfsetstrokecolor{textcolor}%
\pgfsetfillcolor{textcolor}%
\pgftext[x=0.462601in, y=0.048710in, left, base,rotate=30.000000]{\color{textcolor}{\sffamily\fontsize{8.000000}{9.600000}\selectfont\catcode`\^=\active\def^{\ifmmode\sp\else\^{}\fi}\catcode`\%=\active\def
\end{pgfscope}%
\begin{pgfscope}%
\pgfsetbuttcap%
\pgfsetroundjoin%
\definecolor{currentfill}{rgb}{0.000000,0.000000,0.000000}%
\pgfsetfillcolor{currentfill}%
\pgfsetlinewidth{0.803000pt}%
\definecolor{currentstroke}{rgb}{0.000000,0.000000,0.000000}%
\pgfsetstrokecolor{currentstroke}%
\pgfsetdash{}{0pt}%
\pgfsys@defobject{currentmarker}{\pgfqpoint{0.000000in}{-0.048611in}}{\pgfqpoint{0.000000in}{0.000000in}}{%
\pgfpathmoveto{\pgfqpoint{0.000000in}{0.000000in}}%
\pgfpathlineto{\pgfqpoint{0.000000in}{-0.048611in}}%
\pgfusepath{stroke,fill}%
}%
\begin{pgfscope}%
\pgfsys@transformshift{2.604473in}{0.570689in}%
\pgfsys@useobject{currentmarker}{}%
\end{pgfscope}%
\end{pgfscope}%
\begin{pgfscope}%
\definecolor{textcolor}{rgb}{0.000000,0.000000,0.000000}%
\pgfsetstrokecolor{textcolor}%
\pgfsetfillcolor{textcolor}%
\pgftext[x=2.134591in, y=0.141594in, left, base,rotate=30.000000]{\color{textcolor}{\sffamily\fontsize{8.000000}{9.600000}\selectfont\catcode`\^=\active\def^{\ifmmode\sp\else\^{}\fi}\catcode`\%=\active\def
\end{pgfscope}%
\begin{pgfscope}%
\pgfpathrectangle{\pgfqpoint{0.352918in}{0.570689in}}{\pgfqpoint{2.992000in}{1.413600in}}%
\pgfusepath{clip}%
\pgfsetrectcap%
\pgfsetroundjoin%
\pgfsetlinewidth{0.401500pt}%
\definecolor{currentstroke}{rgb}{0.690196,0.690196,0.690196}%
\pgfsetstrokecolor{currentstroke}%
\pgfsetstrokeopacity{0.350000}%
\pgfsetdash{}{0pt}%
\pgfpathmoveto{\pgfqpoint{0.352918in}{0.570689in}}%
\pgfpathlineto{\pgfqpoint{3.344918in}{0.570689in}}%
\pgfusepath{stroke}%
\end{pgfscope}%
\begin{pgfscope}%
\pgfsetbuttcap%
\pgfsetroundjoin%
\definecolor{currentfill}{rgb}{0.000000,0.000000,0.000000}%
\pgfsetfillcolor{currentfill}%
\pgfsetlinewidth{0.803000pt}%
\definecolor{currentstroke}{rgb}{0.000000,0.000000,0.000000}%
\pgfsetstrokecolor{currentstroke}%
\pgfsetdash{}{0pt}%
\pgfsys@defobject{currentmarker}{\pgfqpoint{-0.048611in}{0.000000in}}{\pgfqpoint{-0.000000in}{0.000000in}}{%
\pgfpathmoveto{\pgfqpoint{-0.000000in}{0.000000in}}%
\pgfpathlineto{\pgfqpoint{-0.048611in}{0.000000in}}%
\pgfusepath{stroke,fill}%
}%
\begin{pgfscope}%
\pgfsys@transformshift{0.352918in}{0.570689in}%
\pgfsys@useobject{currentmarker}{}%
\end{pgfscope}%
\end{pgfscope}%
\begin{pgfscope}%
\definecolor{textcolor}{rgb}{0.000000,0.000000,0.000000}%
\pgfsetstrokecolor{textcolor}%
\pgfsetfillcolor{textcolor}%
\pgftext[x=0.196667in, y=0.532109in, left, base]{\color{textcolor}{\sffamily\fontsize{8.000000}{9.600000}\selectfont\catcode`\^=\active\def^{\ifmmode\sp\else\^{}\fi}\catcode`\%=\active\def
\end{pgfscope}%
\begin{pgfscope}%
\pgfpathrectangle{\pgfqpoint{0.352918in}{0.570689in}}{\pgfqpoint{2.992000in}{1.413600in}}%
\pgfusepath{clip}%
\pgfsetrectcap%
\pgfsetroundjoin%
\pgfsetlinewidth{0.401500pt}%
\definecolor{currentstroke}{rgb}{0.690196,0.690196,0.690196}%
\pgfsetstrokecolor{currentstroke}%
\pgfsetstrokeopacity{0.350000}%
\pgfsetdash{}{0pt}%
\pgfpathmoveto{\pgfqpoint{0.352918in}{0.894940in}}%
\pgfpathlineto{\pgfqpoint{3.344918in}{0.894940in}}%
\pgfusepath{stroke}%
\end{pgfscope}%
\begin{pgfscope}%
\pgfsetbuttcap%
\pgfsetroundjoin%
\definecolor{currentfill}{rgb}{0.000000,0.000000,0.000000}%
\pgfsetfillcolor{currentfill}%
\pgfsetlinewidth{0.803000pt}%
\definecolor{currentstroke}{rgb}{0.000000,0.000000,0.000000}%
\pgfsetstrokecolor{currentstroke}%
\pgfsetdash{}{0pt}%
\pgfsys@defobject{currentmarker}{\pgfqpoint{-0.048611in}{0.000000in}}{\pgfqpoint{-0.000000in}{0.000000in}}{%
\pgfpathmoveto{\pgfqpoint{-0.000000in}{0.000000in}}%
\pgfpathlineto{\pgfqpoint{-0.048611in}{0.000000in}}%
\pgfusepath{stroke,fill}%
}%
\begin{pgfscope}%
\pgfsys@transformshift{0.352918in}{0.894940in}%
\pgfsys@useobject{currentmarker}{}%
\end{pgfscope}%
\end{pgfscope}%
\begin{pgfscope}%
\definecolor{textcolor}{rgb}{0.000000,0.000000,0.000000}%
\pgfsetstrokecolor{textcolor}%
\pgfsetfillcolor{textcolor}%
\pgftext[x=0.196667in, y=0.856360in, left, base]{\color{textcolor}{\sffamily\fontsize{8.000000}{9.600000}\selectfont\catcode`\^=\active\def^{\ifmmode\sp\else\^{}\fi}\catcode`\%=\active\def
\end{pgfscope}%
\begin{pgfscope}%
\pgfpathrectangle{\pgfqpoint{0.352918in}{0.570689in}}{\pgfqpoint{2.992000in}{1.413600in}}%
\pgfusepath{clip}%
\pgfsetrectcap%
\pgfsetroundjoin%
\pgfsetlinewidth{0.401500pt}%
\definecolor{currentstroke}{rgb}{0.690196,0.690196,0.690196}%
\pgfsetstrokecolor{currentstroke}%
\pgfsetstrokeopacity{0.350000}%
\pgfsetdash{}{0pt}%
\pgfpathmoveto{\pgfqpoint{0.352918in}{1.219192in}}%
\pgfpathlineto{\pgfqpoint{3.344918in}{1.219192in}}%
\pgfusepath{stroke}%
\end{pgfscope}%
\begin{pgfscope}%
\pgfsetbuttcap%
\pgfsetroundjoin%
\definecolor{currentfill}{rgb}{0.000000,0.000000,0.000000}%
\pgfsetfillcolor{currentfill}%
\pgfsetlinewidth{0.803000pt}%
\definecolor{currentstroke}{rgb}{0.000000,0.000000,0.000000}%
\pgfsetstrokecolor{currentstroke}%
\pgfsetdash{}{0pt}%
\pgfsys@defobject{currentmarker}{\pgfqpoint{-0.048611in}{0.000000in}}{\pgfqpoint{-0.000000in}{0.000000in}}{%
\pgfpathmoveto{\pgfqpoint{-0.000000in}{0.000000in}}%
\pgfpathlineto{\pgfqpoint{-0.048611in}{0.000000in}}%
\pgfusepath{stroke,fill}%
}%
\begin{pgfscope}%
\pgfsys@transformshift{0.352918in}{1.219192in}%
\pgfsys@useobject{currentmarker}{}%
\end{pgfscope}%
\end{pgfscope}%
\begin{pgfscope}%
\definecolor{textcolor}{rgb}{0.000000,0.000000,0.000000}%
\pgfsetstrokecolor{textcolor}%
\pgfsetfillcolor{textcolor}%
\pgftext[x=0.196667in, y=1.180612in, left, base]{\color{textcolor}{\sffamily\fontsize{8.000000}{9.600000}\selectfont\catcode`\^=\active\def^{\ifmmode\sp\else\^{}\fi}\catcode`\%=\active\def
\end{pgfscope}%
\begin{pgfscope}%
\pgfpathrectangle{\pgfqpoint{0.352918in}{0.570689in}}{\pgfqpoint{2.992000in}{1.413600in}}%
\pgfusepath{clip}%
\pgfsetrectcap%
\pgfsetroundjoin%
\pgfsetlinewidth{0.401500pt}%
\definecolor{currentstroke}{rgb}{0.690196,0.690196,0.690196}%
\pgfsetstrokecolor{currentstroke}%
\pgfsetstrokeopacity{0.350000}%
\pgfsetdash{}{0pt}%
\pgfpathmoveto{\pgfqpoint{0.352918in}{1.543443in}}%
\pgfpathlineto{\pgfqpoint{3.344918in}{1.543443in}}%
\pgfusepath{stroke}%
\end{pgfscope}%
\begin{pgfscope}%
\pgfsetbuttcap%
\pgfsetroundjoin%
\definecolor{currentfill}{rgb}{0.000000,0.000000,0.000000}%
\pgfsetfillcolor{currentfill}%
\pgfsetlinewidth{0.803000pt}%
\definecolor{currentstroke}{rgb}{0.000000,0.000000,0.000000}%
\pgfsetstrokecolor{currentstroke}%
\pgfsetdash{}{0pt}%
\pgfsys@defobject{currentmarker}{\pgfqpoint{-0.048611in}{0.000000in}}{\pgfqpoint{-0.000000in}{0.000000in}}{%
\pgfpathmoveto{\pgfqpoint{-0.000000in}{0.000000in}}%
\pgfpathlineto{\pgfqpoint{-0.048611in}{0.000000in}}%
\pgfusepath{stroke,fill}%
}%
\begin{pgfscope}%
\pgfsys@transformshift{0.352918in}{1.543443in}%
\pgfsys@useobject{currentmarker}{}%
\end{pgfscope}%
\end{pgfscope}%
\begin{pgfscope}%
\definecolor{textcolor}{rgb}{0.000000,0.000000,0.000000}%
\pgfsetstrokecolor{textcolor}%
\pgfsetfillcolor{textcolor}%
\pgftext[x=0.196667in, y=1.504863in, left, base]{\color{textcolor}{\sffamily\fontsize{8.000000}{9.600000}\selectfont\catcode`\^=\active\def^{\ifmmode\sp\else\^{}\fi}\catcode`\%=\active\def
\end{pgfscope}%
\begin{pgfscope}%
\pgfpathrectangle{\pgfqpoint{0.352918in}{0.570689in}}{\pgfqpoint{2.992000in}{1.413600in}}%
\pgfusepath{clip}%
\pgfsetrectcap%
\pgfsetroundjoin%
\pgfsetlinewidth{0.401500pt}%
\definecolor{currentstroke}{rgb}{0.690196,0.690196,0.690196}%
\pgfsetstrokecolor{currentstroke}%
\pgfsetstrokeopacity{0.350000}%
\pgfsetdash{}{0pt}%
\pgfpathmoveto{\pgfqpoint{0.352918in}{1.867695in}}%
\pgfpathlineto{\pgfqpoint{3.344918in}{1.867695in}}%
\pgfusepath{stroke}%
\end{pgfscope}%
\begin{pgfscope}%
\pgfsetbuttcap%
\pgfsetroundjoin%
\definecolor{currentfill}{rgb}{0.000000,0.000000,0.000000}%
\pgfsetfillcolor{currentfill}%
\pgfsetlinewidth{0.803000pt}%
\definecolor{currentstroke}{rgb}{0.000000,0.000000,0.000000}%
\pgfsetstrokecolor{currentstroke}%
\pgfsetdash{}{0pt}%
\pgfsys@defobject{currentmarker}{\pgfqpoint{-0.048611in}{0.000000in}}{\pgfqpoint{-0.000000in}{0.000000in}}{%
\pgfpathmoveto{\pgfqpoint{-0.000000in}{0.000000in}}%
\pgfpathlineto{\pgfqpoint{-0.048611in}{0.000000in}}%
\pgfusepath{stroke,fill}%
}%
\begin{pgfscope}%
\pgfsys@transformshift{0.352918in}{1.867695in}%
\pgfsys@useobject{currentmarker}{}%
\end{pgfscope}%
\end{pgfscope}%
\begin{pgfscope}%
\definecolor{textcolor}{rgb}{0.000000,0.000000,0.000000}%
\pgfsetstrokecolor{textcolor}%
\pgfsetfillcolor{textcolor}%
\pgftext[x=0.196667in, y=1.829115in, left, base]{\color{textcolor}{\sffamily\fontsize{8.000000}{9.600000}\selectfont\catcode`\^=\active\def^{\ifmmode\sp\else\^{}\fi}\catcode`\%=\active\def
\end{pgfscope}%
\begin{pgfscope}%
\definecolor{textcolor}{rgb}{0.000000,0.000000,0.000000}%
\pgfsetstrokecolor{textcolor}%
\pgfsetfillcolor{textcolor}%
\pgftext[x=0.141111in,y=1.277489in,,bottom,rotate=90.000000]{\color{textcolor}{\sffamily\fontsize{9.000000}{10.800000}\selectfont\catcode`\^=\active\def^{\ifmmode\sp\else\^{}\fi}\catcode`\%=\active\def
\end{pgfscope}%
\begin{pgfscope}%
\pgfpathrectangle{\pgfqpoint{0.352918in}{0.570689in}}{\pgfqpoint{2.992000in}{1.413600in}}%
\pgfusepath{clip}%
\pgfsetbuttcap%
\pgfsetroundjoin%
\pgfsetlinewidth{1.405250pt}%
\definecolor{currentstroke}{rgb}{0.000000,0.000000,0.000000}%
\pgfsetstrokecolor{currentstroke}%
\pgfsetdash{}{0pt}%
\pgfpathmoveto{\pgfqpoint{1.093362in}{1.535540in}}%
\pgfpathlineto{\pgfqpoint{1.093362in}{1.592723in}}%
\pgfusepath{stroke}%
\end{pgfscope}%
\begin{pgfscope}%
\pgfpathrectangle{\pgfqpoint{0.352918in}{0.570689in}}{\pgfqpoint{2.992000in}{1.413600in}}%
\pgfusepath{clip}%
\pgfsetbuttcap%
\pgfsetroundjoin%
\pgfsetlinewidth{1.405250pt}%
\definecolor{currentstroke}{rgb}{0.000000,0.000000,0.000000}%
\pgfsetstrokecolor{currentstroke}%
\pgfsetdash{}{0pt}%
\pgfpathmoveto{\pgfqpoint{2.604473in}{1.859791in}}%
\pgfpathlineto{\pgfqpoint{2.604473in}{1.916975in}}%
\pgfusepath{stroke}%
\end{pgfscope}%
\begin{pgfscope}%
\pgfpathrectangle{\pgfqpoint{0.352918in}{0.570689in}}{\pgfqpoint{2.992000in}{1.413600in}}%
\pgfusepath{clip}%
\pgfsetbuttcap%
\pgfsetroundjoin%
\definecolor{currentfill}{rgb}{0.121569,0.466667,0.705882}%
\pgfsetfillcolor{currentfill}%
\pgfsetlinewidth{1.003750pt}%
\definecolor{currentstroke}{rgb}{0.000000,0.000000,0.000000}%
\pgfsetstrokecolor{currentstroke}%
\pgfsetdash{}{0pt}%
\pgfsys@defobject{currentmarker}{\pgfqpoint{-0.027778in}{-0.000000in}}{\pgfqpoint{0.027778in}{0.000000in}}{%
\pgfpathmoveto{\pgfqpoint{0.027778in}{-0.000000in}}%
\pgfpathlineto{\pgfqpoint{-0.027778in}{0.000000in}}%
\pgfusepath{stroke,fill}%
}%
\begin{pgfscope}%
\pgfsys@transformshift{1.093362in}{1.535540in}%
\pgfsys@useobject{currentmarker}{}%
\end{pgfscope}%
\begin{pgfscope}%
\pgfsys@transformshift{2.604473in}{1.859791in}%
\pgfsys@useobject{currentmarker}{}%
\end{pgfscope}%
\end{pgfscope}%
\begin{pgfscope}%
\pgfpathrectangle{\pgfqpoint{0.352918in}{0.570689in}}{\pgfqpoint{2.992000in}{1.413600in}}%
\pgfusepath{clip}%
\pgfsetbuttcap%
\pgfsetroundjoin%
\definecolor{currentfill}{rgb}{0.121569,0.466667,0.705882}%
\pgfsetfillcolor{currentfill}%
\pgfsetlinewidth{1.003750pt}%
\definecolor{currentstroke}{rgb}{0.000000,0.000000,0.000000}%
\pgfsetstrokecolor{currentstroke}%
\pgfsetdash{}{0pt}%
\pgfsys@defobject{currentmarker}{\pgfqpoint{-0.027778in}{-0.000000in}}{\pgfqpoint{0.027778in}{0.000000in}}{%
\pgfpathmoveto{\pgfqpoint{0.027778in}{-0.000000in}}%
\pgfpathlineto{\pgfqpoint{-0.027778in}{0.000000in}}%
\pgfusepath{stroke,fill}%
}%
\begin{pgfscope}%
\pgfsys@transformshift{1.093362in}{1.592723in}%
\pgfsys@useobject{currentmarker}{}%
\end{pgfscope}%
\begin{pgfscope}%
\pgfsys@transformshift{2.604473in}{1.916975in}%
\pgfsys@useobject{currentmarker}{}%
\end{pgfscope}%
\end{pgfscope}%
\begin{pgfscope}%
\pgfsetrectcap%
\pgfsetmiterjoin%
\pgfsetlinewidth{0.803000pt}%
\definecolor{currentstroke}{rgb}{0.000000,0.000000,0.000000}%
\pgfsetstrokecolor{currentstroke}%
\pgfsetdash{}{0pt}%
\pgfpathmoveto{\pgfqpoint{0.352918in}{0.570689in}}%
\pgfpathlineto{\pgfqpoint{0.352918in}{1.984289in}}%
\pgfusepath{stroke}%
\end{pgfscope}%
\begin{pgfscope}%
\pgfsetrectcap%
\pgfsetmiterjoin%
\pgfsetlinewidth{0.803000pt}%
\definecolor{currentstroke}{rgb}{0.000000,0.000000,0.000000}%
\pgfsetstrokecolor{currentstroke}%
\pgfsetdash{}{0pt}%
\pgfpathmoveto{\pgfqpoint{0.352917in}{0.570689in}}%
\pgfpathlineto{\pgfqpoint{3.344918in}{0.570689in}}%
\pgfusepath{stroke}%
\end{pgfscope}%
\end{pgfpicture}%
\makeatother%
\endgroup%

%% file: report/eeg_spd/eeg_downstream_metrics.pgf
\begingroup%
\makeatletter%
\begin{pgfpicture}%
\pgfpathrectangle{\pgfpointorigin}{\pgfqpoint{6.194740in}{2.166643in}}%
\pgfusepath{use as bounding box, clip}%
\begin{pgfscope}%
\pgfsetbuttcap%
\pgfsetmiterjoin%
\definecolor{currentfill}{rgb}{1.000000,1.000000,1.000000}%
\pgfsetfillcolor{currentfill}%
\pgfsetlinewidth{0.000000pt}%
\definecolor{currentstroke}{rgb}{1.000000,1.000000,1.000000}%
\pgfsetstrokecolor{currentstroke}%
\pgfsetdash{}{0pt}%
\pgfpathmoveto{\pgfqpoint{0.000000in}{0.000000in}}%
\pgfpathlineto{\pgfqpoint{6.194740in}{0.000000in}}%
\pgfpathlineto{\pgfqpoint{6.194740in}{2.166643in}}%
\pgfpathlineto{\pgfqpoint{0.000000in}{2.166643in}}%
\pgfpathlineto{\pgfqpoint{0.000000in}{0.000000in}}%
\pgfpathclose%
\pgfusepath{fill}%
\end{pgfscope}%
\begin{pgfscope}%
\pgfsetbuttcap%
\pgfsetmiterjoin%
\definecolor{currentfill}{rgb}{1.000000,1.000000,1.000000}%
\pgfsetfillcolor{currentfill}%
\pgfsetlinewidth{0.000000pt}%
\definecolor{currentstroke}{rgb}{0.000000,0.000000,0.000000}%
\pgfsetstrokecolor{currentstroke}%
\pgfsetstrokeopacity{0.000000}%
\pgfsetdash{}{0pt}%
\pgfpathmoveto{\pgfqpoint{0.444740in}{0.559404in}}%
\pgfpathlineto{\pgfqpoint{2.868469in}{0.559404in}}%
\pgfpathlineto{\pgfqpoint{2.868469in}{1.942004in}}%
\pgfpathlineto{\pgfqpoint{0.444740in}{1.942004in}}%
\pgfpathlineto{\pgfqpoint{0.444740in}{0.559404in}}%
\pgfpathclose%
\pgfusepath{fill}%
\end{pgfscope}%
\begin{pgfscope}%
\pgfpathrectangle{\pgfqpoint{0.444740in}{0.559404in}}{\pgfqpoint{2.423729in}{1.382600in}}%
\pgfusepath{clip}%
\pgfsetbuttcap%
\pgfsetmiterjoin%
\definecolor{currentfill}{rgb}{0.121569,0.466667,0.705882}%
\pgfsetfillcolor{currentfill}%
\pgfsetlinewidth{0.000000pt}%
\definecolor{currentstroke}{rgb}{0.000000,0.000000,0.000000}%
\pgfsetstrokecolor{currentstroke}%
\pgfsetstrokeopacity{0.000000}%
\pgfsetdash{}{0pt}%
\pgfpathmoveto{\pgfqpoint{0.554909in}{0.559404in}}%
\pgfpathlineto{\pgfqpoint{0.922141in}{0.559404in}}%
\pgfpathlineto{\pgfqpoint{0.922141in}{1.795514in}}%
\pgfpathlineto{\pgfqpoint{0.554909in}{1.795514in}}%
\pgfpathlineto{\pgfqpoint{0.554909in}{0.559404in}}%
\pgfpathclose%
\pgfusepath{fill}%
\end{pgfscope}%
\begin{pgfscope}%
\pgfpathrectangle{\pgfqpoint{0.444740in}{0.559404in}}{\pgfqpoint{2.423729in}{1.382600in}}%
\pgfusepath{clip}%
\pgfsetbuttcap%
\pgfsetmiterjoin%
\definecolor{currentfill}{rgb}{0.121569,0.466667,0.705882}%
\pgfsetfillcolor{currentfill}%
\pgfsetlinewidth{0.000000pt}%
\definecolor{currentstroke}{rgb}{0.000000,0.000000,0.000000}%
\pgfsetstrokecolor{currentstroke}%
\pgfsetstrokeopacity{0.000000}%
\pgfsetdash{}{0pt}%
\pgfpathmoveto{\pgfqpoint{1.013949in}{0.559404in}}%
\pgfpathlineto{\pgfqpoint{1.381181in}{0.559404in}}%
\pgfpathlineto{\pgfqpoint{1.381181in}{1.790331in}}%
\pgfpathlineto{\pgfqpoint{1.013949in}{1.790331in}}%
\pgfpathlineto{\pgfqpoint{1.013949in}{0.559404in}}%
\pgfpathclose%
\pgfusepath{fill}%
\end{pgfscope}%
\begin{pgfscope}%
\pgfpathrectangle{\pgfqpoint{0.444740in}{0.559404in}}{\pgfqpoint{2.423729in}{1.382600in}}%
\pgfusepath{clip}%
\pgfsetbuttcap%
\pgfsetmiterjoin%
\definecolor{currentfill}{rgb}{0.121569,0.466667,0.705882}%
\pgfsetfillcolor{currentfill}%
\pgfsetlinewidth{0.000000pt}%
\definecolor{currentstroke}{rgb}{0.000000,0.000000,0.000000}%
\pgfsetstrokecolor{currentstroke}%
\pgfsetstrokeopacity{0.000000}%
\pgfsetdash{}{0pt}%
\pgfpathmoveto{\pgfqpoint{1.472988in}{0.559404in}}%
\pgfpathlineto{\pgfqpoint{1.840220in}{0.559404in}}%
\pgfpathlineto{\pgfqpoint{1.840220in}{1.767872in}}%
\pgfpathlineto{\pgfqpoint{1.472988in}{1.767872in}}%
\pgfpathlineto{\pgfqpoint{1.472988in}{0.559404in}}%
\pgfpathclose%
\pgfusepath{fill}%
\end{pgfscope}%
\begin{pgfscope}%
\pgfpathrectangle{\pgfqpoint{0.444740in}{0.559404in}}{\pgfqpoint{2.423729in}{1.382600in}}%
\pgfusepath{clip}%
\pgfsetbuttcap%
\pgfsetmiterjoin%
\definecolor{currentfill}{rgb}{0.121569,0.466667,0.705882}%
\pgfsetfillcolor{currentfill}%
\pgfsetlinewidth{0.000000pt}%
\definecolor{currentstroke}{rgb}{0.000000,0.000000,0.000000}%
\pgfsetstrokecolor{currentstroke}%
\pgfsetstrokeopacity{0.000000}%
\pgfsetdash{}{0pt}%
\pgfpathmoveto{\pgfqpoint{1.932028in}{0.559404in}}%
\pgfpathlineto{\pgfqpoint{2.299260in}{0.559404in}}%
\pgfpathlineto{\pgfqpoint{2.299260in}{1.801560in}}%
\pgfpathlineto{\pgfqpoint{1.932028in}{1.801560in}}%
\pgfpathlineto{\pgfqpoint{1.932028in}{0.559404in}}%
\pgfpathclose%
\pgfusepath{fill}%
\end{pgfscope}%
\begin{pgfscope}%
\pgfpathrectangle{\pgfqpoint{0.444740in}{0.559404in}}{\pgfqpoint{2.423729in}{1.382600in}}%
\pgfusepath{clip}%
\pgfsetbuttcap%
\pgfsetmiterjoin%
\definecolor{currentfill}{rgb}{0.121569,0.466667,0.705882}%
\pgfsetfillcolor{currentfill}%
\pgfsetlinewidth{0.000000pt}%
\definecolor{currentstroke}{rgb}{0.000000,0.000000,0.000000}%
\pgfsetstrokecolor{currentstroke}%
\pgfsetstrokeopacity{0.000000}%
\pgfsetdash{}{0pt}%
\pgfpathmoveto{\pgfqpoint{2.391068in}{0.559404in}}%
\pgfpathlineto{\pgfqpoint{2.758299in}{0.559404in}}%
\pgfpathlineto{\pgfqpoint{2.758299in}{1.767872in}}%
\pgfpathlineto{\pgfqpoint{2.391068in}{1.767872in}}%
\pgfpathlineto{\pgfqpoint{2.391068in}{0.559404in}}%
\pgfpathclose%
\pgfusepath{fill}%
\end{pgfscope}%
\begin{pgfscope}%
\pgfsetbuttcap%
\pgfsetroundjoin%
\definecolor{currentfill}{rgb}{0.000000,0.000000,0.000000}%
\pgfsetfillcolor{currentfill}%
\pgfsetlinewidth{0.803000pt}%
\definecolor{currentstroke}{rgb}{0.000000,0.000000,0.000000}%
\pgfsetstrokecolor{currentstroke}%
\pgfsetdash{}{0pt}%
\pgfsys@defobject{currentmarker}{\pgfqpoint{0.000000in}{-0.048611in}}{\pgfqpoint{0.000000in}{0.000000in}}{%
\pgfpathmoveto{\pgfqpoint{0.000000in}{0.000000in}}%
\pgfpathlineto{\pgfqpoint{0.000000in}{-0.048611in}}%
\pgfusepath{stroke,fill}%
}%
\begin{pgfscope}%
\pgfsys@transformshift{0.738525in}{0.559404in}%
\pgfsys@useobject{currentmarker}{}%
\end{pgfscope}%
\end{pgfscope}%
\begin{pgfscope}%
\definecolor{textcolor}{rgb}{0.000000,0.000000,0.000000}%
\pgfsetstrokecolor{textcolor}%
\pgfsetfillcolor{textcolor}%
\pgftext[x=0.388491in, y=0.199504in, left, base,rotate=30.000000]{\color{textcolor}{\sffamily\fontsize{8.000000}{9.600000}\selectfont\catcode`\^=\active\def^{\ifmmode\sp\else\^{}\fi}\catcode`\%=\active\def
\end{pgfscope}%
\begin{pgfscope}%
\pgfsetbuttcap%
\pgfsetroundjoin%
\definecolor{currentfill}{rgb}{0.000000,0.000000,0.000000}%
\pgfsetfillcolor{currentfill}%
\pgfsetlinewidth{0.803000pt}%
\definecolor{currentstroke}{rgb}{0.000000,0.000000,0.000000}%
\pgfsetstrokecolor{currentstroke}%
\pgfsetdash{}{0pt}%
\pgfsys@defobject{currentmarker}{\pgfqpoint{0.000000in}{-0.048611in}}{\pgfqpoint{0.000000in}{0.000000in}}{%
\pgfpathmoveto{\pgfqpoint{0.000000in}{0.000000in}}%
\pgfpathlineto{\pgfqpoint{0.000000in}{-0.048611in}}%
\pgfusepath{stroke,fill}%
}%
\begin{pgfscope}%
\pgfsys@transformshift{1.197565in}{0.559404in}%
\pgfsys@useobject{currentmarker}{}%
\end{pgfscope}%
\end{pgfscope}%
\begin{pgfscope}%
\definecolor{textcolor}{rgb}{0.000000,0.000000,0.000000}%
\pgfsetstrokecolor{textcolor}%
\pgfsetfillcolor{textcolor}%
\pgftext[x=0.739376in, y=0.137061in, left, base,rotate=30.000000]{\color{textcolor}{\sffamily\fontsize{8.000000}{9.600000}\selectfont\catcode`\^=\active\def^{\ifmmode\sp\else\^{}\fi}\catcode`\%=\active\def
\end{pgfscope}%
\begin{pgfscope}%
\pgfsetbuttcap%
\pgfsetroundjoin%
\definecolor{currentfill}{rgb}{0.000000,0.000000,0.000000}%
\pgfsetfillcolor{currentfill}%
\pgfsetlinewidth{0.803000pt}%
\definecolor{currentstroke}{rgb}{0.000000,0.000000,0.000000}%
\pgfsetstrokecolor{currentstroke}%
\pgfsetdash{}{0pt}%
\pgfsys@defobject{currentmarker}{\pgfqpoint{0.000000in}{-0.048611in}}{\pgfqpoint{0.000000in}{0.000000in}}{%
\pgfpathmoveto{\pgfqpoint{0.000000in}{0.000000in}}%
\pgfpathlineto{\pgfqpoint{0.000000in}{-0.048611in}}%
\pgfusepath{stroke,fill}%
}%
\begin{pgfscope}%
\pgfsys@transformshift{1.656604in}{0.559404in}%
\pgfsys@useobject{currentmarker}{}%
\end{pgfscope}%
\end{pgfscope}%
\begin{pgfscope}%
\definecolor{textcolor}{rgb}{0.000000,0.000000,0.000000}%
\pgfsetstrokecolor{textcolor}%
\pgfsetfillcolor{textcolor}%
\pgftext[x=1.423111in, y=0.266788in, left, base,rotate=30.000000]{\color{textcolor}{\sffamily\fontsize{8.000000}{9.600000}\selectfont\catcode`\^=\active\def^{\ifmmode\sp\else\^{}\fi}\catcode`\%=\active\def
\end{pgfscope}%
\begin{pgfscope}%
\pgfsetbuttcap%
\pgfsetroundjoin%
\definecolor{currentfill}{rgb}{0.000000,0.000000,0.000000}%
\pgfsetfillcolor{currentfill}%
\pgfsetlinewidth{0.803000pt}%
\definecolor{currentstroke}{rgb}{0.000000,0.000000,0.000000}%
\pgfsetstrokecolor{currentstroke}%
\pgfsetdash{}{0pt}%
\pgfsys@defobject{currentmarker}{\pgfqpoint{0.000000in}{-0.048611in}}{\pgfqpoint{0.000000in}{0.000000in}}{%
\pgfpathmoveto{\pgfqpoint{0.000000in}{0.000000in}}%
\pgfpathlineto{\pgfqpoint{0.000000in}{-0.048611in}}%
\pgfusepath{stroke,fill}%
}%
\begin{pgfscope}%
\pgfsys@transformshift{2.115644in}{0.559404in}%
\pgfsys@useobject{currentmarker}{}%
\end{pgfscope}%
\end{pgfscope}%
\begin{pgfscope}%
\definecolor{textcolor}{rgb}{0.000000,0.000000,0.000000}%
\pgfsetstrokecolor{textcolor}%
\pgfsetfillcolor{textcolor}%
\pgftext[x=1.504429in, y=0.048710in, left, base,rotate=30.000000]{\color{textcolor}{\sffamily\fontsize{8.000000}{9.600000}\selectfont\catcode`\^=\active\def^{\ifmmode\sp\else\^{}\fi}\catcode`\%=\active\def
\end{pgfscope}%
\begin{pgfscope}%
\pgfsetbuttcap%
\pgfsetroundjoin%
\definecolor{currentfill}{rgb}{0.000000,0.000000,0.000000}%
\pgfsetfillcolor{currentfill}%
\pgfsetlinewidth{0.803000pt}%
\definecolor{currentstroke}{rgb}{0.000000,0.000000,0.000000}%
\pgfsetstrokecolor{currentstroke}%
\pgfsetdash{}{0pt}%
\pgfsys@defobject{currentmarker}{\pgfqpoint{0.000000in}{-0.048611in}}{\pgfqpoint{0.000000in}{0.000000in}}{%
\pgfpathmoveto{\pgfqpoint{0.000000in}{0.000000in}}%
\pgfpathlineto{\pgfqpoint{0.000000in}{-0.048611in}}%
\pgfusepath{stroke,fill}%
}%
\begin{pgfscope}%
\pgfsys@transformshift{2.574683in}{0.559404in}%
\pgfsys@useobject{currentmarker}{}%
\end{pgfscope}%
\end{pgfscope}%
\begin{pgfscope}%
\definecolor{textcolor}{rgb}{0.000000,0.000000,0.000000}%
\pgfsetstrokecolor{textcolor}%
\pgfsetfillcolor{textcolor}%
\pgftext[x=2.241957in, y=0.209496in, left, base,rotate=30.000000]{\color{textcolor}{\sffamily\fontsize{8.000000}{9.600000}\selectfont\catcode`\^=\active\def^{\ifmmode\sp\else\^{}\fi}\catcode`\%=\active\def
\end{pgfscope}%
\begin{pgfscope}%
\pgfpathrectangle{\pgfqpoint{0.444740in}{0.559404in}}{\pgfqpoint{2.423729in}{1.382600in}}%
\pgfusepath{clip}%
\pgfsetrectcap%
\pgfsetroundjoin%
\pgfsetlinewidth{0.401500pt}%
\definecolor{currentstroke}{rgb}{0.690196,0.690196,0.690196}%
\pgfsetstrokecolor{currentstroke}%
\pgfsetstrokeopacity{0.350000}%
\pgfsetdash{}{0pt}%
\pgfpathmoveto{\pgfqpoint{0.444740in}{0.559404in}}%
\pgfpathlineto{\pgfqpoint{2.868469in}{0.559404in}}%
\pgfusepath{stroke}%
\end{pgfscope}%
\begin{pgfscope}%
\pgfsetbuttcap%
\pgfsetroundjoin%
\definecolor{currentfill}{rgb}{0.000000,0.000000,0.000000}%
\pgfsetfillcolor{currentfill}%
\pgfsetlinewidth{0.803000pt}%
\definecolor{currentstroke}{rgb}{0.000000,0.000000,0.000000}%
\pgfsetstrokecolor{currentstroke}%
\pgfsetdash{}{0pt}%
\pgfsys@defobject{currentmarker}{\pgfqpoint{-0.048611in}{0.000000in}}{\pgfqpoint{-0.000000in}{0.000000in}}{%
\pgfpathmoveto{\pgfqpoint{-0.000000in}{0.000000in}}%
\pgfpathlineto{\pgfqpoint{-0.048611in}{0.000000in}}%
\pgfusepath{stroke,fill}%
}%
\begin{pgfscope}%
\pgfsys@transformshift{0.444740in}{0.559404in}%
\pgfsys@useobject{currentmarker}{}%
\end{pgfscope}%
\end{pgfscope}%
\begin{pgfscope}%
\definecolor{textcolor}{rgb}{0.000000,0.000000,0.000000}%
\pgfsetstrokecolor{textcolor}%
\pgfsetfillcolor{textcolor}%
\pgftext[x=0.196667in, y=0.520824in, left, base]{\color{textcolor}{\sffamily\fontsize{8.000000}{9.600000}\selectfont\catcode`\^=\active\def^{\ifmmode\sp\else\^{}\fi}\catcode`\%=\active\def
\end{pgfscope}%
\begin{pgfscope}%
\pgfpathrectangle{\pgfqpoint{0.444740in}{0.559404in}}{\pgfqpoint{2.423729in}{1.382600in}}%
\pgfusepath{clip}%
\pgfsetrectcap%
\pgfsetroundjoin%
\pgfsetlinewidth{0.401500pt}%
\definecolor{currentstroke}{rgb}{0.690196,0.690196,0.690196}%
\pgfsetstrokecolor{currentstroke}%
\pgfsetstrokeopacity{0.350000}%
\pgfsetdash{}{0pt}%
\pgfpathmoveto{\pgfqpoint{0.444740in}{1.007202in}}%
\pgfpathlineto{\pgfqpoint{2.868469in}{1.007202in}}%
\pgfusepath{stroke}%
\end{pgfscope}%
\begin{pgfscope}%
\pgfsetbuttcap%
\pgfsetroundjoin%
\definecolor{currentfill}{rgb}{0.000000,0.000000,0.000000}%
\pgfsetfillcolor{currentfill}%
\pgfsetlinewidth{0.803000pt}%
\definecolor{currentstroke}{rgb}{0.000000,0.000000,0.000000}%
\pgfsetstrokecolor{currentstroke}%
\pgfsetdash{}{0pt}%
\pgfsys@defobject{currentmarker}{\pgfqpoint{-0.048611in}{0.000000in}}{\pgfqpoint{-0.000000in}{0.000000in}}{%
\pgfpathmoveto{\pgfqpoint{-0.000000in}{0.000000in}}%
\pgfpathlineto{\pgfqpoint{-0.048611in}{0.000000in}}%
\pgfusepath{stroke,fill}%
}%
\begin{pgfscope}%
\pgfsys@transformshift{0.444740in}{1.007202in}%
\pgfsys@useobject{currentmarker}{}%
\end{pgfscope}%
\end{pgfscope}%
\begin{pgfscope}%
\definecolor{textcolor}{rgb}{0.000000,0.000000,0.000000}%
\pgfsetstrokecolor{textcolor}%
\pgfsetfillcolor{textcolor}%
\pgftext[x=0.196667in, y=0.968622in, left, base]{\color{textcolor}{\sffamily\fontsize{8.000000}{9.600000}\selectfont\catcode`\^=\active\def^{\ifmmode\sp\else\^{}\fi}\catcode`\%=\active\def
\end{pgfscope}%
\begin{pgfscope}%
\pgfpathrectangle{\pgfqpoint{0.444740in}{0.559404in}}{\pgfqpoint{2.423729in}{1.382600in}}%
\pgfusepath{clip}%
\pgfsetrectcap%
\pgfsetroundjoin%
\pgfsetlinewidth{0.401500pt}%
\definecolor{currentstroke}{rgb}{0.690196,0.690196,0.690196}%
\pgfsetstrokecolor{currentstroke}%
\pgfsetstrokeopacity{0.350000}%
\pgfsetdash{}{0pt}%
\pgfpathmoveto{\pgfqpoint{0.444740in}{1.455001in}}%
\pgfpathlineto{\pgfqpoint{2.868469in}{1.455001in}}%
\pgfusepath{stroke}%
\end{pgfscope}%
\begin{pgfscope}%
\pgfsetbuttcap%
\pgfsetroundjoin%
\definecolor{currentfill}{rgb}{0.000000,0.000000,0.000000}%
\pgfsetfillcolor{currentfill}%
\pgfsetlinewidth{0.803000pt}%
\definecolor{currentstroke}{rgb}{0.000000,0.000000,0.000000}%
\pgfsetstrokecolor{currentstroke}%
\pgfsetdash{}{0pt}%
\pgfsys@defobject{currentmarker}{\pgfqpoint{-0.048611in}{0.000000in}}{\pgfqpoint{-0.000000in}{0.000000in}}{%
\pgfpathmoveto{\pgfqpoint{-0.000000in}{0.000000in}}%
\pgfpathlineto{\pgfqpoint{-0.048611in}{0.000000in}}%
\pgfusepath{stroke,fill}%
}%
\begin{pgfscope}%
\pgfsys@transformshift{0.444740in}{1.455001in}%
\pgfsys@useobject{currentmarker}{}%
\end{pgfscope}%
\end{pgfscope}%
\begin{pgfscope}%
\definecolor{textcolor}{rgb}{0.000000,0.000000,0.000000}%
\pgfsetstrokecolor{textcolor}%
\pgfsetfillcolor{textcolor}%
\pgftext[x=0.196667in, y=1.416420in, left, base]{\color{textcolor}{\sffamily\fontsize{8.000000}{9.600000}\selectfont\catcode`\^=\active\def^{\ifmmode\sp\else\^{}\fi}\catcode`\%=\active\def
\end{pgfscope}%
\begin{pgfscope}%
\pgfpathrectangle{\pgfqpoint{0.444740in}{0.559404in}}{\pgfqpoint{2.423729in}{1.382600in}}%
\pgfusepath{clip}%
\pgfsetrectcap%
\pgfsetroundjoin%
\pgfsetlinewidth{0.401500pt}%
\definecolor{currentstroke}{rgb}{0.690196,0.690196,0.690196}%
\pgfsetstrokecolor{currentstroke}%
\pgfsetstrokeopacity{0.350000}%
\pgfsetdash{}{0pt}%
\pgfpathmoveto{\pgfqpoint{0.444740in}{1.902799in}}%
\pgfpathlineto{\pgfqpoint{2.868469in}{1.902799in}}%
\pgfusepath{stroke}%
\end{pgfscope}%
\begin{pgfscope}%
\pgfsetbuttcap%
\pgfsetroundjoin%
\definecolor{currentfill}{rgb}{0.000000,0.000000,0.000000}%
\pgfsetfillcolor{currentfill}%
\pgfsetlinewidth{0.803000pt}%
\definecolor{currentstroke}{rgb}{0.000000,0.000000,0.000000}%
\pgfsetstrokecolor{currentstroke}%
\pgfsetdash{}{0pt}%
\pgfsys@defobject{currentmarker}{\pgfqpoint{-0.048611in}{0.000000in}}{\pgfqpoint{-0.000000in}{0.000000in}}{%
\pgfpathmoveto{\pgfqpoint{-0.000000in}{0.000000in}}%
\pgfpathlineto{\pgfqpoint{-0.048611in}{0.000000in}}%
\pgfusepath{stroke,fill}%
}%
\begin{pgfscope}%
\pgfsys@transformshift{0.444740in}{1.902799in}%
\pgfsys@useobject{currentmarker}{}%
\end{pgfscope}%
\end{pgfscope}%
\begin{pgfscope}%
\definecolor{textcolor}{rgb}{0.000000,0.000000,0.000000}%
\pgfsetstrokecolor{textcolor}%
\pgfsetfillcolor{textcolor}%
\pgftext[x=0.196667in, y=1.864219in, left, base]{\color{textcolor}{\sffamily\fontsize{8.000000}{9.600000}\selectfont\catcode`\^=\active\def^{\ifmmode\sp\else\^{}\fi}\catcode`\%=\active\def
\end{pgfscope}%
\begin{pgfscope}%
\definecolor{textcolor}{rgb}{0.000000,0.000000,0.000000}%
\pgfsetstrokecolor{textcolor}%
\pgfsetfillcolor{textcolor}%
\pgftext[x=0.141111in,y=1.250704in,,bottom,rotate=90.000000]{\color{textcolor}{\sffamily\fontsize{9.000000}{10.800000}\selectfont\catcode`\^=\active\def^{\ifmmode\sp\else\^{}\fi}\catcode`\%=\active\def
\end{pgfscope}%
\begin{pgfscope}%
\pgfpathrectangle{\pgfqpoint{0.444740in}{0.559404in}}{\pgfqpoint{2.423729in}{1.382600in}}%
\pgfusepath{clip}%
\pgfsetbuttcap%
\pgfsetroundjoin%
\pgfsetlinewidth{1.405250pt}%
\definecolor{currentstroke}{rgb}{0.000000,0.000000,0.000000}%
\pgfsetstrokecolor{currentstroke}%
\pgfsetdash{}{0pt}%
\pgfpathmoveto{\pgfqpoint{0.738525in}{1.719944in}}%
\pgfpathlineto{\pgfqpoint{0.738525in}{1.871083in}}%
\pgfusepath{stroke}%
\end{pgfscope}%
\begin{pgfscope}%
\pgfpathrectangle{\pgfqpoint{0.444740in}{0.559404in}}{\pgfqpoint{2.423729in}{1.382600in}}%
\pgfusepath{clip}%
\pgfsetbuttcap%
\pgfsetroundjoin%
\pgfsetlinewidth{1.405250pt}%
\definecolor{currentstroke}{rgb}{0.000000,0.000000,0.000000}%
\pgfsetstrokecolor{currentstroke}%
\pgfsetdash{}{0pt}%
\pgfpathmoveto{\pgfqpoint{1.197565in}{1.715314in}}%
\pgfpathlineto{\pgfqpoint{1.197565in}{1.865348in}}%
\pgfusepath{stroke}%
\end{pgfscope}%
\begin{pgfscope}%
\pgfpathrectangle{\pgfqpoint{0.444740in}{0.559404in}}{\pgfqpoint{2.423729in}{1.382600in}}%
\pgfusepath{clip}%
\pgfsetbuttcap%
\pgfsetroundjoin%
\pgfsetlinewidth{1.405250pt}%
\definecolor{currentstroke}{rgb}{0.000000,0.000000,0.000000}%
\pgfsetstrokecolor{currentstroke}%
\pgfsetdash{}{0pt}%
\pgfpathmoveto{\pgfqpoint{1.656604in}{1.697209in}}%
\pgfpathlineto{\pgfqpoint{1.656604in}{1.838535in}}%
\pgfusepath{stroke}%
\end{pgfscope}%
\begin{pgfscope}%
\pgfpathrectangle{\pgfqpoint{0.444740in}{0.559404in}}{\pgfqpoint{2.423729in}{1.382600in}}%
\pgfusepath{clip}%
\pgfsetbuttcap%
\pgfsetroundjoin%
\pgfsetlinewidth{1.405250pt}%
\definecolor{currentstroke}{rgb}{0.000000,0.000000,0.000000}%
\pgfsetstrokecolor{currentstroke}%
\pgfsetdash{}{0pt}%
\pgfpathmoveto{\pgfqpoint{2.115644in}{1.726955in}}%
\pgfpathlineto{\pgfqpoint{2.115644in}{1.876166in}}%
\pgfusepath{stroke}%
\end{pgfscope}%
\begin{pgfscope}%
\pgfpathrectangle{\pgfqpoint{0.444740in}{0.559404in}}{\pgfqpoint{2.423729in}{1.382600in}}%
\pgfusepath{clip}%
\pgfsetbuttcap%
\pgfsetroundjoin%
\pgfsetlinewidth{1.405250pt}%
\definecolor{currentstroke}{rgb}{0.000000,0.000000,0.000000}%
\pgfsetstrokecolor{currentstroke}%
\pgfsetdash{}{0pt}%
\pgfpathmoveto{\pgfqpoint{2.574683in}{1.710162in}}%
\pgfpathlineto{\pgfqpoint{2.574683in}{1.825582in}}%
\pgfusepath{stroke}%
\end{pgfscope}%
\begin{pgfscope}%
\pgfpathrectangle{\pgfqpoint{0.444740in}{0.559404in}}{\pgfqpoint{2.423729in}{1.382600in}}%
\pgfusepath{clip}%
\pgfsetbuttcap%
\pgfsetroundjoin%
\definecolor{currentfill}{rgb}{0.121569,0.466667,0.705882}%
\pgfsetfillcolor{currentfill}%
\pgfsetlinewidth{1.003750pt}%
\definecolor{currentstroke}{rgb}{0.000000,0.000000,0.000000}%
\pgfsetstrokecolor{currentstroke}%
\pgfsetdash{}{0pt}%
\pgfsys@defobject{currentmarker}{\pgfqpoint{-0.027778in}{-0.000000in}}{\pgfqpoint{0.027778in}{0.000000in}}{%
\pgfpathmoveto{\pgfqpoint{0.027778in}{-0.000000in}}%
\pgfpathlineto{\pgfqpoint{-0.027778in}{0.000000in}}%
\pgfusepath{stroke,fill}%
}%
\begin{pgfscope}%
\pgfsys@transformshift{0.738525in}{1.719944in}%
\pgfsys@useobject{currentmarker}{}%
\end{pgfscope}%
\begin{pgfscope}%
\pgfsys@transformshift{1.197565in}{1.715314in}%
\pgfsys@useobject{currentmarker}{}%
\end{pgfscope}%
\begin{pgfscope}%
\pgfsys@transformshift{1.656604in}{1.697209in}%
\pgfsys@useobject{currentmarker}{}%
\end{pgfscope}%
\begin{pgfscope}%
\pgfsys@transformshift{2.115644in}{1.726955in}%
\pgfsys@useobject{currentmarker}{}%
\end{pgfscope}%
\begin{pgfscope}%
\pgfsys@transformshift{2.574683in}{1.710162in}%
\pgfsys@useobject{currentmarker}{}%
\end{pgfscope}%
\end{pgfscope}%
\begin{pgfscope}%
\pgfpathrectangle{\pgfqpoint{0.444740in}{0.559404in}}{\pgfqpoint{2.423729in}{1.382600in}}%
\pgfusepath{clip}%
\pgfsetbuttcap%
\pgfsetroundjoin%
\definecolor{currentfill}{rgb}{0.121569,0.466667,0.705882}%
\pgfsetfillcolor{currentfill}%
\pgfsetlinewidth{1.003750pt}%
\definecolor{currentstroke}{rgb}{0.000000,0.000000,0.000000}%
\pgfsetstrokecolor{currentstroke}%
\pgfsetdash{}{0pt}%
\pgfsys@defobject{currentmarker}{\pgfqpoint{-0.027778in}{-0.000000in}}{\pgfqpoint{0.027778in}{0.000000in}}{%
\pgfpathmoveto{\pgfqpoint{0.027778in}{-0.000000in}}%
\pgfpathlineto{\pgfqpoint{-0.027778in}{0.000000in}}%
\pgfusepath{stroke,fill}%
}%
\begin{pgfscope}%
\pgfsys@transformshift{0.738525in}{1.871083in}%
\pgfsys@useobject{currentmarker}{}%
\end{pgfscope}%
\begin{pgfscope}%
\pgfsys@transformshift{1.197565in}{1.865348in}%
\pgfsys@useobject{currentmarker}{}%
\end{pgfscope}%
\begin{pgfscope}%
\pgfsys@transformshift{1.656604in}{1.838535in}%
\pgfsys@useobject{currentmarker}{}%
\end{pgfscope}%
\begin{pgfscope}%
\pgfsys@transformshift{2.115644in}{1.876166in}%
\pgfsys@useobject{currentmarker}{}%
\end{pgfscope}%
\begin{pgfscope}%
\pgfsys@transformshift{2.574683in}{1.825582in}%
\pgfsys@useobject{currentmarker}{}%
\end{pgfscope}%
\end{pgfscope}%
\begin{pgfscope}%
\pgfsetrectcap%
\pgfsetmiterjoin%
\pgfsetlinewidth{0.803000pt}%
\definecolor{currentstroke}{rgb}{0.000000,0.000000,0.000000}%
\pgfsetstrokecolor{currentstroke}%
\pgfsetdash{}{0pt}%
\pgfpathmoveto{\pgfqpoint{0.444740in}{0.559404in}}%
\pgfpathlineto{\pgfqpoint{0.444740in}{1.942004in}}%
\pgfusepath{stroke}%
\end{pgfscope}%
\begin{pgfscope}%
\pgfsetrectcap%
\pgfsetmiterjoin%
\pgfsetlinewidth{0.803000pt}%
\definecolor{currentstroke}{rgb}{0.000000,0.000000,0.000000}%
\pgfsetstrokecolor{currentstroke}%
\pgfsetdash{}{0pt}%
\pgfpathmoveto{\pgfqpoint{0.444740in}{0.559404in}}%
\pgfpathlineto{\pgfqpoint{2.868469in}{0.559404in}}%
\pgfusepath{stroke}%
\end{pgfscope}%
\begin{pgfscope}%
\pgfsetbuttcap%
\pgfsetmiterjoin%
\definecolor{currentfill}{rgb}{1.000000,1.000000,1.000000}%
\pgfsetfillcolor{currentfill}%
\pgfsetlinewidth{0.000000pt}%
\definecolor{currentstroke}{rgb}{0.000000,0.000000,0.000000}%
\pgfsetstrokecolor{currentstroke}%
\pgfsetstrokeopacity{0.000000}%
\pgfsetdash{}{0pt}%
\pgfpathmoveto{\pgfqpoint{3.741011in}{0.559404in}}%
\pgfpathlineto{\pgfqpoint{6.164740in}{0.559404in}}%
\pgfpathlineto{\pgfqpoint{6.164740in}{1.942004in}}%
\pgfpathlineto{\pgfqpoint{3.741011in}{1.942004in}}%
\pgfpathlineto{\pgfqpoint{3.741011in}{0.559404in}}%
\pgfpathclose%
\pgfusepath{fill}%
\end{pgfscope}%
\begin{pgfscope}%
\pgfpathrectangle{\pgfqpoint{3.741011in}{0.559404in}}{\pgfqpoint{2.423729in}{1.382600in}}%
\pgfusepath{clip}%
\pgfsetbuttcap%
\pgfsetmiterjoin%
\definecolor{currentfill}{rgb}{0.121569,0.466667,0.705882}%
\pgfsetfillcolor{currentfill}%
\pgfsetlinewidth{0.000000pt}%
\definecolor{currentstroke}{rgb}{0.000000,0.000000,0.000000}%
\pgfsetstrokecolor{currentstroke}%
\pgfsetstrokeopacity{0.000000}%
\pgfsetdash{}{0pt}%
\pgfpathmoveto{\pgfqpoint{3.851181in}{0.559404in}}%
\pgfpathlineto{\pgfqpoint{4.218412in}{0.559404in}}%
\pgfpathlineto{\pgfqpoint{4.218412in}{0.984022in}}%
\pgfpathlineto{\pgfqpoint{3.851181in}{0.984022in}}%
\pgfpathlineto{\pgfqpoint{3.851181in}{0.559404in}}%
\pgfpathclose%
\pgfusepath{fill}%
\end{pgfscope}%
\begin{pgfscope}%
\pgfpathrectangle{\pgfqpoint{3.741011in}{0.559404in}}{\pgfqpoint{2.423729in}{1.382600in}}%
\pgfusepath{clip}%
\pgfsetbuttcap%
\pgfsetmiterjoin%
\definecolor{currentfill}{rgb}{0.121569,0.466667,0.705882}%
\pgfsetfillcolor{currentfill}%
\pgfsetlinewidth{0.000000pt}%
\definecolor{currentstroke}{rgb}{0.000000,0.000000,0.000000}%
\pgfsetstrokecolor{currentstroke}%
\pgfsetstrokeopacity{0.000000}%
\pgfsetdash{}{0pt}%
\pgfpathmoveto{\pgfqpoint{4.310220in}{0.559404in}}%
\pgfpathlineto{\pgfqpoint{4.677452in}{0.559404in}}%
\pgfpathlineto{\pgfqpoint{4.677452in}{0.985981in}}%
\pgfpathlineto{\pgfqpoint{4.310220in}{0.985981in}}%
\pgfpathlineto{\pgfqpoint{4.310220in}{0.559404in}}%
\pgfpathclose%
\pgfusepath{fill}%
\end{pgfscope}%
\begin{pgfscope}%
\pgfpathrectangle{\pgfqpoint{3.741011in}{0.559404in}}{\pgfqpoint{2.423729in}{1.382600in}}%
\pgfusepath{clip}%
\pgfsetbuttcap%
\pgfsetmiterjoin%
\definecolor{currentfill}{rgb}{0.121569,0.466667,0.705882}%
\pgfsetfillcolor{currentfill}%
\pgfsetlinewidth{0.000000pt}%
\definecolor{currentstroke}{rgb}{0.000000,0.000000,0.000000}%
\pgfsetstrokecolor{currentstroke}%
\pgfsetstrokeopacity{0.000000}%
\pgfsetdash{}{0pt}%
\pgfpathmoveto{\pgfqpoint{4.769260in}{0.559404in}}%
\pgfpathlineto{\pgfqpoint{5.136491in}{0.559404in}}%
\pgfpathlineto{\pgfqpoint{5.136491in}{1.054165in}}%
\pgfpathlineto{\pgfqpoint{4.769260in}{1.054165in}}%
\pgfpathlineto{\pgfqpoint{4.769260in}{0.559404in}}%
\pgfpathclose%
\pgfusepath{fill}%
\end{pgfscope}%
\begin{pgfscope}%
\pgfpathrectangle{\pgfqpoint{3.741011in}{0.559404in}}{\pgfqpoint{2.423729in}{1.382600in}}%
\pgfusepath{clip}%
\pgfsetbuttcap%
\pgfsetmiterjoin%
\definecolor{currentfill}{rgb}{0.121569,0.466667,0.705882}%
\pgfsetfillcolor{currentfill}%
\pgfsetlinewidth{0.000000pt}%
\definecolor{currentstroke}{rgb}{0.000000,0.000000,0.000000}%
\pgfsetstrokecolor{currentstroke}%
\pgfsetstrokeopacity{0.000000}%
\pgfsetdash{}{0pt}%
\pgfpathmoveto{\pgfqpoint{5.228299in}{0.559404in}}%
\pgfpathlineto{\pgfqpoint{5.595531in}{0.559404in}}%
\pgfpathlineto{\pgfqpoint{5.595531in}{0.992481in}}%
\pgfpathlineto{\pgfqpoint{5.228299in}{0.992481in}}%
\pgfpathlineto{\pgfqpoint{5.228299in}{0.559404in}}%
\pgfpathclose%
\pgfusepath{fill}%
\end{pgfscope}%
\begin{pgfscope}%
\pgfpathrectangle{\pgfqpoint{3.741011in}{0.559404in}}{\pgfqpoint{2.423729in}{1.382600in}}%
\pgfusepath{clip}%
\pgfsetbuttcap%
\pgfsetmiterjoin%
\definecolor{currentfill}{rgb}{0.121569,0.466667,0.705882}%
\pgfsetfillcolor{currentfill}%
\pgfsetlinewidth{0.000000pt}%
\definecolor{currentstroke}{rgb}{0.000000,0.000000,0.000000}%
\pgfsetstrokecolor{currentstroke}%
\pgfsetstrokeopacity{0.000000}%
\pgfsetdash{}{0pt}%
\pgfpathmoveto{\pgfqpoint{5.687339in}{0.559404in}}%
\pgfpathlineto{\pgfqpoint{6.054570in}{0.559404in}}%
\pgfpathlineto{\pgfqpoint{6.054570in}{1.821190in}}%
\pgfpathlineto{\pgfqpoint{5.687339in}{1.821190in}}%
\pgfpathlineto{\pgfqpoint{5.687339in}{0.559404in}}%
\pgfpathclose%
\pgfusepath{fill}%
\end{pgfscope}%
\begin{pgfscope}%
\pgfsetbuttcap%
\pgfsetroundjoin%
\definecolor{currentfill}{rgb}{0.000000,0.000000,0.000000}%
\pgfsetfillcolor{currentfill}%
\pgfsetlinewidth{0.803000pt}%
\definecolor{currentstroke}{rgb}{0.000000,0.000000,0.000000}%
\pgfsetstrokecolor{currentstroke}%
\pgfsetdash{}{0pt}%
\pgfsys@defobject{currentmarker}{\pgfqpoint{0.000000in}{-0.048611in}}{\pgfqpoint{0.000000in}{0.000000in}}{%
\pgfpathmoveto{\pgfqpoint{0.000000in}{0.000000in}}%
\pgfpathlineto{\pgfqpoint{0.000000in}{-0.048611in}}%
\pgfusepath{stroke,fill}%
}%
\begin{pgfscope}%
\pgfsys@transformshift{4.034796in}{0.559404in}%
\pgfsys@useobject{currentmarker}{}%
\end{pgfscope}%
\end{pgfscope}%
\begin{pgfscope}%
\definecolor{textcolor}{rgb}{0.000000,0.000000,0.000000}%
\pgfsetstrokecolor{textcolor}%
\pgfsetfillcolor{textcolor}%
\pgftext[x=3.684763in, y=0.199504in, left, base,rotate=30.000000]{\color{textcolor}{\sffamily\fontsize{8.000000}{9.600000}\selectfont\catcode`\^=\active\def^{\ifmmode\sp\else\^{}\fi}\catcode`\%=\active\def
\end{pgfscope}%
\begin{pgfscope}%
\pgfsetbuttcap%
\pgfsetroundjoin%
\definecolor{currentfill}{rgb}{0.000000,0.000000,0.000000}%
\pgfsetfillcolor{currentfill}%
\pgfsetlinewidth{0.803000pt}%
\definecolor{currentstroke}{rgb}{0.000000,0.000000,0.000000}%
\pgfsetstrokecolor{currentstroke}%
\pgfsetdash{}{0pt}%
\pgfsys@defobject{currentmarker}{\pgfqpoint{0.000000in}{-0.048611in}}{\pgfqpoint{0.000000in}{0.000000in}}{%
\pgfpathmoveto{\pgfqpoint{0.000000in}{0.000000in}}%
\pgfpathlineto{\pgfqpoint{0.000000in}{-0.048611in}}%
\pgfusepath{stroke,fill}%
}%
\begin{pgfscope}%
\pgfsys@transformshift{4.493836in}{0.559404in}%
\pgfsys@useobject{currentmarker}{}%
\end{pgfscope}%
\end{pgfscope}%
\begin{pgfscope}%
\definecolor{textcolor}{rgb}{0.000000,0.000000,0.000000}%
\pgfsetstrokecolor{textcolor}%
\pgfsetfillcolor{textcolor}%
\pgftext[x=4.035648in, y=0.137061in, left, base,rotate=30.000000]{\color{textcolor}{\sffamily\fontsize{8.000000}{9.600000}\selectfont\catcode`\^=\active\def^{\ifmmode\sp\else\^{}\fi}\catcode`\%=\active\def
\end{pgfscope}%
\begin{pgfscope}%
\pgfsetbuttcap%
\pgfsetroundjoin%
\definecolor{currentfill}{rgb}{0.000000,0.000000,0.000000}%
\pgfsetfillcolor{currentfill}%
\pgfsetlinewidth{0.803000pt}%
\definecolor{currentstroke}{rgb}{0.000000,0.000000,0.000000}%
\pgfsetstrokecolor{currentstroke}%
\pgfsetdash{}{0pt}%
\pgfsys@defobject{currentmarker}{\pgfqpoint{0.000000in}{-0.048611in}}{\pgfqpoint{0.000000in}{0.000000in}}{%
\pgfpathmoveto{\pgfqpoint{0.000000in}{0.000000in}}%
\pgfpathlineto{\pgfqpoint{0.000000in}{-0.048611in}}%
\pgfusepath{stroke,fill}%
}%
\begin{pgfscope}%
\pgfsys@transformshift{4.952875in}{0.559404in}%
\pgfsys@useobject{currentmarker}{}%
\end{pgfscope}%
\end{pgfscope}%
\begin{pgfscope}%
\definecolor{textcolor}{rgb}{0.000000,0.000000,0.000000}%
\pgfsetstrokecolor{textcolor}%
\pgfsetfillcolor{textcolor}%
\pgftext[x=4.719382in, y=0.266788in, left, base,rotate=30.000000]{\color{textcolor}{\sffamily\fontsize{8.000000}{9.600000}\selectfont\catcode`\^=\active\def^{\ifmmode\sp\else\^{}\fi}\catcode`\%=\active\def
\end{pgfscope}%
\begin{pgfscope}%
\pgfsetbuttcap%
\pgfsetroundjoin%
\definecolor{currentfill}{rgb}{0.000000,0.000000,0.000000}%
\pgfsetfillcolor{currentfill}%
\pgfsetlinewidth{0.803000pt}%
\definecolor{currentstroke}{rgb}{0.000000,0.000000,0.000000}%
\pgfsetstrokecolor{currentstroke}%
\pgfsetdash{}{0pt}%
\pgfsys@defobject{currentmarker}{\pgfqpoint{0.000000in}{-0.048611in}}{\pgfqpoint{0.000000in}{0.000000in}}{%
\pgfpathmoveto{\pgfqpoint{0.000000in}{0.000000in}}%
\pgfpathlineto{\pgfqpoint{0.000000in}{-0.048611in}}%
\pgfusepath{stroke,fill}%
}%
\begin{pgfscope}%
\pgfsys@transformshift{5.411915in}{0.559404in}%
\pgfsys@useobject{currentmarker}{}%
\end{pgfscope}%
\end{pgfscope}%
\begin{pgfscope}%
\definecolor{textcolor}{rgb}{0.000000,0.000000,0.000000}%
\pgfsetstrokecolor{textcolor}%
\pgfsetfillcolor{textcolor}%
\pgftext[x=4.800700in, y=0.048710in, left, base,rotate=30.000000]{\color{textcolor}{\sffamily\fontsize{8.000000}{9.600000}\selectfont\catcode`\^=\active\def^{\ifmmode\sp\else\^{}\fi}\catcode`\%=\active\def
\end{pgfscope}%
\begin{pgfscope}%
\pgfsetbuttcap%
\pgfsetroundjoin%
\definecolor{currentfill}{rgb}{0.000000,0.000000,0.000000}%
\pgfsetfillcolor{currentfill}%
\pgfsetlinewidth{0.803000pt}%
\definecolor{currentstroke}{rgb}{0.000000,0.000000,0.000000}%
\pgfsetstrokecolor{currentstroke}%
\pgfsetdash{}{0pt}%
\pgfsys@defobject{currentmarker}{\pgfqpoint{0.000000in}{-0.048611in}}{\pgfqpoint{0.000000in}{0.000000in}}{%
\pgfpathmoveto{\pgfqpoint{0.000000in}{0.000000in}}%
\pgfpathlineto{\pgfqpoint{0.000000in}{-0.048611in}}%
\pgfusepath{stroke,fill}%
}%
\begin{pgfscope}%
\pgfsys@transformshift{5.870955in}{0.559404in}%
\pgfsys@useobject{currentmarker}{}%
\end{pgfscope}%
\end{pgfscope}%
\begin{pgfscope}%
\definecolor{textcolor}{rgb}{0.000000,0.000000,0.000000}%
\pgfsetstrokecolor{textcolor}%
\pgfsetfillcolor{textcolor}%
\pgftext[x=5.538228in, y=0.209496in, left, base,rotate=30.000000]{\color{textcolor}{\sffamily\fontsize{8.000000}{9.600000}\selectfont\catcode`\^=\active\def^{\ifmmode\sp\else\^{}\fi}\catcode`\%=\active\def
\end{pgfscope}%
\begin{pgfscope}%
\pgfpathrectangle{\pgfqpoint{3.741011in}{0.559404in}}{\pgfqpoint{2.423729in}{1.382600in}}%
\pgfusepath{clip}%
\pgfsetrectcap%
\pgfsetroundjoin%
\pgfsetlinewidth{0.401500pt}%
\definecolor{currentstroke}{rgb}{0.690196,0.690196,0.690196}%
\pgfsetstrokecolor{currentstroke}%
\pgfsetstrokeopacity{0.350000}%
\pgfsetdash{}{0pt}%
\pgfpathmoveto{\pgfqpoint{3.741011in}{0.559404in}}%
\pgfpathlineto{\pgfqpoint{6.164740in}{0.559404in}}%
\pgfusepath{stroke}%
\end{pgfscope}%
\begin{pgfscope}%
\pgfsetbuttcap%
\pgfsetroundjoin%
\definecolor{currentfill}{rgb}{0.000000,0.000000,0.000000}%
\pgfsetfillcolor{currentfill}%
\pgfsetlinewidth{0.803000pt}%
\definecolor{currentstroke}{rgb}{0.000000,0.000000,0.000000}%
\pgfsetstrokecolor{currentstroke}%
\pgfsetdash{}{0pt}%
\pgfsys@defobject{currentmarker}{\pgfqpoint{-0.048611in}{0.000000in}}{\pgfqpoint{-0.000000in}{0.000000in}}{%
\pgfpathmoveto{\pgfqpoint{-0.000000in}{0.000000in}}%
\pgfpathlineto{\pgfqpoint{-0.048611in}{0.000000in}}%
\pgfusepath{stroke,fill}%
}%
\begin{pgfscope}%
\pgfsys@transformshift{3.741011in}{0.559404in}%
\pgfsys@useobject{currentmarker}{}%
\end{pgfscope}%
\end{pgfscope}%
\begin{pgfscope}%
\definecolor{textcolor}{rgb}{0.000000,0.000000,0.000000}%
\pgfsetstrokecolor{textcolor}%
\pgfsetfillcolor{textcolor}%
\pgftext[x=3.492938in, y=0.520824in, left, base]{\color{textcolor}{\sffamily\fontsize{8.000000}{9.600000}\selectfont\catcode`\^=\active\def^{\ifmmode\sp\else\^{}\fi}\catcode`\%=\active\def
\end{pgfscope}%
\begin{pgfscope}%
\pgfpathrectangle{\pgfqpoint{3.741011in}{0.559404in}}{\pgfqpoint{2.423729in}{1.382600in}}%
\pgfusepath{clip}%
\pgfsetrectcap%
\pgfsetroundjoin%
\pgfsetlinewidth{0.401500pt}%
\definecolor{currentstroke}{rgb}{0.690196,0.690196,0.690196}%
\pgfsetstrokecolor{currentstroke}%
\pgfsetstrokeopacity{0.350000}%
\pgfsetdash{}{0pt}%
\pgfpathmoveto{\pgfqpoint{3.741011in}{0.912817in}}%
\pgfpathlineto{\pgfqpoint{6.164740in}{0.912817in}}%
\pgfusepath{stroke}%
\end{pgfscope}%
\begin{pgfscope}%
\pgfsetbuttcap%
\pgfsetroundjoin%
\definecolor{currentfill}{rgb}{0.000000,0.000000,0.000000}%
\pgfsetfillcolor{currentfill}%
\pgfsetlinewidth{0.803000pt}%
\definecolor{currentstroke}{rgb}{0.000000,0.000000,0.000000}%
\pgfsetstrokecolor{currentstroke}%
\pgfsetdash{}{0pt}%
\pgfsys@defobject{currentmarker}{\pgfqpoint{-0.048611in}{0.000000in}}{\pgfqpoint{-0.000000in}{0.000000in}}{%
\pgfpathmoveto{\pgfqpoint{-0.000000in}{0.000000in}}%
\pgfpathlineto{\pgfqpoint{-0.048611in}{0.000000in}}%
\pgfusepath{stroke,fill}%
}%
\begin{pgfscope}%
\pgfsys@transformshift{3.741011in}{0.912817in}%
\pgfsys@useobject{currentmarker}{}%
\end{pgfscope}%
\end{pgfscope}%
\begin{pgfscope}%
\definecolor{textcolor}{rgb}{0.000000,0.000000,0.000000}%
\pgfsetstrokecolor{textcolor}%
\pgfsetfillcolor{textcolor}%
\pgftext[x=3.492938in, y=0.874237in, left, base]{\color{textcolor}{\sffamily\fontsize{8.000000}{9.600000}\selectfont\catcode`\^=\active\def^{\ifmmode\sp\else\^{}\fi}\catcode`\%=\active\def
\end{pgfscope}%
\begin{pgfscope}%
\pgfpathrectangle{\pgfqpoint{3.741011in}{0.559404in}}{\pgfqpoint{2.423729in}{1.382600in}}%
\pgfusepath{clip}%
\pgfsetrectcap%
\pgfsetroundjoin%
\pgfsetlinewidth{0.401500pt}%
\definecolor{currentstroke}{rgb}{0.690196,0.690196,0.690196}%
\pgfsetstrokecolor{currentstroke}%
\pgfsetstrokeopacity{0.350000}%
\pgfsetdash{}{0pt}%
\pgfpathmoveto{\pgfqpoint{3.741011in}{1.266230in}}%
\pgfpathlineto{\pgfqpoint{6.164740in}{1.266230in}}%
\pgfusepath{stroke}%
\end{pgfscope}%
\begin{pgfscope}%
\pgfsetbuttcap%
\pgfsetroundjoin%
\definecolor{currentfill}{rgb}{0.000000,0.000000,0.000000}%
\pgfsetfillcolor{currentfill}%
\pgfsetlinewidth{0.803000pt}%
\definecolor{currentstroke}{rgb}{0.000000,0.000000,0.000000}%
\pgfsetstrokecolor{currentstroke}%
\pgfsetdash{}{0pt}%
\pgfsys@defobject{currentmarker}{\pgfqpoint{-0.048611in}{0.000000in}}{\pgfqpoint{-0.000000in}{0.000000in}}{%
\pgfpathmoveto{\pgfqpoint{-0.000000in}{0.000000in}}%
\pgfpathlineto{\pgfqpoint{-0.048611in}{0.000000in}}%
\pgfusepath{stroke,fill}%
}%
\begin{pgfscope}%
\pgfsys@transformshift{3.741011in}{1.266230in}%
\pgfsys@useobject{currentmarker}{}%
\end{pgfscope}%
\end{pgfscope}%
\begin{pgfscope}%
\definecolor{textcolor}{rgb}{0.000000,0.000000,0.000000}%
\pgfsetstrokecolor{textcolor}%
\pgfsetfillcolor{textcolor}%
\pgftext[x=3.492938in, y=1.227650in, left, base]{\color{textcolor}{\sffamily\fontsize{8.000000}{9.600000}\selectfont\catcode`\^=\active\def^{\ifmmode\sp\else\^{}\fi}\catcode`\%=\active\def
\end{pgfscope}%
\begin{pgfscope}%
\pgfpathrectangle{\pgfqpoint{3.741011in}{0.559404in}}{\pgfqpoint{2.423729in}{1.382600in}}%
\pgfusepath{clip}%
\pgfsetrectcap%
\pgfsetroundjoin%
\pgfsetlinewidth{0.401500pt}%
\definecolor{currentstroke}{rgb}{0.690196,0.690196,0.690196}%
\pgfsetstrokecolor{currentstroke}%
\pgfsetstrokeopacity{0.350000}%
\pgfsetdash{}{0pt}%
\pgfpathmoveto{\pgfqpoint{3.741011in}{1.619643in}}%
\pgfpathlineto{\pgfqpoint{6.164740in}{1.619643in}}%
\pgfusepath{stroke}%
\end{pgfscope}%
\begin{pgfscope}%
\pgfsetbuttcap%
\pgfsetroundjoin%
\definecolor{currentfill}{rgb}{0.000000,0.000000,0.000000}%
\pgfsetfillcolor{currentfill}%
\pgfsetlinewidth{0.803000pt}%
\definecolor{currentstroke}{rgb}{0.000000,0.000000,0.000000}%
\pgfsetstrokecolor{currentstroke}%
\pgfsetdash{}{0pt}%
\pgfsys@defobject{currentmarker}{\pgfqpoint{-0.048611in}{0.000000in}}{\pgfqpoint{-0.000000in}{0.000000in}}{%
\pgfpathmoveto{\pgfqpoint{-0.000000in}{0.000000in}}%
\pgfpathlineto{\pgfqpoint{-0.048611in}{0.000000in}}%
\pgfusepath{stroke,fill}%
}%
\begin{pgfscope}%
\pgfsys@transformshift{3.741011in}{1.619643in}%
\pgfsys@useobject{currentmarker}{}%
\end{pgfscope}%
\end{pgfscope}%
\begin{pgfscope}%
\definecolor{textcolor}{rgb}{0.000000,0.000000,0.000000}%
\pgfsetstrokecolor{textcolor}%
\pgfsetfillcolor{textcolor}%
\pgftext[x=3.492938in, y=1.581063in, left, base]{\color{textcolor}{\sffamily\fontsize{8.000000}{9.600000}\selectfont\catcode`\^=\active\def^{\ifmmode\sp\else\^{}\fi}\catcode`\%=\active\def
\end{pgfscope}%
\begin{pgfscope}%
\definecolor{textcolor}{rgb}{0.000000,0.000000,0.000000}%
\pgfsetstrokecolor{textcolor}%
\pgfsetfillcolor{textcolor}%
\pgftext[x=3.437382in,y=1.250704in,,bottom,rotate=90.000000]{\color{textcolor}{\sffamily\fontsize{9.000000}{10.800000}\selectfont\catcode`\^=\active\def^{\ifmmode\sp\else\^{}\fi}\catcode`\%=\active\def
\end{pgfscope}%
\begin{pgfscope}%
\pgfpathrectangle{\pgfqpoint{3.741011in}{0.559404in}}{\pgfqpoint{2.423729in}{1.382600in}}%
\pgfusepath{clip}%
\pgfsetbuttcap%
\pgfsetroundjoin%
\pgfsetlinewidth{1.405250pt}%
\definecolor{currentstroke}{rgb}{0.000000,0.000000,0.000000}%
\pgfsetstrokecolor{currentstroke}%
\pgfsetdash{}{0pt}%
\pgfpathmoveto{\pgfqpoint{4.034796in}{0.968022in}}%
\pgfpathlineto{\pgfqpoint{4.034796in}{1.000022in}}%
\pgfusepath{stroke}%
\end{pgfscope}%
\begin{pgfscope}%
\pgfpathrectangle{\pgfqpoint{3.741011in}{0.559404in}}{\pgfqpoint{2.423729in}{1.382600in}}%
\pgfusepath{clip}%
\pgfsetbuttcap%
\pgfsetroundjoin%
\pgfsetlinewidth{1.405250pt}%
\definecolor{currentstroke}{rgb}{0.000000,0.000000,0.000000}%
\pgfsetstrokecolor{currentstroke}%
\pgfsetdash{}{0pt}%
\pgfpathmoveto{\pgfqpoint{4.493836in}{0.969897in}}%
\pgfpathlineto{\pgfqpoint{4.493836in}{1.002066in}}%
\pgfusepath{stroke}%
\end{pgfscope}%
\begin{pgfscope}%
\pgfpathrectangle{\pgfqpoint{3.741011in}{0.559404in}}{\pgfqpoint{2.423729in}{1.382600in}}%
\pgfusepath{clip}%
\pgfsetbuttcap%
\pgfsetroundjoin%
\pgfsetlinewidth{1.405250pt}%
\definecolor{currentstroke}{rgb}{0.000000,0.000000,0.000000}%
\pgfsetstrokecolor{currentstroke}%
\pgfsetdash{}{0pt}%
\pgfpathmoveto{\pgfqpoint{4.952875in}{1.037661in}}%
\pgfpathlineto{\pgfqpoint{4.952875in}{1.070670in}}%
\pgfusepath{stroke}%
\end{pgfscope}%
\begin{pgfscope}%
\pgfpathrectangle{\pgfqpoint{3.741011in}{0.559404in}}{\pgfqpoint{2.423729in}{1.382600in}}%
\pgfusepath{clip}%
\pgfsetbuttcap%
\pgfsetroundjoin%
\pgfsetlinewidth{1.405250pt}%
\definecolor{currentstroke}{rgb}{0.000000,0.000000,0.000000}%
\pgfsetstrokecolor{currentstroke}%
\pgfsetdash{}{0pt}%
\pgfpathmoveto{\pgfqpoint{5.411915in}{0.976090in}}%
\pgfpathlineto{\pgfqpoint{5.411915in}{1.008873in}}%
\pgfusepath{stroke}%
\end{pgfscope}%
\begin{pgfscope}%
\pgfpathrectangle{\pgfqpoint{3.741011in}{0.559404in}}{\pgfqpoint{2.423729in}{1.382600in}}%
\pgfusepath{clip}%
\pgfsetbuttcap%
\pgfsetroundjoin%
\pgfsetlinewidth{1.405250pt}%
\definecolor{currentstroke}{rgb}{0.000000,0.000000,0.000000}%
\pgfsetstrokecolor{currentstroke}%
\pgfsetdash{}{0pt}%
\pgfpathmoveto{\pgfqpoint{5.870955in}{1.766215in}}%
\pgfpathlineto{\pgfqpoint{5.870955in}{1.876166in}}%
\pgfusepath{stroke}%
\end{pgfscope}%
\begin{pgfscope}%
\pgfpathrectangle{\pgfqpoint{3.741011in}{0.559404in}}{\pgfqpoint{2.423729in}{1.382600in}}%
\pgfusepath{clip}%
\pgfsetbuttcap%
\pgfsetroundjoin%
\definecolor{currentfill}{rgb}{0.121569,0.466667,0.705882}%
\pgfsetfillcolor{currentfill}%
\pgfsetlinewidth{1.003750pt}%
\definecolor{currentstroke}{rgb}{0.000000,0.000000,0.000000}%
\pgfsetstrokecolor{currentstroke}%
\pgfsetdash{}{0pt}%
\pgfsys@defobject{currentmarker}{\pgfqpoint{-0.027778in}{-0.000000in}}{\pgfqpoint{0.027778in}{0.000000in}}{%
\pgfpathmoveto{\pgfqpoint{0.027778in}{-0.000000in}}%
\pgfpathlineto{\pgfqpoint{-0.027778in}{0.000000in}}%
\pgfusepath{stroke,fill}%
}%
\begin{pgfscope}%
\pgfsys@transformshift{4.034796in}{0.968022in}%
\pgfsys@useobject{currentmarker}{}%
\end{pgfscope}%
\begin{pgfscope}%
\pgfsys@transformshift{4.493836in}{0.969897in}%
\pgfsys@useobject{currentmarker}{}%
\end{pgfscope}%
\begin{pgfscope}%
\pgfsys@transformshift{4.952875in}{1.037661in}%
\pgfsys@useobject{currentmarker}{}%
\end{pgfscope}%
\begin{pgfscope}%
\pgfsys@transformshift{5.411915in}{0.976090in}%
\pgfsys@useobject{currentmarker}{}%
\end{pgfscope}%
\begin{pgfscope}%
\pgfsys@transformshift{5.870955in}{1.766215in}%
\pgfsys@useobject{currentmarker}{}%
\end{pgfscope}%
\end{pgfscope}%
\begin{pgfscope}%
\pgfpathrectangle{\pgfqpoint{3.741011in}{0.559404in}}{\pgfqpoint{2.423729in}{1.382600in}}%
\pgfusepath{clip}%
\pgfsetbuttcap%
\pgfsetroundjoin%
\definecolor{currentfill}{rgb}{0.121569,0.466667,0.705882}%
\pgfsetfillcolor{currentfill}%
\pgfsetlinewidth{1.003750pt}%
\definecolor{currentstroke}{rgb}{0.000000,0.000000,0.000000}%
\pgfsetstrokecolor{currentstroke}%
\pgfsetdash{}{0pt}%
\pgfsys@defobject{currentmarker}{\pgfqpoint{-0.027778in}{-0.000000in}}{\pgfqpoint{0.027778in}{0.000000in}}{%
\pgfpathmoveto{\pgfqpoint{0.027778in}{-0.000000in}}%
\pgfpathlineto{\pgfqpoint{-0.027778in}{0.000000in}}%
\pgfusepath{stroke,fill}%
}%
\begin{pgfscope}%
\pgfsys@transformshift{4.034796in}{1.000022in}%
\pgfsys@useobject{currentmarker}{}%
\end{pgfscope}%
\begin{pgfscope}%
\pgfsys@transformshift{4.493836in}{1.002066in}%
\pgfsys@useobject{currentmarker}{}%
\end{pgfscope}%
\begin{pgfscope}%
\pgfsys@transformshift{4.952875in}{1.070670in}%
\pgfsys@useobject{currentmarker}{}%
\end{pgfscope}%
\begin{pgfscope}%
\pgfsys@transformshift{5.411915in}{1.008873in}%
\pgfsys@useobject{currentmarker}{}%
\end{pgfscope}%
\begin{pgfscope}%
\pgfsys@transformshift{5.870955in}{1.876166in}%
\pgfsys@useobject{currentmarker}{}%
\end{pgfscope}%
\end{pgfscope}%
\begin{pgfscope}%
\pgfsetrectcap%
\pgfsetmiterjoin%
\pgfsetlinewidth{0.803000pt}%
\definecolor{currentstroke}{rgb}{0.000000,0.000000,0.000000}%
\pgfsetstrokecolor{currentstroke}%
\pgfsetdash{}{0pt}%
\pgfpathmoveto{\pgfqpoint{3.741011in}{0.559404in}}%
\pgfpathlineto{\pgfqpoint{3.741011in}{1.942004in}}%
\pgfusepath{stroke}%
\end{pgfscope}%
\begin{pgfscope}%
\pgfsetrectcap%
\pgfsetmiterjoin%
\pgfsetlinewidth{0.803000pt}%
\definecolor{currentstroke}{rgb}{0.000000,0.000000,0.000000}%
\pgfsetstrokecolor{currentstroke}%
\pgfsetdash{}{0pt}%
\pgfpathmoveto{\pgfqpoint{3.741011in}{0.559404in}}%
\pgfpathlineto{\pgfqpoint{6.164740in}{0.559404in}}%
\pgfusepath{stroke}%
\end{pgfscope}%
\begin{pgfscope}%
\definecolor{textcolor}{rgb}{0.000000,0.000000,0.000000}%
\pgfsetstrokecolor{textcolor}%
\pgfsetfillcolor{textcolor}%
\pgftext[x=0.444740in,y=1.997754in,left,bottom]{\color{textcolor}{\sffamily\fontsize{10.000000}{12.000000}\bfseries\selectfont\catcode`\^=\active\def^{\ifmmode\sp\else\^{}\fi}\catcode`\%=\active\def
\end{pgfscope}%
\begin{pgfscope}%
\definecolor{textcolor}{rgb}{0.000000,0.000000,0.000000}%
\pgfsetstrokecolor{textcolor}%
\pgfsetfillcolor{textcolor}%
\pgftext[x=3.741011in,y=1.997754in,left,bottom]{\color{textcolor}{\sffamily\fontsize{10.000000}{12.000000}\bfseries\selectfont\catcode`\^=\active\def^{\ifmmode\sp\else\^{}\fi}\catcode`\%=\active\def
\end{pgfscope}%
\end{pgfpicture}%
\makeatother%
\endgroup%